\documentclass[moor]{informs3}              % for a regular run
\usepackage{booktabs}
\newcolumntype{x}[1]{>{\centering\arraybackslash\hspace{0pt}}p{#1}}

\newcommand{\santinote}[1]{}

\newcommand{\expcrossu}{{\small{\sf{EXP3.CL-U}}}}
\newcommand{\expcrossk}{{\small{\sf{EXP3.CL}}}}
\newcommand{\ucbcross}{{\small{\sf{UCB1.CL}}}}
\usepackage{multirow}

\newcommand{\D}{\mathcal{D}}
\newcommand{\F}{\mathcal{F}}
\newcommand{\R}{\mathbb{R}}
\newcommand{\eps}{\varepsilon}

\def\E{\mathbb{E}}
\def \Reg  {{\sf Reg}}
\def \post {{\sf{\small{Post}}}}
\def \ante {{\sf{\small{Ante}}}}
\def \pre {{\sf{\small{Pre}}}}
\def \stoch {{\sf{\small{Stoc}}}}
\def \adv {{\sf{\small{Adv}}}}

\usepackage{natbib}
 \NatBibNumeric
 \def\bibfont{\small}%
 \bibpunct[, ]{[}{]}{,}{n}{}{,}%
 
 \usepackage{endnotes}
\usepackage{algorithm2e}
\usepackage[noend]{algorithmic}
\usepackage{enumitem}
\usepackage{caption}
\usepackage{subcaption}
\newenvironment{policy}[1][htb]
  {% Update algorithm name
   \begin{algorithm}[#1]%
  }{\end{algorithm}}
  
  \usepackage{soul}
 \usepackage{bbm, dsfont} 
\usepackage{natbib}
 \bibpunct[, ]{(}{)}{,}{a}{}{,}%
 \def\bibfont{\small}%
\usepackage[colorlinks=true,breaklinks=true,bookmarks=true,urlcolor=blue,
     citecolor=blue,linkcolor=blue,bookmarksopen=false,draft=false]{hyperref}

\def\EMAIL#1{\href{mailto:#1}{#1}}% When hyperref is used, otherwise outcomment 
\TheoremsNumberedThrough     % Preferred (Theorem 1, Lemma 1, Theorem 2)
\EquationsNumberedThrough    % Default: (1), (2), ...
\OneAndAHalfSpacedXI

\begin{document}
%%%%%%%%%%%%%%%%

%\setcounter{page}{1}
%\DOI{10.1287/xxxx.0000.0000}%


%Title or shortened title suitable for running heads. Sample:
\RUNTITLE{Contextual Bandits with Cross-Learning}

% \TITLE{Bundling Information Goods of Decreasing Value}
\TITLE{Contextual Bandits with Cross-Learning}
 %\RUNTITLE{Bundling Information Goods of Decreasing Value}
%\RUNTITLE{}

% \TITLE{Bundling Information Goods of Decreasing Value}
%\TITLE{}

\ARTICLEAUTHORS{%
\AUTHOR{Santiago Balseiro}
\AFF{Columbia Business School, Columbia University, New York, NY, \EMAIL{srb2155@columbia.edu}}
\AUTHOR{Negin Golrezaei}
\AFF{Sloan School of Management, Massachusetts Institute of Technology, Cambridge, MA, \EMAIL{golrezae@mit.edu}}
\AUTHOR{Mohammad Mahdian, Vahab Mirrokni, and Jon Schneider}
\AFF{Google Research, New York, NY, \EMAIL{{mahdian, mirrokni, jschnei}@google.com}}
} % end of the block

\ABSTRACT{%
In the classical contextual bandits problem, in each round $t$, a learner observes some context $c$, chooses some action $i$ to perform, and receives some reward $r_{i,t}(c)$. We consider the variant of this problem where in addition to receiving the reward $r_{i,t}(c)$, the learner also learns the values of $r_{i,t}(c')$ for some other contexts $c'$ in set $\mathcal{O}_i(c)$; i.e., the rewards that would have been achieved by performing that action under different contexts $c'\in \mathcal{O}_i(c)$. This variant arises in several strategic settings, such as learning how to bid in non-truthful repeated auctions, which has gained a lot of attention lately as many platforms have switched to running first-price auctions. We call this problem the contextual bandits problem with cross-learning. The best algorithms for the classical contextual bandits problem achieve $\tilde{O}(\sqrt{CKT})$ regret against all stationary policies, where $C$ is the number of contexts, $K$ the number of actions, and $T$ the number of rounds. {\color{black}We design and analyze  new algorithms for the contextual bandits problem with cross-learning and show that their regret has  better dependence on the number of contexts. Under \emph{complete cross-learning} where the rewards for all contexts are learned when choosing an action, i.e.,  set $\mathcal{O}_i(c)$ contains all contexts, we show that our algorithms 
 achieve regret $\tilde{O}(\sqrt{KT})$, removing the dependence on $C$. For any other cases, i.e., under \emph{partial cross-learning} where $|\mathcal{O}_i(c)|< C$ for some context-action pair of $(i,c)$, the regret bounds depend on how the sets $\mathcal O_i(c)$ impact the degree to which cross-learning between contexts is possible.} We simulate our algorithms on real auction data from an ad exchange running first-price auctions and show  that they outperform traditional contextual bandit algorithms.
}%

%\KEYWORDS{deterministic inventory theory; infinite linear programming duality; 
%  existence of optimal policies; semi-Markov decision process; cyclic schedule}
%\MSCCLASS{Primary: 90B05; secondary: 90C40, 90C90}
%\ORMSCLASS{Primary: Inventory/production: deterministic multi-item;
%  secondary: dynamic programming/optimal control: deterministic 
%  semi-Markov; programming: infinite dimensional}
%\HISTORY{Received November 20, 2003; revised March 8, 2004, and March 26, 2004.}

\KEYWORDS{contextual bandits, cross learning, bidding, first-price auctions\footnote{Part of this
work has appeared in \cite{crosslearningNIPS2019}.}}
%\MSCCLASS{}
%\ORMSCLASS{Primary: ; secondary: }

\maketitle
%%%%%%%%%%%%%%%%%%%%%%%%%%%%%%%%%%%%%%%%%%%%%%%%%%%%%%%%%%%%%%%%%%%%%%

% Samples of sectioning (and labeling) in MOOR.
% NOTE: (1) all section levels end with a period,
%       (2) capitalization is as shown (sentence style, not title style).
%
%\section{Introduction.}\label{intro} %%1.
%\subsection{Duality and the classical EOQ problem.}\label{class-EOQ} %% 1.1.
%\subsection{Outline.}\label{outline1} %% 1.2.
%\subsubsection{Cyclic schedules for the general deterministic SMDP.}
%  \label{cyclic-schedules} %% 1.2.1
%\section{Problem description.}\label{problemdescription} %% 2.

% Text of your paper here

% Appendix here
% Options are (1) APPENDIX (with or without general title) or 
%             (2) APPENDICES (if it has more than one unrelated sections)
% Outcomment the appropriate case if necessary
%
% \begin{APPENDIX}{<Title of the Appendix>}
% \end{APPENDIX}
%
%   or 
%
% \begin{APPENDICES}
% \section{<Title of Section A>}
% \section{<Title of Section B>}
% etc
% \end{APPENDICES}
\section{Introduction}
In the contextual bandits problem, a learner repeatedly observes some context and, depending on the context, the learner takes some action and receives some reward. The learner's goal is to maximize their total reward over some number of rounds. The contextual bandits problem is a fundamental problem in online learning: it is a simplified (yet analyzable) variant of reinforcement learning and it captures a large class of repeated decision problems. In addition, the algorithms developed for the contextual bandits problem have been successfully applied in domains like ad placement, news recommendation, and clinical trials \citep{kale2010non, li2010contextual, villar2015multi}.

Ideally, one would like an algorithm for the contextual bandits problem which performs approximately as well as the best stationary strategy (i.e., the best fixed mapping from contexts to actions). This can be accomplished by running a separate instance of some low-regret algorithm for the non-contextual bandits problem (e.g., EXP3 proposed in \citealt{AuerCNS03}) for every context. This algorithm achieves regret $\tilde{O}(\sqrt{CKT})$,  where $C$ is the number of contexts, $K$ the number of actions, and $T$ the number of rounds. This bound can be shown to be tight \citep{Bubeck12}. Since the number of contexts can be very large, these algorithms can be impractical to use, and modern current research on the contextual bandits problem instead aims to achieve low regret with respect to some smaller set of policies \citep{AuerCNS03, LangfordZ08, beygelzimer2011contextual}.

Some settings, however, possess additional structure between the rewards and contexts that allow one to achieve less than $\tilde{O}(\sqrt{CKT})$ regret while still competing with the best stationary strategy. In this paper, we look at a specific type of structure we call \emph{cross-learning between contexts} that is particularly common in strategic settings with private information in which agents can compute counterfactual rewards under different contexts. In variants of the contextual bandits problem with this structure, playing an action $i$ in some context $c$ at round $t$ not only reveals the reward $r_{i,t}(c)$ of playing this action in this context (which the learner receives), but also reveals to the learner the rewards $r_{i,t}(c')$ for  some other context $c'$ in some set $\cal O_{i}(c)$, which depends on the action played. The set $\cal O_{i}(c)$ can include every context, a case which we refer to as \emph{complete cross-learning}, or only a subset of contexts, which we refer to as \emph{partial cross-learning}. Partial cross learning setting can be used to model conservative learners that only use local information obtained from contexts that are ``close" to each other in some sense. Such a conservative learner might be concerned about inaccuracies in their cross-learning model due to a significant differences in the contexts.

Contextual bandits with cross learning appear in many settings including (i) bidding in non-truthful auctions, (ii) multi-armed bandits with exogenous costs, (iii) dynamic pricing with variable cost, (iv) sleeping bandits, and (v) repeated Bayesian games with private types. While our model and results are general, the main application of interest  in this paper is the problem of bidding in non-truthful auctions (such as first-price auctions), which we describe in detail below.  We refer the reader to  Section~\ref{sect:applications} for more details about the other applications.

\subsection{Bidding in Non-truthful Auctions}

The problem of bidding in non-truthful auctions has gained a lot of attention recently as online advertising platforms have recently switched from running second-price to  first-price auctions. Many online publishers have adopted header bidding, in which publishers offer ad impressions to multiple ad exchanges simultaneously using a first-price auction, rather than  offering ad impressions sequentially to different exchanges, which would typically auction impressions using second-price auctions, in a waterfall fashion. Additionally, some major ad exchanges have adopted first-price auctions to sell all their inventory \citep{google2019firstprice}. In a first-price auction, the highest bidder is the winner and pays their bid (as opposed to second-price auctions where the winner pays the second highest-bid). First-price auctions are non-truthful mechanisms  as bidders have incentives to shade bids so that they enjoy a positive utility when they win \citep{vickrey1961counterspeculation}. As opposed to second-price auctions in which bidding is simple, determining bids in first-price auctions is challenging as bidders need to take into account the competitive landscape, which is typically unknown.

More formally, in the problem of bidding in repeated non-truthful auctions, at every round, the bidder receives a (private) value for the current item, and based on this, must submit a bid for the item. The auctioneer then collects the bids from all participants, and decides whether to allocate the item to our bidder, and if so, how much to charge the bidder. The bidding problem in first-price auctions can be seen as a contextual bandits problem for the bidder where the context $c$ is the bidder's value for the item, the action $i$ is their bid, and their reward is their net utility from the auction: zero if they do not win, and their value for the item minus their payment $p$ if they do win. Note that this problem also allows for cross-learning between contexts -- the net utility $r_{i,t}(c')$ that would have been received if they had value $c'$ instead of value $c$ is just $(c'-p)\cdot \mathbbm{I}(\mbox{win item})$, 
which can be  computed from the outcome of the auction assuming the value and highest competing bid (that influences $\mathbbm{I}(\mbox{win item})$) are independent of each other; that is $\mathcal O_{i}(c)$ is the set of all possible contexts.

This independence assumption, however, may not hold in practice, as we also observe in our empirical studies in Section \ref{sect:experiments}. When the bidder's value (context) is correlated with the highest competing bid, 
 with the same action/bid, the chance of winning (i.e., $\mathbbm{I}(\mbox{win item})$) under two values that are far from each other  may not be the same. For instance, when there is a positive correlation between values and highest competing bids, as the value increases, the highest competing bid may increase as well, reducing the chance of winning. 
To handle this, such learners can only allow for cross-learning between close values. That is, $\mathcal O_{i}(c)$ can be chosen to be the set of contexts $c'$ that are close  enough to $c$.

\iffalse
\begin{table}
\footnotesize{
\begin{center}
 \begin{tabular}{c|c| c| c |c|} 
 \cline{2-5}
 &Bidding in  & Multi-armed  & Dynamic & Sleeping \\ [0.5ex] 
&First-price & Bandits with & Pricing with & Bandits \\
&Auctions&  Exogenous Costs & Variable Cost &\\ \hline
 \multicolumn{1}{|c|} {Our}&  
 &  \\ 
\multicolumn{1}{|c|} {Bound} \\
 \hline
 \multicolumn{1}{|c|} {Best Prior } \\
  \multicolumn{1}{|c|} {Bound}   \\
 \hline
\end{tabular}
\end{center}}
\end{table}
\fi
\subsection{Main Contributions}
We introduce and study contextual bandit  problems  with cross-learning between contexts. In this problem,  for every action $i \in [K]$, there is a directed graph $G_i$ over the set of contexts $[C]$, where  an edge $c \rightarrow c'$ in $G_i$ indicates that playing action $i$ in context $c$, reveals  the reward of action $i$ in context $c'$. 
We refer to these graphs as cross-learning (CL) graphs, where these graphs are known to the learner. (Set $\cal O_i(c)$, defined before, is  the set of vertices of out-neighbors of node/context $c$ in CL graph $G_i$.) We study to what extent cross-learning between contexts can improve regret bounds.  In many settings, the number of possible contexts $C$ can be huge: exponential in the number of actions $K$ or uncountably infinite. This makes the naive $O(\sqrt{CKT})$-regret algorithm undesirable in these settings.  We show that the extent with which cross-learning can improve the regret bounds depend on how ``well-connected" the CL graphs are. For instance, when CL graphs are complete graphs, i.e., under {complete cross-learning} between contexts, we show that    it is possible to design algorithms which completely remove the dependence on the number of contexts $C$ in their regret bound. For any other general CL graphs, i.e., under {partial cross-learning},  in addition to  the CL graphs themselves, the improvement in obtained regret depends on how contexts and rewards are generated. 

We consider both settings where the contexts are generated stochastically (from some distribution $\D$ that may or may not be known to the learner) and settings where the contexts are chosen adversarially. Similarly, we also consider settings where the rewards are generated stochastically and settings where they are chosen adversarially. By considering stochastic rewards,  we can capture environments under which the obtained reward (given a context) is stationary and predictable. Whereas by considering adversarial rewards, we can capture environments with non-stationary and hard-to-predict rewards. Similarly, stochastic (respectively adversarial) contexts model environments with stationary and predictable (respectively unpredictable) side information.   
Our results, which are also summarized in Table \ref{table:results},   include:

\begin{itemize}[leftmargin=*]
\item \textbf{Stochastic rewards, stochastic or adversarial contexts}: We design an algorithm called~\ucbcross{} with regret of $\tilde O(\sqrt{\overline{\kappa}KT})$, where $\overline{\kappa}$ is the average size of the \textit{minimum clique cover} of the CL graphs; the minimum clique cover of a directed graph is defined in Definition~\ref{def:clique}. %Here, ``CL" stands for Cross Learning. 
Observe that for complete CL graphs, $\overline{\kappa}=1$, and hence    \ucbcross{} obtains regret $\tilde O(\sqrt{KT})$ under complete cross-learning,  removing the dependence of the regret bound on the number of contexts $C$. On the other hand, when CL graphs contain only self-loops, i.e.,  under \emph{no cross-learning},  $\overline{\kappa}=C$, which leads to a regret of $\tilde O(\sqrt{CKT})$, as expected. 
\item \textbf{Adversarial rewards, stochastic contexts with known distribution}: 
We design an algorithm called \expcrossk{} with  regret of  $\tilde O(\sqrt{\overline{\lambda}KT})$, where $\overline{\lambda}$ is the average size of the \textit{maximum acyclic subgraph} of the CL graphs; see Definition \ref{def:acyclic}. We again note that for complete CL graphs, we have $\overline{\lambda} =1$, which implies that \expcrossk{} obtains regret of $\tilde O(\sqrt{KT})$ under complete cross-learning. 
\item \textbf{Adversarial rewards, stochastic contexts with unknown distribution}: We design an algorithm called \expcrossu{} with regret $\tilde{O}(\bar \lambda K^{1/3}T^{2/3})$.\footnote{For this result to hold we need an assumption that CL graph $G_i =G$ for any $i\in [K]$.}   Here, ``U" stands for Unknown distribution. 
\item \textbf{Lower bound for adversarial rewards, adversarial contexts}: We show that when both rewards and contexts are controlled by an adversary, even under complete cross-learning, any algorithm must obtain a regret of at least $\tilde \Omega(\sqrt{CKT})$.

\end{itemize}\medskip

\renewcommand{\arraystretch}{1.3}
\begin{table}
\begin{center}\footnotesize
 %\begin{tabular}{c c|x{2cm} x{2cm} x{2cm} x{2cm}|x{2cm}|} 
 \begin{tabular}{c c|c c c c|c|} 
 
 \cline{3-7} 
  & & \multicolumn{2}{c|}{\multirow{ 2}{*}{Stoc. Rewards}} & \multicolumn{2}{c|}{Adv. Rewards \& }  & \multirow{ 2}{*}{Adv. Rewards \&}\\ [0.5ex] 
  && &\multicolumn{1}{c|}{}
  & \multicolumn{2}{c|}{Stoc. Contexts}  &\\
  \cline{3-6} &&\multicolumn{1}{c|}{Stoc.} & \multicolumn{1}{c|}{Adv.} &  \multicolumn{1}{c|}{Known} & Unknown &   \multicolumn{1}{c|}{\multirow{ 2}{*}{Adv. Contexts}}\\ 
  &&\multicolumn{1}{c|}{Contexts} & \multicolumn{1}{c|}{Contexts} & \multicolumn{1}{c|}{Context Dist.} & Context Dist. &\\
 \hline 
 \multicolumn{1}{|c|}{\multirow{2}{*}{Complete CL}} & Upper Bound &  \multicolumn{3}{c|}{$\tilde{O}(\sqrt{KT})$} %& \multicolumn{1}{c|}{$\tilde{O}(\sqrt{KT})$}
 & $\tilde{O}(K^{1/3}T^{2/3})$ 
 &  $\tilde O(\sqrt{CKT})$ \\   \cline{2-7}
 \multicolumn{1}{|c|}{}& Lower Bound & \multicolumn{4}{c|}{$\tilde{\Omega}(\sqrt{KT})$} 
 &  $\tilde \Omega(\sqrt{CKT})$ \\   \cline{1-7}
\multicolumn{1}{|c|}{\multirow{2}{*}{Partial CL}} & Upper Bound & \multicolumn{2}{c|}{$\tilde O(\sqrt{\overline{\kappa}KT\log T})$} &\multicolumn{1}{c|}{$\tilde O(\sqrt{\overline{\lambda}KT})$} & $\tilde{O}(\bar \lambda K^{1/3}T^{2/3})$ & $\tilde O(\sqrt{CKT})$\\ \cline{2-7}
\multicolumn{1}{|c|}{}& Lower Bound & \multicolumn{1}{c|}{$\tilde \Omega(\sqrt{\nu_2(G)KT})$} & \multicolumn{1}{c|}{$\tilde \Omega(\sqrt{\lambda (G)KT})$} & \multicolumn{2}{c|}{$\tilde \Omega(\sqrt{\nu_2(G)KT})$}  & $\tilde \Omega(\sqrt{CKT})$ \\

 %\hline
 %\multicolumn{1}{|c|} {Best Prior } & Complete CL & \multicolumn{4}{c|}{$\tilde{O}(\sqrt{CKT})$}   \\ \cline{2-6}
%  \multicolumn{1}{|c|} {Upper Bound} & Partial CL & \multicolumn{4}{c|}{$\tilde{O}(\sqrt{CKT})$}    \\
 \hline
\end{tabular}
\end{center}
\caption{\label{table:results} Here, $\overline{\kappa} = \frac{1}{K}\sum_{i \in [K]} \kappa(G_i)$, and $\kappa(G_i)$, which  is the clique covering number of the CL graph $G_i$, is defined in Definition \ref{def:clique}. Moreover, $\bar{\lambda} = \frac{1}{K}\sum_{i \in [K]}\lambda(G_i)$, and $\lambda(G_i)$, which is the maximum acyclic subgraph number of $G_i$, is defined in Definition \ref{def:acyclic}. Finally, $\nu_2(G)$, which  is defined in Definition \ref{def:nu_2}, can be seen as $L_2$ variant of the maximum acyclic subgraph number $\lambda(G)$. (Our regret upper bound  for partial CL with adversarial rewards and stochastic context with unknown distribution (i.e., regret of \expcrossu) hold when all CL graphs $G_i$, $i\in [K]$ are the same cross all the actions.) The best known regret upper bound prior to this work is $\tilde O(\sqrt{CKT})$.   In our regret lower bounds, we assume that $G_i = G$ for any $i\in [K]$. }
\end{table}

All of these algorithms are easy to implement, in the sense that they can be obtained via simple modifications from existing multi-armed bandit algorithms like EXP3 (\cite{AuerCNS03}) and UCB1 \citep{Robbins52, LaiR85}, and efficient, in the sense that all algorithms run in time at most $O(C+K)$ per round (and for many of the applications  mentioned above, this can be further improved to $O(K)$ time per round). Our main technical contribution is our analysis of \ucbcross{}, which requires arguing that UCB1 can effectively use the information from cross-learning despite it being drawn from a distribution that differs from the desired exploration distribution. We accomplish this by constructing a linear program whose value upper bounds (one of the terms in) the regret of \ucbcross{}, and bounding the value of this linear program.

{\color{black}Our \expcrossk{} and \expcrossu{} algorithms that are designed for the adversarial rewards and stochastic contexts setting are modifications of the EXP3 algorithm.   These algorithms maintain a weight for each action in each context, and update the weights via multiplicative updates by an exponential of an estimator of the reward, taking advantage of cross-learning between contexts via their update rules and estimators. 
 The main difference between  \expcrossk{} and \expcrossu{} is how their reward estimators are constructed. In  \expcrossk{}, which is designed for the case of known context distribution $\mathcal D$, the estimator is unbiased, and crucially uses the knowledge of $\mathcal D$ to obtain minimal variance. In  \expcrossu{}, which is designed for the case of unknown context distribution $\mathcal D$, the estimator is only unbiased under the complete cross-learning setting. The estimator, however, is \emph{consistently biased} for the partial cross-learning setting, easing the analysis. See Section \ref{sec:unknown_dist} for more details.  }

To shed lights on how tight regret bounds of our algorithms are, we further present regret lower bound for the aforementioned settings; see Table \ref{table:results}.  We show that when both rewards and contexts are generated stochastically, any algorithm must obtain a regret of at least $\tilde \Omega(\sqrt{\nu_2(G)KT})$, where $\nu_2(G)$, which  is defined in Definition \ref{def:nu_2}, can be seen as $L_2$ variant of the maximum acyclic subgraph number $\lambda(G)$. In all of regret lower bounds, we assume that CL graphs do not depend on actions; that is, $G_i = G$ for any $i\in [K]$.
We further show a regret lower bound of $\tilde \Omega(\sqrt{\lambda(G)KT})$ when rewards  are generated stochastically and contexts are generated adversarially. This regret lower bound  also leads to the same regret lower bound for the case of adversarial rewards and stochastic contexts. 

By comparing our regret lower bounds with the regret of our algorithms, we observe that regret of our algorithms  are indeed tight for many CL graphs in various settings. 
Consider settings where either (i) rewards are generated stochastically and contexts are either stochastic  or adversarial, or  (ii) rewards  are adversarial and contexts are stochastic (with known context distribution). Then,  for these settings, 
 our regret upper bounds are tight for CL graphs that are the union of $r$ disjoint cliques. In fact, our regret upper bounds are tight for an even larger family of graphs, including line graphs, and any undirected graph that is \textit{perfect}.  Perfect graphs include all bipartite graphs, forests, interval graphs, and comparability graphs of posets (see, e.g., \citealt{west1996introduction}).

 {Our regret upper bound for the settings with adversarial rewards and stochastic contexts (with unknown context distribution), i.e., regret of \expcrossu{} algorithm,  does not match its associated regret lower bound. In Appendix \ref{sec:discuss}, we present another candidate algorithm for this setting, which can be viewed as a variant of \expcrossk{} algorithm that uses an empirical estimate of  context distribution in place of the true context distribution, which is not available. While this variant performs very well in our empirical studies in Section \ref{sect:experiments}, because of several technical  challenges explained in Appendix \ref{sec:discuss}, we are not able to show the regret of this algorithm matches the regret lower bound. We leave analyzing the regret of this variant as a future research direction.     }

We also apply our results to some of the applications listed above. In each case, our algorithms obtain optimal regret bounds with asymptotically less regret than a naive application of contextual bandits algorithms. In particular, for the problem of learning to bid in a first-price auction, standard contextual bandit algorithms get regret $O(T^{3/4})$. Our algorithms achieve regret $O(T^{2/3})$. This is optimal even when there is only a single context (value). Note that in this problem, the set of contexts (values) and actions (bids) are infinite. Thus, one needs to discretize  the set of contexts and actions, and such discretizations increases regret. Since our  algorithms have regret bounds that are independent of $C$, discretizing the context space arbitrarily finely does not deteriorate performance (indeed, as we show in Section \ref{sect:applications}, we can often implement our algorithms for infinite context spaces). We discuss the results for the other applications in Section~\ref{sect:applications}.

Finally, we test the performance of these algorithms on real auction data from a first-price ad exchange. In order for cross-learning to be effective in first-price auctions, the bidder should be able to determine the counterfactual utility for different values. That is, after observing the outcome of the auction, the bidder should predict how would their utility change if their value was different. As stated earlier, this is possible when the bidder's values are independent of other players' bid. In practice, however, one would expect certain degree of correlation between these quantities and, thus, the independence assumption might not hold. Even though our algorithms under the assumption of complete cross-learning  do not explicitly account for correlation, numerical results show that our algorithms are somewhat robust to errors in the cross-learning hypothesis and outperform traditional bandit algorithms. 

\subsection{Related Work}

For a general overview of research on the multi-armed bandit problem, we recommend the reader to the survey by \citet{Bubeck12}. Our algorithms build off of pre-existing algorithms in the bandits literature, such as EXP3 \citep{AuerCNS03} and UCB1 \citep{Robbins52, LaiR85}. Contextual bandits were first introduced under that name in \citet{LangfordZ08}, although similar ideas were present in previous works (e.g., the EXP4 algorithm was proposed in \citealt{AuerCNS03}). 

One line of research related to ours studies bandit problems under other structural assumptions on the problem instances which allow for improved regret bounds.  \cite{slivkins2011contextual}   study a setting where contexts and actions belong to a joint metric space, and context/action pairs that are close to each other give similar rewards, thus allowing for some amount of ``cross-learning.'' See also \cite{hazan2007online} and  \cite{lu2009showing} for works that  consider local smoothness over a continuum of contexts that lie in a known metric space. Some other structural assumptions widely studied in the contextual bandit literature include contextual Gaussian process bandits \citep{krause2011contextual}, linear bandits \citep{li2010contextual}, and contextual bandits with covariates \citep{rigollet2010nonparametric, perchet2013multi, qian2016kernel, guan2018nonparametric}. These structural assumptions allow for some amount of cross-learning between contexts; however, they do not capture the general cross-learning setting we study in this paper, nor do they fit into our setting as we require the learner to be able to obtain sample rewards from other contexts.\footnote{See also \cite{van2020optimal} for the work that leverages convex structural information  and \cite{ kakade2009playing,niazadeh2021online} for work that  exploits combinatorial structures.  }

Several works \citep{mannor2011bandits, alon2015online} study a partial-feedback variant of the (non-contextual) multi-armed bandit problem where performing some action provides some information on the rewards of performing other actions (thus interpolating between the bandits and experts settings). Our setting can be thought of as a contextual version of this variant, and our results in the partial cross-learning setting share similarities with these results (indeed, three of the four graph invariants we consider -- $\kappa(G)$, $\lambda(G)$, and the independence number of graph $G$, denoted by $\iota(G)$ -- appear prominently in the bounds of \citealt{mannor2011bandits, alon2015online}, albeit applied to different graphs). However, since the learner cannot choose the context each round, these two settings are qualitatively different. As far as we are aware, the specific problem of contextual bandits with cross-learning between contexts has not appeared in the literature before.

Recently there has been a surge of interest in applying methods from online learning and bandits to auction design. 
The majority of the work in this area has been from the perspective of the auctioneer \citep{MorgensternR16, mohri2016learning, CaiD17, DudikHLSSV17,kanoria2017dynamic, golrezaei2019IC, golrezaei2018dynamic} in which the goal is to learn how to design an auction over time based on bidder behavior. In fact, many papers have looked at this problem in a simple posted price auction; see, for example 
 \cite{ araman2009dynamic, farias2010dynamic, cheung2017dynamic,den2013simultaneously}, and \cite{besbes2009dynamic}.
Some recent work, which is at the intersection of learning and auction design, studies this problem from the perspective of a buyer learning how to bid \citep{weed2016online, feng2018learning, braverman2017selling}. (See \cite{golrezaei2021bidding} for a recent work that considers  both perspectives of a buyer and a seller.) In particular, \cite{weed2016online} studies the problem of learning to bid in a second-price auction over time, but where the bidder's value remains constant (so there is no context). More generally, ideas from online learning (in particular, the concept of no-regret learning) have been applied to the study of general Bayesian games, where one can characterize the set of equilibria attainable when all players are running low-regret learning algorithms; see, for example,  \citep{hartline2015no, golrezaei2020no}.

\section{Model and Preliminaries}

We start with providing a short overview of non-contextual multi-armed bandits. We then present contextual multi-armed bandits problems, which will be our focus. 

\subsection{Non-Contextual Multi-Armed Bandits}

In the classic (non-contextual) multi-armed bandit problem, a learner chooses one of $K$ arms per round over the course of $T$ rounds. On round $t$, the learner receives some reward $r_{i,t} \in [0, 1]$ for pulling arm $i$ (where the rewards $r_{i,t}$ may be chosen adversarially and {\color{black}may depend on $T$}). The learner's goal is to maximize their total reward.

{\color{black}Let $I_t$ denote the arm pulled by the decision maker's algorithm $\mathcal A$ at round $t$.  The algorithm maps the history set $\{(I_1, r_{I_1, 1}), (I_2, r_{I_2, 2}), \ldots, (I_{t-1},r_{I_{t-1}, t-1})\}$ of pulled arms and their realized rewards during the first $t-1$ rounds, any (realized) randomness in the first $t-1$ rounds, and the total number of rounds $T$ to an arm $I_t$, where this mapping can be deterministic or random.  Throughput this work, we assume that all the algorithms know the total number of rounds $T$. This assumption, which is common in the literature, can be relaxed via the \emph{doubling trick} (see, e.g., \citealt{AuerCNS03, Bubeck12, lattimore2020bandit}). This trick, which can be applied in a black-box fashion, can  efficiently deal with an unknown number of rounds $T$ by repeatedly running an algorithm with horizons of increasing length.}

\subsection{Contextual Multi-Armed Bandits}\label{sect:contextual_prelims}

In our model, we consider a \textit{contextual multi-armed bandits} problem. In the contextual bandits problem, in each round $t$,  the learner is additionally provided with a \textit{context} $c_t$, and the learner now receives reward $r_{i,t}(c_t)$ if they pull arm $i$ on round $t$ while having context $c_t$. The contexts $c_t$ are either chosen adversarially at the beginning of the game or drawn independently each round from some distribution $\D$. Similarly, the rewards $r_{i,t}(c)$ are either chosen adversarially or each independently drawn from some distribution $\F_i(c)$. We assume as is standard that $r_{i,t}(c)$ is always bounded in $[0,1]$.

{\color{black}Again, let $I_t$ denote the arm pulled by the decision maker's algorithm $\mathcal A$ at round $t$ under context $c_t$. Here, at around $t$, the algorithm maps the history set $\{(I_1, r_{I_1, 1}, c_1), (I_2, r_{I_2, 2}, c_2), \ldots, (I_{t-1},r_{I_{t-1}, t-1}, c_{t-1}), c_t\}$ of contexts, the pulled arms and their realized rewards during the first $t-1$ rounds, as well as, the current context $c_t$, any (realized) randomness in the first $t-1$ rounds, and the total number of rounds $T$ to an arm $I_t$, where this mapping can be deterministic or random.}  

In the contextual bandits setting, we  define the regret of an algorithm $\cal A$ in terms of regret against the best stationary benchmark $\pi:\{1,\ldots,C\} \rightarrow \{1,\ldots,K\}$, mapping a context $c$ to an action $\pi(c)$. That is, the regret is defined as  $ 
\sum_{t=1}^{T}r_{\pi(c_t),t}(c_t) - \sum_{t=1}^{T}r_{I_{t}, t}(c_t)$, where $I_t$ 
is the arm pulled by $\cal A$ on round $t$. The definition of best stationary policy $\pi$ depends slightly on how contexts and rewards are generated. {\color{black} In all of these definitions, as  is common in the bandit literature (e.g., the seminal work of \citealt{lai1985asymptotically}),  we assume that the best stationary policy $\pi$ is unique. Nonetheless, all of our gap-independent results continue to hold even when the best stationary policy is non-unique.} %{\color{red} @Jon, OK with this? Any thing else to add about uniqueness. Could we say this is only for technical reasons and algorithms still work if this assumption fails? } \jonnote{Added a sentence about this.}

%{\color{red}Page 8 has 2 definitions of regret (l44 and 25), this is a bit confusing. ‘$p_i( c) =\arg \max$’ discuss uniqueness of the $\arg\max$. Similar comment for ``the" best stationary benchmark, p10l16, and many other places in the paper where uniqueness of optimal arms is assumed.}

%{\color{red}@Jon, What do you think about this comment? \emph{``Page 9 second bullet point: derive briefly why this $\pi$ is the best benchmark. And why couldn’t the adversary provide ‘misleading’ information for the rewards of non-observed contexts? Shouldn’t it be a property of an (optimal) algorithm to show that this information is neglected?  More generally, it may help the reader to define $\pi$ less informally, and emphasize that in some settings $\pi$ depends on information in all time periods $1,\ldots,T$." %The paragraph ‘We consider a variant’ seems misplaced. The last sentence of 2.2 `We define the’ is a bit vague, it would help the reader to provide a clear and formal definition of algorithm, regret, and the benchmark $\pi$. }} \jonnote{I think talked about this a bit last time -- I don't really see what the reviewer is getting at here, and I think what is currently written is fine.}

\begin{itemize}[leftmargin=*]
    \item {\color{black}\textbf{Benchmark under stochastic rewards.}} {When rewards are stochastic, i.e., $r_{i,t}(c)$ is drawn independently\footnote{As is the case with standard stochastic bandits, our proofs work even when the rewards $r_{i,t}(c)$ are correlated across actions and contexts, as long as they are independent across rounds.} from $\F_i(c)$ with mean $\mu_i(c)$, we define $\pi$ to be the stationary policy that maximizes the expectation of performance over rewards $\sum_{t=1}^T \E_{r_{i,c_t} \sim \F_i(c_t)}[ r_{\pi(c_t), t}(c_t)] = \sum_{t=1}^T \mu_{\pi(c_t)}(c_t)$, which leads to the optimal policy $\pi(c) = \arg\max_{i\in [K]}\mu_i(c)$.} That is,  under the best stationary policy, for every context $c$, an arm with the highest average reward is pulled. {\color{black} We highlight that unlike our algorithms, the benchmark $\pi$ has full knowledge of the reward  distributions of all the arms under any contexts and given this knowledge, under context $c$, it chooses arm $i$ with the highest average reward $\mu_i(c)$.  }
    
    \item {\color{black}\textbf{Benchmark under adversarial rewards and stochastic contexts.}} { When rewards are adversarial but contexts are stochastic, we define $\pi$ to be the stationary policy that maximizes the expectation of performance over contexts $\sum_{t=1}^T \E_{c_t \sim \D}[ r_{\pi(c_t), t}(c_t)] = \sum_{c \in [C]} \Pr[c] \sum_{t=1}^T r_{\pi(c), t}(c)$, where the last equation follows because contexts are identically distributed across all the rounds. This is achieved by the policy $\pi(c) = \arg\max_{i\in [K]}\sum_{t=1}^T r_{i, t}(c)$ since the benchmark is separable over contexts.} % ({\color{black} Note that $\pi(c)$ depend on $r_{i,t}(c)$, $i\in [K]$ and $t\in[T]$.})%; see Appendix~\ref{sect:regret} for further details.  

    \item {\color{black}\textbf{Benchmark under adversarial rewards and  contexts.}} When both rewards and contexts are adversarial, we define $\pi(c)$ to be the stationary policy which maximizes $\sum_{t=1}^T r_{\pi(c_t), t}(c_t)$. Precisely, $\pi(c) = \arg\max_{i\in [K]}\sum_{t=1}^T r_{i, t}(c) \mathbbm{I}(c_t =c)$. {In this case, it suffices for the adversary to only specify $\{r_{i,t}(c_t)\}_{i=1}^K$ in each round $t$, i.e., the $K$ rewards for context $c_t$, as the other rewards are never realized}. 
\end{itemize}

Our choices of benchmarks are unified in the following way: in all of the above cases, $\pi$ is the best stationary policy in expectation for someone who knows all the decisions of the adversary and details of the system ahead of time, but not the randomness in the instantiations of contexts/rewards from distributions. This matches commonly studied notions of regret in the contextual bandits literature.

{\color{black}We now comment on our benchmark when rewards are adversarial and  contexts are stochastic. In this case, there are two different natural ways to define ``the best stationary policy.'' The first maximizes the empirical cumulative rewards or, equivalently, the rewards the specific contexts $c_t$ we observed in the run of our algorithm:
$\pi'(c) = \arg\max_{i\in [K]} \sum_{t=1}^{T}r_{i,t}(c)\mathbbm{I}({c_{t} = c})\,.$
The second way that we consider in this work  simply maximizes the reward of this strategy in expectation over all time 
$\pi(c) = \arg\max_{i\in [K]} \sum_{t=1}^{T}r_{i,t}(c).$} 
{\color{black} 
Under $\pi$, at the beginning of the game, the adversary knows all the rewards, but not when each context will occur, and hence, $\pi$ is the best stationary strategy in expectation, where the expectation is taken with respect to contexts. (Recall that  $\sum_{t=1}^T \E_{c_t \sim \D}[ r_{\pi(c_t), t}(c_t)]$ is maximized under the best stationary strategy $\pi$.) Under $\pi'$, on the other hand, at the beginning of the game,  the adversary knows all the rewards and  contexts in each round, and hence  $\pi'$ is  the best stationary strategy in hindsight.

In this paper, when rewards are adversarial and  contexts are stochastic, all bounds we show are with respect to the best stationary strategy in expectation, i.e.,  $\pi(c) = \arg\max_{i\in [K]} \sum_{t=1}^{T}r_{i,t}(c).$ This is because 
the best stationary strategy in hindsight $\pi'$  is too strong when contexts are stochastically drawn from a known distribution. With the latter strategy as a benchmark, no algorithm can be shown to achieve sub-linear regret when the number of contexts is large enough (see Theorem~\ref{theorem:expost-lower} in Appendix~\ref{sect:regret}).} {\color{black} That being said, when rewards and contexts are chosen adversarially, policy $\pi(c) = \arg\max_{i\in [K]} \sum_{t=1}^{T}r_{i,t}(c)$ is no longer well motivated as contexts are not generated stochastically. Hence,  when rewards and contexts are chosen adversarially, 
the stationary policy we consider is the best policy in hindsight (i.e., $\pi'(c) = \arg\max_{i\in [K]} \sum_{t=1}^{T}r_{i,t}(c)\mathbbm{I}({c_{t} = c}).$)}

We conclude by noting that there is a simple way to construct an algorithm $\cal A'$ with sublinear regret of $o(T)$ for the contextual bandits problem from a sublinear-regret  algorithm $\cal A$ for the classic bandits problem: simply maintain a separate instance of $\cal A$ for every different context $c$. In the contextual bandits literature, this is sometimes referred to as the $S$-EXP3 algorithm when $\cal A$ is EXP3 \citep{Bubeck12}. The $S$-EXP3 algorithm has  regret of order  $\tilde{O}(\sqrt{CKT})$. We define the $S$-UCB1 algorithm similarly, which also obtains  $\tilde{O}(\sqrt{CKT})$ regret when rewards are generated stochastically. { Our goal in this work is to develop algorithms with better dependence on the number of contexts by exploiting the possibility of cross learning between them. See the formal definition of cross learning between contexts in the next section.}

\subsection{Contextual Multi-Armed Bandits with Cross-Learning between Contexts}
We consider a variant of the contextual bandits problem we call (partial) \textit{contextual bandits with cross-learning}. {In this variant, whenever the learner pulls arm $i$ in round $t$ while having context $c$ and receives reward $r_{i,t}(c)$, they also learn the value of $r_{i,t}(c')$ for some subset of contexts $c'$ (e.g., contexts similar to $c$).}  
{More precisely, for every action $i \in [K]$, we specify a directed graph $G_i$ over the set of contexts $[C]$. An edge $c \rightarrow c'$ in $G_i$ indicates that playing action $i$ in context $c$, reveals  the reward of action $i$ in context $c'$. We assume that all self-loops $c \rightarrow c$ are present in all graphs $G_i$ (i.e., playing action $i$ in context $c$, reveals the reward of action $i$ in context $c$). We refer to these graphs as cross-learning (CL) graphs, and we assume that the CL graphs are known to the learner. }
 
Throughout the paper, we pay a special attention to two particular cases of contextual bandits with cross learning: (i) contextual bandits with no cross-learning and (ii)  contextual bandits with {complete cross-learning}. In the former, the graphs $G_i$ for every $i\in [K]$ only contain self-loops. In the latter, the graphs $G_i$ are fully connected, that is, there is an edge between every pair of contexts. When we are not in either of these two special cases, we  refer to the setting as  contextual bandits with {partial cross-learning.}
Figure \ref{fig:CL_graph} depicts three examples of CL graphs: (i) a CL graph with three contexts ($C= 3$) and no cross-learning between contexts (see Figure \ref{fig:no_CL}), (ii) a CL graph with three contexts ($C= 3$) and complete cross-learning between contexts (see Figure \ref{fig:full_CL}), and (iii) a CL graph with partial cross-learning (see Figure \ref{fig:partial}). 
In Figure \ref{fig:partial}, contexts are ordered and cross learning happens  when $|c-c'|\le 1$, i.e., between consecutive contexts.

\begin{figure}[t]
  \centering
  \begin{subfigure}[b]{0.48\textwidth}
    \centering
    \includegraphics[width=0.5\textwidth]{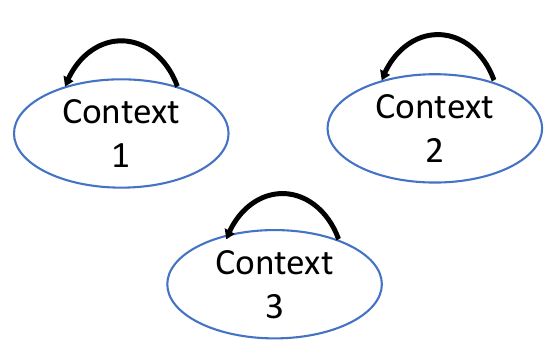}
    \caption{A CL graph with no cross-learning }
    \label{fig:no_CL}
  \end{subfigure}
  \begin{subfigure}[b]{0.48\textwidth}
    \centering
    \includegraphics[width=0.5\textwidth]{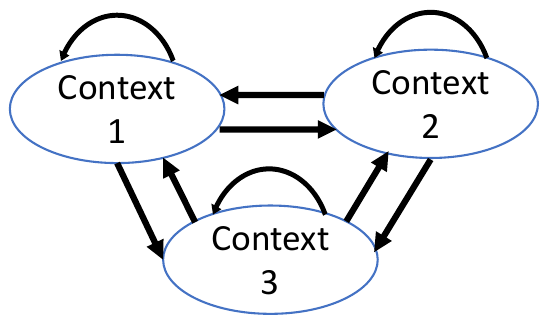}
    \caption{A CL graph with complete cross-learning }
    \label{fig:full_CL}
  \end{subfigure}\vspace{2em}
    \begin{subfigure}[b]{\textwidth}
    \centering
    \includegraphics[width=0.65\textwidth]{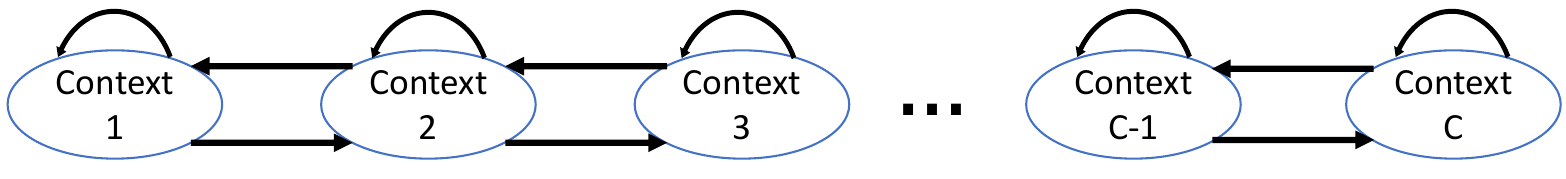}
    \caption{A CL graph with partial cross-learning }
    \label{fig:partial}
  \end{subfigure}
  \caption{Examples of CL graphs \label{fig:CL_graph}}
\end{figure}

{\color{black} In most of the applications we consider in this work (see Section \ref{sect:applications}), one can consider complete cross-learning between contexts. {For example, in the problem of bidding in non-truthful auctions, if a bidder wins an auction with a certain bid, it is generally possible to evaluate the counterfactual outcome ``how much utility would have been obtained if the valuation was different but the bid was the same.'' Such counterfactual outcomes can be estimated if, for the same bid, a change in the bidder's value would not affect the bids of competitors. In many applications, however, such counterfactual outcomes cannot be estimated for all contexts but, instead, for contexts that are ``close" to each other in some sense. In the bidding example, it is reasonable to assume that small changes in the bidder's valuation should not drastically impact the competitive landscape. Therefore, partial cross-learning can be used to model conservative learners that only use local information obtained from cross-learning.} Such a conservative
learner might be concerned about inaccuracies in their cross-learning model due to a significant
differences in the contexts.}

{\color{black}We highlight that in defining our regret bounds for the contextual bandit setting with cross-learning between contexts, we use our stationary benchmarks defined in Section \ref{sect:contextual_prelims}. 
We finish this section by presenting graph invariants. These invariants appear in our analysis of our algorithms, as well as, our regret bounds.} 
%{\color{red} Is this the best place to put the section on ``Graph Invariants"?}

\subsubsection{Graph Invariants} \label{sec:graph}

  Throughout the remainder of this section, we will assume that all graphs $G$ are directed and contain all self-loops. Given a vertex $v$ in $G$, let $\cal I (v)$ denote the set of in-neighbors, i.e., the set of  vertices $w$ such that there exists an edge $w \rightarrow v$, and let $\cal O(v)$ denote the set of vertices of out-neighbors, i.e., the set of vertices $w$ such that there exists an edge $v \rightarrow w$. Because all our graphs contain self-loops, $v \in \cal O(v)$ and $v \in \cal I(v)$. 

We define the following quantities of graph $G$: (i)  clique covering number, denoted by $\kappa(G)$, (ii)  independence number, denoted by  $\iota(G)$, and (iii) maximum acyclic subgraph number, denoted by  $\lambda(G)$. These quantities will appear in our regret bounds. We further define another metric, denoted by $\nu_2(G)$, which {can be thought of as an ``$L_2$ variant'' of $\lambda(G)$ (see Lemma \ref{lem:mas_equiv})}. This function will appear in our lower bounds, provided in Section \ref{sec:lb_partial}. Finally, at the end of this section, we present a lemma comparing these quantities. To wit, we show that for any graph $G$, we have $\iota(G) \leq \nu_2(G) \leq \lambda(G) \leq \kappa(G)$.

\begin{definition}[Clique Covering Number]\label{def:clique}
A \emph{subclique} of a graph $G$ is a subset of vertices $S$ such that for any two vertices $u, v \in S$, there exists an edge $u \rightarrow v$. A \emph{clique cover} of a graph $G$ is a partition of its set of vertices into subcliques $S_1, S_2, \dots, S_r$ (we say $r$ is the size of the clique cover). The \emph{clique covering number} $\kappa(G)$ is the minimum size of a clique cover of $G$.
\end{definition}

For CL graphs associated with complete cross-learning (see Figure \ref{fig:full_CL}), $\kappa(G)=1$. For CL graphs associated with no cross-learning (see Figure \ref{fig:no_CL}), $\kappa(G)=C$.  For the CL graph in Figure \ref{fig:CL_graph}, $\kappa(G)$ is $\lceil C/2 \rceil$.

\begin{definition}[Independence Number]
An \emph{independent set} in a graph $G$ is a subset of vertices $S$ such that for any two distinct vertices $u, v \in S$, the edge $u \rightarrow v$ does not exist in $G$. The \emph{independence number} $\iota(G)$ is the maximum size of an independent set of $G$. 
\end{definition}

Again, it is easy to see that for the CL graphs in Figures \ref{fig:no_CL}, \ref{fig:full_CL}, and \ref{fig:partial}, $\iota(G)$ is $C$, $1$, and $\lceil C/2 \rceil$, respectively.

\begin{definition}[Maximum Acyclic Subgraph Number] \label{def:acyclic}
An \emph{acyclic subgraph} of a graph $G$ is a set of vertices that can be ordered $v_1, v_2, \dots, v_r$ such that for any $i > j$, there is no edge $v_i \rightarrow v_j$. The \emph{maximum acyclic subgraph number} $\lambda(G)$ is the size of the largest acyclic subgraph of $G$.
\end{definition}

We note that for the CL graphs in  in Figures \ref{fig:no_CL}, \ref{fig:full_CL} and \ref{fig:partial}, $\lambda(G)$ is again $C$, $1$, and $\lceil C/2 \rceil$, respectively. The following lemma sheds light on the maximum acyclic subgraph number $\lambda(G)$.

\begin{lemma}\label{lem:mas_equiv}
For all directed graphs $G$ with self-loops,
$$\lambda(G) = \sup_{f : V \rightarrow \R^{+}} \sum_{v \in V} \frac{f(v)}{\sum_{w \in \cal I(v)} f(w)}\,.$$ 
\end{lemma}

Next, we define another graph quantity  and compare all the quantities that we have defined so far.

\begin{definition}\label{def:nu_2}
The value $\nu_2(G)$ of a graph $G$ (with vertex set $V$) is given by

$$\nu_2(G) = \sup_{\substack{f : V \rightarrow \R^{+} \\ \sum_{v\in V} f(v) = 1}} \left(\sum_{v \in V} \frac{f(v)}{\sqrt{\sum_{w \in \cal I(v)} f(w)}}\right)^2.$$ 

\end{definition}

 Comparing the definition of 
$\nu_2(G)$ with $\lambda(G)$ in Lemma \ref{lem:mas_equiv} confirms that $\nu_2(G)$ can be seen as $L_2$ variant of the maximum acyclic subgraph number $\lambda(G)$. For the CL graphs   in Figures \ref{fig:no_CL}, \ref{fig:full_CL} and \ref{fig:partial}, $\nu_2(G)$ is again $C$, $1$, and $\lceil C/2 \rceil$, respectively. 
The result for the CL graph in Figure \ref{fig:partial}  can be obtained by choosing $f(v) = 1/\lceil C/2 \rceil$ for all $v$ in an independent set of $G$ in Definition \ref{def:nu_2}; alternatively, it follows from the fact that $\iota(G) = \kappa(G) = \lceil C/2 \rceil$; see the following lemma.

\begin{lemma}[Comparing Graph Quantities]\label{lem:orderlemma}
For all graphs $G$,
$$\iota(G) \leq \nu_2(G) \leq \lambda(G) \leq \kappa(G).$$
Furthermore, when $G$ is the union of $r$ disjoint cliques, equality holds for all inequalities and all invariants equal $r$. 
\end{lemma}

Lemma \ref{lem:orderlemma} shows that for any graph $G$, its  independence number  $\iota(G)$ is smaller than its $L_2$ variant of the maximum acyclic number, i.e.,  $\nu_2(G)$, and the latter is smaller than or equal to the maximum acyclic number of the graph, i.e.,  $\lambda(G)$. Finally, $\lambda(G)$ is smaller than or equal to the clique covering number of the graph $\kappa(G)$. These inequalities will help us compare our regret bounds under the  cross learning; see Section \ref{sec:lb_partial}.  In Lemma \ref{lem:orderlemma}, we further show that if graph $G$ is the union of $r$ disjoint cliques, then  $\iota(G) = \nu_2(G) = \lambda(G) = \kappa(G)= r$.

\section{Algorithms for Cross-Learning Between Contexts}\label{sect:crosslearning}

{\color{black}In this section, we present three algorithms for the contextual bandits problem with cross-learning. The first algorithm called  \ucbcross{} is designed for  settings with stochastic rewards and adversarial contexts (Section \ref{sect:ucb1cross}). The two other algorithms  called \expcrossk{} and \expcrossu{} are designed for settings with adversarial rewards and stochastic contexts. While \expcrossk{} (Section \ref{sect:exp3cross2}) has full knowledge of the context distribution, \expcrossu{} (Section \ref{sec:unknown_dist}) does not have this knowledge. (``U" in \expcrossu{} stands for ``Unknown" context distribution.)}
 Then, in Section \ref{sect:advlb}, we show that {even with complete cross-learning,} it is impossible to achieve regret better than $\tilde{O}(\sqrt{CKT})$ when both rewards and contexts are controlled by an adversary (in particular, when both rewards and contexts are adversarial, cross-learning may not be beneficial at all).

\subsection{ \ucbcross{}  Algorithm for Stochastic Rewards}\label{sect:ucb1cross}

In this section, we present a low-regret algorithm, called \ucbcross{}, for the contextual bandits problem with  cross-learning  in the stochastic reward setting: i.e., every reward $r_{i,t}(c)$ is drawn independently from an unknown distribution $\F_i(c)$ supported on $[0,1]$.   Importantly, this algorithm works even when the contexts are chosen adversarially, unlike our algorithms for the adversarial reward setting. We call this algorithm \ucbcross{} (Algorithm \ref{alg:UCB1PC}). For simplicity, we will begin by describing \ucbcross{} and the intuition behind it in the complete cross-learning setting (when all the graphs $G_i$ are the complete directed graph). We will then describe how to modify it for the case of partial cross-learning.

\subsubsection{\ucbcross{} Algorithm for Complete Cross-learning} The \ucbcross{} algorithm is a generalization of $S$-UCB1; both algorithms maintain a mean and upper confidence bound for each action in each context, and always choose the action with the highest upper confidence bound. The difference being that \ucbcross{} uses cross-learning to update the means and confidence bounds for every context.  Namely, when arm $i$ is pulled in a round under context $c$, we update the means and confidence bounds for arm $i$ in every other context $c'$. (Recall that we under complete cross-learning, the CL graphs are complete directed graphs.)

	\begin{policy}[htb]
\begin{center}
\fbox{ 
\begin{minipage}{0.9\textwidth}
{
\parbox{\columnwidth}{ \vspace*{3mm}
\setlength{\parskip}{0.5em}
\textbf{\ucbcross{} Algorithm}
\begin{itemize}[itemsep=0.25em]
    \item  Define the function $\omega(s) = \sqrt{(2 \log T) / s}$ {(with $\omega(0) = +\infty$)}.
    	\item {For all $i \in [K]$ and $c \in [C]$, maintain a counter $\tau_{i,t}(c)$ equal to the number of times we have observed the reward of arm $i$ in context $c$ up to round $t$, i.e., {\color{black}$\tau_{i,t}(c) = \sum_{t'=1}^{t-1} \mathbb{I}(I_{t'} =i, c\in \mathcal O_i(c_{t'}))$.} }
    	\item {For all $i \in [K]$ and $c \in [C]$, maintain a running total $\sigma_{i, t}(c)$ of the total reward observed from arm $i$ in context $c$ up to round $t$, i.e., {\color{black}$\sigma_{i,t}(c) = \sum_{t'=1}^{t-1} r_{i,t'}(c) \mathbb{I}(I_{t'} =i, c\in \mathcal O_i(c_{t'}))$.}} 
    	\item Write $\overline{r}_{i,t}(c)$ as shorthand for $\sigma_{i,t}(c)/\tau_{i,t}(c)$. If $\tau_{i, t}(c) = \sigma_{i, t}(c) = 0$, let $\overline{r}_{i,t}(c) = 0$.
    	
    	\item  For {$t=1$ to $T$},
    	\begin{itemize}[itemsep=0.25em]
    	    \item  Receive context $c_t$.
        	\item Let $I_t$ be the arm which maximizes $\overline{r}_{I_t,t-1}(c_t) + \omega(\tau_{I_t, t-1})$, breaking ties arbitrarily. {\color{black} That is, $I_t\in \arg\max_{i\in [K]} \overline{r}_{i,t-1}(c_t) + \omega(\tau_{i, t-1})$.}
			\item Pull arm $I_t$, receiving reward $r_{I_t, t}(c_t)$, and learning the value of $r_{I_t, t}(c)$ for all $c \in \cal O_{I_t}(c_t)$.
			\item For {each $c$ in $\cal {O}_{I_t}(c_t)$}, set $\tau_{I_t,t}(c) = \tau_{I_t, t-1}(c) + 1$. {\color{black}For each $c \notin \cal {O}_{I_t}(c_t)$, set $\tau_{I_t,t}(c) = \tau_{I_t, t-1}(c)$. For each $i\in [K] \setminus \{I_t\}$ and $c\in [C]$, set $\tau_{i,t}(c) = \tau_{i, t-1}(c)$.}
			\item For {each $c$ in $\cal {O}_{I_t}(c_t)$}, set $\sigma_{I_t,t}(c) = \sigma_{I_t, t-1}(c) + r_{I_t, t}(c)$. {\color{black}For each $c \notin \cal {O}_{I_t}(c_t)$, set $\sigma_{I_t,t}(c) = \sigma_{I_t, t-1}(c)$. For each $i\in [K] \setminus \{I_t\}$ and $c\in [C]$, set $\sigma_{i,t}(c) = \sigma_{i, t-1}(c)$.}
    	\end{itemize}
\end{itemize}
}
}
\vspace{0.25em}
\end{minipage}
}
\end{center}  
\vspace{0.3cm}
\caption{$O(\sqrt{\overline{\kappa}KT\log T})$ regret algorithm (\ucbcross{}) for the contextual bandits problem with  cross-learning where rewards are stochastic. Here, $\overline{\kappa} = \frac{1}{K}\sum_{i \in [K]} \kappa(G_i)$, and $\kappa(G_i)$, which  is the clique covering number of the CL graph $G_i$, is defined in Definition \ref{def:clique}.}
\label{alg:UCB1PC}
\end{policy}

{While there are similarities between $S$-UCB1 and  \ucbcross{}, the analysis of \ucbcross{} requires new ideas to deal with the fact that the observations of rewards may be drawn from a very different distribution than the desired exploration distribution. Very roughly, the analysis is structured as follows. Since rewards are stochastic, in every context $c$, there is a ``best arm'' $i^{\star}(c)$ that the optimal policy always plays. Every other arm $i$ is some amount $\Delta_{i}(c)$ worse in expectation than the best arm. Here,   $\Delta_{i}(c) = \mu_{i^{\star}(c)}(c)-\mu_i(c)$ and $\mu_i(c)$ is the average reward of arm $i$ under context $c$. After observing this arm $m_i(c) \approx \log(T)/\Delta_{i}^2(c)$ times, one can be confident that this arm is not the best arm. We can decompose the regret into the regret incurred ``before" and ``after" the algorithm is confident that an arm is not optimal in a specific context. The regret after the algorithm is confident, which we call ``post-regret," can be bounded using standard techniques from the bandit literature.  Our main contribution is the bound of the regret that \ucbcross{} incurs before it gets confident. We refer to this regret as ``pre-regret."}

Consider a complete cross-learning setting and fix an arm $i$ and let $X_{i}(c)$ be the number of times the algorithm pulls arm $i$ in context $c$ before pulling arm $i$ a total of $m_i(c)$ times across all contexts.
Because once arm $i$ is pulled $m_i(c)$ times, we are confident about the optimality of pulling that arm in context $c$, we only need to control the number pulls before $m_i(c)$. Therefore, the pre-regret  of arm $i$ is roughly $\sum_{c=1}^CX_{i}(c)\Delta_{i}(c)$.

We control the pre-regret by setting up a linear program in the variables $X_{i}(c)$ with the objective of $\sum_{c\in [C], i\in [K]}X_{i}(c)\Delta_{i}(c)$.
Because $X_i(c)$ counts all pulls of arm $i$ before $m_i(c)$, we have that $X_i(c) \le m_i(c)$. This inequality, while valid, does not lead to a tight regret bound. To obtain a tighter inequality, we first sort the contexts in terms of the samples needed to learn whether an arm is optimal, i.e., in increasing order of $m_i(c)$. Because a different context is realized in every round, we can consider the inequality $\sum_{c':m_i(c') \le m_i(c)} X_i(c') \le m_i(c)$, which counts the subset of first $m_i(c)$ pulls of arm $i$. This set of inequalities give us our desired bound for pre-regret. Under partial cross-learning, we have similar set of inequalities. These inequalities, however,  are slightly modified to incorporate the CL graph under partial cross-learning.     

\subsubsection{\ucbcross{} Algorithm for Partial  Cross-learning} In the case of partial cross-learning, we can run almost the same algorithm as described above, with the slight change that instead of updating our statistics for arm $i$ in every context, we only update these statistics for the contexts which we learn about -- i.e., exactly the contexts in $\mathcal{O}_i(c)$ (recall that an edge $c \rightarrow c'$ in $G_i$ indicates that playing action $i$ in context $c$, reveals the reward of action $i$ in context $c'$). Much of the same intuition applies to the analysis as well, with the change that now the solution of the linear program we obtain will depend on properties of the graphs $G_i$ (specifically, their clique numbers; see Theorem \ref{thm:ucbcross}).

\subsubsection{Regret of \ucbcross{} Algorithm} Let $\overline{\kappa} = \frac{1}{K}\sum_{i \in [K]} \kappa(G_i)$ be the average clique cover size of all graphs $G_i$. Recall that $\kappa(G)$ is the clique covering number of graph $G$ and is defined in Definition \ref{def:clique}.
In Theorem  \ref{thm:ucbcross}, we show that  algorithm \ucbcross{} incurs at most $O(\sqrt{\overline{\kappa}KT\log T})$ regret. Observe that when we have complete cross-learning, the CL graph $G_i$, $i\in [K]$,  is a complete graph with $\kappa(G_i) =1$. Thus, Theorem \ref{thm:ucbcross} gives the regret of $O(\sqrt{KT\log T})$ for the case of complete cross-learning. Observe that the dependency on the number of contexts $C$ is completely removed in the regret bound. 
For the case of no cross-learning, the CL graph $G_i$, $i\in [K]$, only contains self-loops 
with $\kappa(G_i) =C$. Then, with  no  cross-learning, Theorem \ref{thm:ucbcross} gives the regret of $O(\sqrt{KCT\log T})$, as expected.

\begin{theorem}[Regret of \ucbcross{}]\label{thm:ucbcross} 
Let {\color{black}$\Delta^* = \min_{i\in [K], c\in [C]} (\mu^{{\star}}(c) - \mu_i(c))>0$} (where $\mu^{{\star}}(c) = \max_{i\in [K]} \mu_{i}(c)$). Then, \ucbcross{} (Algorithm \ref{alg:UCB1PC}) has an expected gap-dependent regret of $O\left(\frac{K\bar \kappa \log T}{\Delta^*}\right)$ and an expected  gap-independent regret of $O(\sqrt{\overline{\kappa}KT\log T})$ for the contextual bandits problem with  cross-learning in the setting with stochastic rewards and adversarial contexts, where $\overline{\kappa} = \frac{1}{K}\sum_{i \in [K]} \kappa(G_i)$, and $\kappa(G_i)$, which  is the clique covering number of the CL graph $G_i$, is defined in Definition \ref{def:clique}. 
\end{theorem}

To develop some intuition for Theorem~\ref{thm:ucbcross}, consider the case when $G$ is the union of $r$ disjoint cliques. In this case, all metrics (including $\kappa(G)$) presented in Lemma \ref{lem:orderlemma} of Section \ref{sec:graph} take value $r$. Because each component is disjoint and fully connected, it is possible to cross learn for all contexts within each clique, but contexts from
different cliques provide no information on each other. Hence, the
learning problem perfectly decouples across cliques and  leads to a regret of $O(\sqrt{r K T \log T})$. The clique
covering number measures the number of cliques needed to cover the
graph and thus captures to what extent we can cross learn on graphs
that do not perfectly decompose into disjoint cliques.

\iffalse
	\begin{policy}[t]
\begin{center}
\fbox{ 
\begin{minipage}{0.9\textwidth}
{
\parbox{\columnwidth}{ \vspace*{3mm} \label{alg:determin_boosts}
\setlength{\parskip}{0.5em}
\textbf{\ucbcross{} Algorithm}
 \begin{itemize}[itemsep=0.25em]
    \item Define the function $\omega(s) = \sqrt{(2 \log T) / s}$.
    
    \item Pull each arm $i \in [K]$ once (pulling arm $i$ in round $i$). 
    
    \item Maintain a counter $\tau_{i,t}$, equal to the number of times arm $i$ has been pulled up to round $t$ (so $\tau_{i,K} = 1$ for all $i$). 
    
    \item  For all $i \in [K]$ and $c \in [C]$, initialize variable $\sigma_{i,K}(c)$ to $r_{i,i}(c)$. Write $\overline{r}_{i,t}(c)$ as shorthand for $\sigma_{i,t}(c)/\tau_{i,t}$. 
    
    \item For {$t=K+1$ to $T$},
    \begin{itemize}[itemsep=0.25em]
        \item Receive context $c_t$.
        \item Let $I_t$ be  $\arg\max_{i\in [K]}\overline{r}_{i,t-1}(c_t) + \omega(\tau_{i, t-1})$. 
        
			\item Pull arm $I_t$, receiving reward $r_{I_t, t}(c_t)$, and learning the value of $r_{I_t, t}(c)$ for all $c$. 
				\item For each $c$ in $[C]$,  set $\sigma_{I_t,t}(c) = \sigma_{I_t, t-1}(c) + r_{I_t, t}(c)$.
			
    \end{itemize}
\end{itemize}
}
}\vspace{0.25em}
\end{minipage}
}
\end{center}  
\vspace{0.1cm}
\caption{$O(\sqrt{KT\log T})$ regret algorithm (\ucbcross{}) for the contextual bandits problem with cross-learning where rewards are stochastic and contexts are adversarial.}\label{alg:UCB1C}
\end{policy}
\fi

\subsubsection{Proof of Theorem \ref{thm:ucbcross}}
\begin{proof}{Proof of Theorem \ref{thm:ucbcross}} 
 We begin by defining the following notation. Let $\mu_{i}(c)$ be the mean of distribution $\F_i(c)$. Let $i^{\star}(c) = \arg\max_{j\in [K]} \mu_{j}(c)$, and let $\mu^{\star}(c) = \mu_{i^{\star}(c)}(c)$. Let $\Delta_{i}(c) = \mu^{\star}(c) - \mu_{i}(c)$ be the gap between the expected reward of playing arm $i$ in context $c$ and of playing the optimal arm $i^{\star}(c)$ in context $c$. Let $\tau'_{i,t}(c)${\color{black}$ = \sum_{t'=1}^{t-1} \mathbb{I}(I_{t'} =i, c_{t'}=c)$} be the number of times arm $i$ has been pulled in context $c$ up to round $t$. 
  Note that the regret $\Reg({\cal A})$ of our algorithm is then equal to
\begin{eqnarray}
\Reg({\cal A}) = \sum_{i=1}^{K}\sum_{c=1}^{C}\Delta_{i}(c)\tau'_{i,T}(c)\nonumber  
= \sum_{i=1}^{K}\sum_{c=1}^{C}\sum_{t=1}^{T}\Delta_{i}(c)\mathbbm{I}(I_{t}=i, c_{t}=c).%\\
%&\le& \Delta_{\min}T +\sum_{i=1}^{K}\sum_{c=1}^{C}\sum_{t=1}^{T}\Delta_{i}(c)\mathbbm{I}(I_{t}=i, c_{t}=c, \Delta_i(c)\ge \Delta_{\min})
\label{eq:reg1}
\end{eqnarray}

Fix a value $\Delta_{min}$ to be chosen later. Note that the sum of all terms in the above expression with $\Delta_i(c) \leq \Delta_{min}$ is at most $\Delta_{min}T$. We can therefore write

$$\Reg({\cal A}) \leq \Delta_{min}T + \sum_{i=1}^{K}\sum_{c=1}^{C}\sum_{t=1}^{T}\Delta_{i}(c)\mathbbm{I}(I_{t}=i, c_{t}=c, \Delta_{i}(c) \geq \Delta_{min}).$$
For convenience of notation, from now on, without loss of generality, we assume that all $\Delta_{i}(c) \geq \Delta_{min}$, and suppress the condition $\Delta_{i}(c) \ge \Delta_{min}$ in the indicator variables.
%In what follows, to ease the exposition, we hide the condition that $\Delta_i(c)\ge \Delta_{\min}$.

Now, define $m_{i}(c) = \frac{8\log T}{\Delta_{i}^2(c)}$. This quantity represents the number of times one must pull arm $i$ to observe that arm $i$ is not the best arm in context $c$ (we will show this later). Define (as in Algorithm \ref{alg:UCB1PC}) %{\color{red} Given my definition above ($\tau_{i,t}(c) = \sum_{t'=1}^{t-1} \mathbb{I}(I_{t'} =i, c\in \mathcal O_i(c_{t'}))$) this does not seem right.}
$$\tau_{i, t}(c) = \sum_{c' \in \cal I_i(c)} \tau'_{i, t}(c').$$
Note that $\tau_{i, t}(c)$, which can be written as $\tau_{i,t}(c) = \sum_{t'=1}^{t-1} \mathbb{I}(I_{t'} =i, c\in \mathcal O_i(c_{t'}))$, is equal to the number of times up to round $t$ we observe the reward of arm $i$ in context $c$. %{\color{red}Was not this the definition of $\tau_{i,t}$? ``For all $i \in [K]$ and $c \in [C]$, maintain a counter $\tau_{i,t}(c)$ equal to the number of times we have observed the reward of arm $i$ in context $c$ up to round $t$"} \jonnote{addressed these comments} 
We now define 
\begin{align}\label{eq:reg_pre_partial}\Reg_{\pre} &= \sum_{i=1}^{K}\sum_{c=1}^{C}\sum_{t=1}^{T}\Delta_{i}(c)\mathbbm{I}(I_{t}=i, c_{t}=c, \tau_{i,t}(c) \leq m_{i}(c)),\\
%We now divide the sum in (\ref{eq:reg1}) into two parts. Define:
%\begin{align}\label{eq:reg_pre}\Reg_{\pre} = \sum_{i=1}^{K}\sum_{c=1}^{C}\sum_{t=1}^{T}\Delta_{i}(c)\mathbbm{I}(I_{t}=i, c_{t}=c, \tau_{i,t} \leq m_{i}(c)),\end{align}
\label{eq:reg_post}\Reg_{\post} &= \sum_{i=1}^{K}\sum_{c=1}^{C}\sum_{t=1}^{T}\Delta_{i}(c)\mathbbm{I}(I_{t}=i, c_{t}=c, \tau_{i,t}(c) > m_{i}(c)).\end{align}
These two quantities represent the regret incurred before and after (respectively) the algorithm ``realizes'' an arm is not optimal in a specific context. We refer to $\Reg_{\pre}$ and $\Reg_{\post}$ as pre- and post-regret, respectively.  With these quantities, almost surely, we have %(\ref{eq:reg1}) almost surely as
\begin{equation} \label{eq:reg2}
\sum_{i=1}^{K}\sum_{c=1}^{C}\sum_{t=1}^{T}\Delta_{i}(c)\mathbbm{I}(I_{t}=i, c_{t}=c) ~=~  \Reg_{\pre} + \Reg_{\post}\,.
\end{equation}

In the following two lemmas, we will  bound the expected values of $\Reg_{\pre}$ and $\Reg_{\post}$.  In particular, the following lemma that bounds $\E[\Reg_{\pre}]$ is our main technical contribution in this proof.

\begin{lemma}[Bounding the Pre-regret] \label{lem:ucbpart1} %\nc{Suppose that $\min_{i, c} \mu^{{\star}}(c) - \mu_i(c)\ge \Delta_{\min}>0$.}
Let $\Reg_{\pre}$ be the quantity defined in (\ref{eq:reg_pre_partial}). Then, 
$$\E\left[\Reg_{\pre}\right] \leq \frac{16\log T \left(\sum_{i=1}^{K}\kappa(G_i)\right)}{\Delta_{min}}.$$
\end{lemma}
\begin{proof}{Proof of Lemma \ref{lem:ucbpart1}}
Fix an action $i$, and let $S_1, S_2, \dots, S_{\kappa(G_i)}$ be a minimal clique covering of the graph $G_i$. Let $r(c)$ equal the value of $r$ such that $c \in S_r$. For each $r \in [\kappa(G_i)]$, define %{\color{red} As stated above, let's make sure the definition of $\tau$ is accurate}
$$\tilde{\tau}_{i, t}(r) = \sum_{c \in S_r}\tau'_{i,t}(c).$$
The quantity $\tilde{\tau}_{i, t}(r)$ can be thought of the number of times we play arm $i$ in a context belonging to clique $r$. Note that since $S_{r}$ is a clique, this is also (a lower bound on) the number of times we observe the reward of all the arms in clique $r$.

Note that for all $c$, $S_{r(c)} \subseteq \cal I_{i}(c)$ (since $S_{r(c)}$ is a clique, all contexts in $S_{r(c)}$ have an edge leading to $c$). It follows that $\tilde{\tau}_{i, t}(r(c)) \leq \tau_{i,t}(c)$. (Under complete cross-learning, $\tilde{\tau}_{i, t}(r(c)) = \tau_{i,t}(c)$.) Now, define $X_i(c)$ as
\begin{equation*}
X_i(c) = \sum_{t=1}^{T} \mathbbm{I}(c_t = c, I_t = i, \tilde{\tau}_{i, t}(r(c)) \leq m_i(c)). 
\end{equation*}
The quantity $X_i(c)$ can be thought of as the number of times action $i$ is played during context $c$ before the $m_i(c)$th time we have observed the payoff of action $i$ in this context (and the other contexts of $S_{r(c)}$).
Note that since $\tilde{\tau}_{i, t}(r(c)) \leq \tau_{i,t}(c)$, it is true that
$$\mathbbm{I}(c_t=c, I_t = i, \tilde{\tau}_{i,t}(r(c)) \leq m_i(c)) \geq \mathbbm{I}(c_t=c, I_t = i, \tau_{i,t}(c) \leq m_i(c)),$$
and therefore
$$\sum_{c=1}^{C}\sum_{t=1}^{T}\Delta_{i}(c)\mathbbm{I}(I_{t}=i, c_{t}=c, \tau_{i,t}(c) \leq m_{i}(c)) \leq \sum_{c=1}^{C}\Delta_{i}(c)X_i(c).$$

 Fix an $r$, and order the contexts in $S_r$ given by $c_{(1)}, c_{(2)}, \dots, c_{(|S_r|)}$ so that $\Delta_{i}(c_{(1)}) \geq \Delta_{i}(c_{(2)}) \geq \dots \geq \Delta_{i}(c_{(|S_r|)})$ (and thus $m_i(c_{(1)}) \leq m_i(c_{(2)}) \leq \dots \leq m_{i}(c_{(|S_r|)})$). From the ordering of the $m_i(c_{(j)})$, we have the following system of inequalities:
\begin{align}
X_i(c_{(1)}) &\leq  m_{i}(c_{(1)}) \nonumber  \\
X_i(c_{(1)}) + X_i(c_{(2)}) &\leq  m_{i}(c_{(2)}) \nonumber\\
&\vdots &\nonumber \\
X_i(c_{(1)}) + X_i(c_{(2)}) + \dots + X_i(c_{(|S_r|)}) &\leq  m_{i}(c_{(|S_r|)})\,. \label{eq:Xp}
\end{align}
To see why the above inequalities hold, for simplicity, focus on the second inequality (the same argument can be applied for other inequalities). First note that by the fact that $m_{i}(c_{(1)}) \leq m_{i}(c_{(2)})$, we have 
\[
X_{i}(c_{(1)})+ X_{i}(c_{(2)}) \le  \sum_{t=1}^{T}\mathbbm{I}(c_{t} = c_{(1)}, I_t = i, \tilde{\tau}_{i,t} (r) \leq m_{i}(c_{(2)}))+  \sum_{t=1}^{T}\mathbbm{I}(c_{t} = c_{(2)}, I_t = i, \tilde{\tau}_{i,t} (r) \leq m_{i}(c_{(2)}))\,.
\]
(Recall that $\tilde{\tau}_{i, t}(r) = \sum_{c \in S_r}\tau'_{i,t}(c)$.) This implies that 
\[
X_{i}(c_{(1)})+ X_{i}(c_{(2)})  ~\le~  \sum_{t=1}^{T}\mathbbm{I}((c_{t} = c_{(1)} \text{~or~} c_{t} = c_{(2)}), I_t = i, \tilde{\tau}_{i,t} (r)  \leq m_{i}(c_{(2)})).
\]

The right hand side of the above inequality is in turn at most $m_{i}(c_{(2)})$, since each time $c_{t} = c_{(1)}$ or $c_{(2)}$ and $I_{t} = i$, $\tilde{\tau}_{i,t}(r)$ increases by 1. This proves the second inequality in \eqref{eq:Xp}; the other inequalities follow similarly.

Now, we wish to bound $\sum_{j=1}^{|S_r|} \Delta_{i}(c_{(j)})X_{i}(c_{(j)})$. To do this, multiply the $j$th inequality in Equation (\ref{eq:Xp})  by $\Delta_{i}(c_{(j)}) - \Delta_{i}(c_{(j+1)})$ (for the last inequality, just multiply it through by $\Delta_{i}(c_{(|S_r|)})$), and sum all of these inequalities to obtain
\begin{eqnarray*}
\sum_{j=1}^{|S_r|}\Delta_{i}(c_{(j)})X_{i}(c_{(j)})&\leq & \Delta_{i}(c_{(|S_r|)})m_{i}(c_{(|S_r)|)}) + \sum_{j=1}^{|S_r|-1}(\Delta_{i}(c_{(j)}) - \Delta_{i}(c_{(j+1)}))m_i(c_{(j)}) \\
&=& 8\log T \left(\frac{1}{\Delta_{i}(c_{|S_r|})} + \sum_{j=1}^{|S_r|-1}\frac{\Delta_{i}(c_{(j)}) - \Delta_{i}(c_{(j+1)})}{\Delta_i(c_{(j)})^2}\right) \\
&\leq & 8\log T \left(\frac{1}{\Delta_{\min}} + \int_{\Delta_{\min}}^{1}\frac{dx}{x^2}\right) \\
&\leq & \frac{16 \log T}{\Delta_{\min}}\,,
\end{eqnarray*}
where the second equation follows because  $m_{i}(c) = \frac{8\log T}{\Delta_{i}^2(c)}$, and  the third equation holds because $ \Delta_{i}(c_{j}) \geq \Delta_{\min}$ for any $j\in [|S_r|]$.  
Therefore, summing over all $r \in [\kappa(G_i)]$, we have that
$$\sum_{c=1}^{C}\Delta_i(c)X_i(c) \leq \frac{16\kappa(G_i)\log T}{\Delta_{min}}.$$
Finally, summing over all actions $i$, we have that

$$\Reg_{\pre} \leq \frac{16\log T}{\Delta_{min}}\left(\sum_{i=1}^{K}\kappa(G_i)\right).\Halmos$$ 
\end{proof}

We next proceed to bound the expected value of $\Reg_{\post}$.  This follows from the standard analysis of UCB1. A proof is provided in the appendix.

\begin{lemma}[Bounding the Post-regret]\label{lem:ucbpart2}
Let $\Reg_{\post}$ be the quantity defined in (\ref{eq:reg_post}). Then, 
$$\E\left[\Reg_{\post}\right] \leq \frac{K\pi^2}{3}\,.$$
\end{lemma}

Substituting the results of Lemmas \ref{lem:ucbpart1} and \ref{lem:ucbpart2} into (\ref{eq:reg2}) with $\Delta_{\min}=\Delta^*$, we obtain
\begin{equation} \label{eq:reg4}
\E[\Reg({\cal{A}})] \leq \frac{16 K\log T}{\Delta^*} \left(\sum_{i=1}^{K}\kappa(G_i)\right)+ \frac{K\pi^2}{3} = O\left(\frac{K\log T}{\Delta^*} \bar \kappa \right)\,,
\end{equation}
where the inequality holds because by definition, $\bar \kappa =\frac{1}{K} \sum_{i=1}^{K}\kappa(G_i)$. 
This proves our gap-dependent regret bound. 
For the gap-independent bound, observe that $\E[\Reg({\cal{A}})] \le \Delta_{\min}T+\Reg_{\pre} + \Reg_{\post}$. 
Then, one can still apply the 
 the results of Lemmas \ref{lem:ucbpart1} and \ref{lem:ucbpart2} to bound $\Reg_{\pre}$ and $\Reg_{\post}$. By doing so, we obtain
\begin{equation*}
\E[\Reg({\cal{A}})] \leq \Delta_{min}T + \frac{16 K\overline{\kappa}\log T}{\Delta_{min}} + \frac{K\pi^2}{3}.
\end{equation*}

Substituting in $\Delta_{min} = \sqrt{\overline{\kappa}K\log T / T}$, it is straightforward to verify that $\E[\Reg({\cal{A}})] \leq O(\sqrt{\overline{\kappa}KT\log T})$, as desired.

\end{proof}

\medskip

\subsection{Algorithms for Adversarial Rewards and Stochastic Contexts}
Here, we consider  contextual bandits problem with cross learning when the rewards are adversarially chosen but contexts are stochastically drawn from some distribution $\D$.   {\color{black}We consider two settings: in the first setting, 
  the learner knows the distribution over contexts $\D$ and in the second setting, the context distribution is unknown.}

\subsubsection{\expcrossk{} Algorithm for Known Context Distribution}\label{sect:exp3cross2} 
In this section, we present an algorithm called \expcrossk{} (Algorithm \ref{alg:exp3PC}). 
Similar to $S$-EXP3, discussed in the introduction,   \expcrossk{}  maintains a weight for each action in each context, and updates the weights via multiplicative updates by an exponential of an unbiased estimator of the reward. There are two main differences between  $S$-EXP3  and  \expcrossk{}. First, 
 while $S$-EXP3 only updates the weight of the chosen action for the current context (i.e., $w_{I_t, t}(c_t)$), \expcrossk{} uses the information from cross-learning to update the weight of the chosen action for more contexts. Precisely, suppose that \expcrossk{} plays arm $I_t$  in round $t$ under context $c_t$. Then, \expcrossk{}  updates the weight of any context $c$ in  $\cal O_{I_t}(c_t)$. (For the case of complete cross-learning, the weight of all contexts $c$ are updated.)

Second, 
 to take advantage of the information from cross-learning, \expcrossk{} modifies $S$-EXP3 by changing the unbiased estimator in the update rule. For $S$-EXP3,  $\hat{r}_{i,t}(c) = (r_{i,t}(c)/p_{i,t}(c_t))\mathbbm{I}(I_t = i)$ is an unbiased estimator (over the algorithm's randomness) of the reward the adversary chooses from pulling arm $i$ in context $c$, where $p_{i,t}(c)$ is the probability the algorithm chooses action $i$ in round $t$ if the context is $c$. However, to minimize regret, \expcrossk{}   chooses an unbiased estimator with minimal variance (as the expected variance of this estimator shows up in the final regret bound). The new estimator in question is
\begin{align}\label{eq:estimator}\hat{r}_{i, t}(c) = \frac{r_{i,t}(c)}{\sum_{c' \in \cal I_i(c)} \Pr[c'] \cdot p_{i, t}(c')}\mathbbm{I}(I_t = i, c_t \in \cal I_i(c))\,.\end{align} 
It is easy to see that this is an unbiased estimator of the reward of arm $i$ under context $c$ (see the proof of Theorem \ref{thm:exppcross}). Observe that in the denominator of the estimator, we only consider contexts $c' \in \cal I_i (c)$. This is because we only learn the reward of arm $i$ under context $c$ when this arm is played under  context $c' \in \cal I_i (c)$. Note that by definition, $c \in \cal I_i (c)$.

Let $\bar{\lambda} = \frac{1}{K}\sum_{i \in [K]}\lambda(G_i)$ be the average size of the maximum acyclic subgraph over all graphs $G_i$. We will show that \expcrossk{} obtains at most $O(\sqrt{\overline{\lambda}KT \log(K)})$ regret. Then, with complete cross-learning,  Theorem \ref{thm:exppcross} provides regret of $O(\sqrt{KT \log(K)})$, completing removing the dependency on the number of contexts $C$. This is because under complete cross-learning, CL graphs are complete graphs and their  maximum acyclic subgraph number $\lambda(G)$ is one. In addition, as expected, with no cross-learning, Theorem \ref{thm:exppcross} provides regret of $O(\sqrt{KT C \log(K)})$.

\begin{policy}[htb]
\begin{center}
\fbox{ 
\begin{minipage}{0.9\textwidth}
{
\parbox{\columnwidth}{ \vspace*{3mm}
\setlength{\parskip}{0.5em}
\textbf{\expcrossk{} Algorithm}
\begin{itemize}[itemsep=0.25em]
    \item   Choose $\alpha = \beta = \sqrt{\frac{\log K}{\overline{\lambda}KT}}$ (where $\overline{\lambda} = \frac{1}{K}\sum_{i=1}^{K}\lambda(G_i)$).  
    	\item Initialize $K\cdot C$ weights, one for each pair of action $i$ and context $c$, letting $w_{i, t}(c)$ be the value of the $i$th weight for context $c$ at round $t$. Initially, set all $w_{i,0} = 1$.
   
   \item For {$t=1$ to $T$},
   \begin{itemize}[itemsep=0.25em]
        	 \item Observe  context $c_t\sim \D$.
			 \item For all $i \in [K]$ and $c \in [C]$, let \[p_{i,t}(c) = (1-K\alpha)\frac{w_{i,t-1}(c)}{ \sum_{j=1}^{K}w_{j,t-1}(c)} + \alpha.\]
			 \item Sample an arm $I_{t}$ from the distribution $p_{t}(c_t)=(p_{i,t}(c_t))_{i\in [K]}$. 
			 \item Pull arm $I_t$, receiving reward $r_{I_t, t}(c_t)$, and learning the value of $r_{I_t, t}(c)$ for all $c \in \cal O_i(c_t)$. 
			 \item For {each $c$ in $\cal O_{I_t}(c_t)$}, 
			set \[w_{I_t,t}(c) = w_{I_t,t-1}(c) \cdot \exp\Big(\beta \cdot \frac{r_{I_t,t}(c)}{\sum_{c' \in \cal I_{I_t}(c)}\Pr[c']\cdot p_{I_t,t}(c')}\Big).\]
			\end{itemize}
\end{itemize}
}
}
\vspace{0.25em}
\end{minipage}
}
\end{center}  
\vspace{0.3cm}
\caption{$O(\sqrt{\overline{\lambda}KT\log K})$ regret algorithm (\expcrossk{}) for the contextual bandits problem with  cross-learning where rewards are adversarial and contexts are stochastic. Here, $\bar{\lambda} = \frac{1}{K}\sum_{i \in [K]}\lambda(G_i)$, and $\lambda(G_i)$, which is the maximum acyclic subgraph number of $G_i$, is defined in Definition \ref{def:acyclic}. }\label{alg:exp3PC}
\end{policy}

\begin{theorem}[Regret of \expcrossk{}]\label{thm:exppcross}
\expcrossk{} (Algorithm \ref{alg:exp3PC}) has regret $O(\sqrt{\overline{\lambda}KT\log K})$ for the contextual bandits problem with  cross-learning when rewards are adversarial and contexts are stochastic, where $\bar{\lambda} = \frac{1}{K}\sum_{i \in [K]}\lambda(G_i)$, and $\lambda(G_i)$, which is the maximum acyclic subgraph number of $G_i$, is defined in Definition \ref{def:acyclic}. 
\end{theorem}

{At a high level, the fact that \expcrossk{} obtains a good regret bound follows from the quality of the estimator used in this algorithm. Once we have this estimator, we follow the standard analysis of multiplicative weights/EXP3 algorithms. More specifically,} we use the sum of the weights $W_{t}(c) = \sum_{i=1}^{K}w_{i,t}(c)$ as a proxy for the regret bound of the \expcrossk{}. Roughly speaking, the higher the sum of the weights $W_{t}(c)$, the better \expcrossk{} has performed till time $t$ (see how the weights are updated in this algorithm).  In light of this, in the proof,  for any context $c$, we lower/upper bound the expected value of  $\log\left( W_{T+1}(c) / W_0(c) \right)$ as a function of mean and variance of the estimator in \expcrossk{}, as well as,  the expected reward earned by \expcrossk{}. Comparing the lower bound with  the upper bound gives us the desired regret bound.

\subsubsection{\expcrossu{} Algorithm for Unknown Context Distribution}\label{sec:unknown_dist}

In Section \ref{sect:exp3cross2}, we presented an algorithm called \expcrossk{} with  regret of  $O(\sqrt{\overline{\lambda}KT\log K})$ for the setting with adversarial rewards and  stochastic contexts. This algorithm uses the knowledge of  the context distribution $\mathcal D$ to come up with an unbiased  estimator with low variance. In this section, we relax the assumption of knowing distribution $\cal D$, and  we present an algorithm for the contextual bandits problem with cross-learning in the setting when rewards are adversarial and contexts are stochastic, but when the learner does not know the distribution $\D$ over contexts. We call this algorithm \expcrossu{} (see Algorithm \ref{alg:exp3C1}). In this algorithm, we  additionally assume all the CL graphs $G_i$, $i\in [K],$ are the same and equal to a single graph $G$ (we will see that this assumption is critical to constructing a \emph{consistently biased} estimator for the rewards; see our discussion about the expectation of our estimator, i.e., $\E[\hat{r}_{i, t}(c)]$ later in this section). Since we assume that $G_i=G$ for any $i\in [K]$, in the algorithm, we simply write $\mathcal{O}(c)$ and $\mathcal{I}(c)$ in place of $\mathcal{O}_i(c)$ and $\mathcal{I}_i(c)$).

\begin{policy}[htb]
\begin{center}
\fbox{ 
\begin{minipage}{0.9\textwidth}
{
\parbox{\columnwidth}{ \vspace*{3mm}
\setlength{\parskip}{0.5em}
\textbf{\expcrossu{} Algorithm }
\begin{itemize}[itemsep=0.25em]
    \item    Choose $\alpha = (\log K / K^2T)^{1/3}$, and $\beta = \sqrt{\frac{\alpha\log K}{T}}$.
    	\item Initialize $K\cdot C$ weights, one for each pair of action $i$ and context $c$, letting $w_{i, t}(c)$ be the value of the $i$th weight for context $c$ at round $t$. Initially, set all $w_{i,0} = 1$.
    	
    	\item For {$t=1$ to $T$}, 
    	\begin{itemize}[itemsep=0.25em]
    	    \item  Observe context $c_t \sim \D$.
			\item  For all $i \in [K]$ and $c \in [C]$, let $p_{i,t}(c) = (1-K\alpha)\frac{w_{i,t-1}(c)}{ \sum_{j=1}^{K}w_{j,t-1}(c)} + \alpha$.
			\item Sample an arm $I_{t}$ from the distribution $p_{t}(c_t)$. 
			\item Pull arm $I_t$, receiving reward $r_{I_t, t}(c_t)$, and learning the value of $r_{I_t, t}(c)$ for all $c$. 
			\item For {each $c$ in $\mathcal{O}(c_t)$}, set $w_{I_t,t}(c) = w_{I_t,t-1}(c) \cdot \exp\left(\beta \cdot \frac{r_{I_t,t}(c)}{p_{I_t, t}(c_t)}\right)$.
    	\end{itemize}
        
\end{itemize}
}
}
\vspace{0.25em}
\end{minipage}
}
\end{center}  
\vspace{0.3cm}
\caption{$O(\lambda K^{1/3}T^{2/3}(\log K)^{1/3})$ regret algorithm (\expcrossu{}) for the contextual bandits problem with cross-learning where rewards are adversarial and contexts are stochastic and
 the distribution $\D$ over contexts is unknown. We additionally assume all the CL graphs $G_i$, $i\in [K],$ are equal to a single graph $G$ (and hence simply write $\mathcal{O}(c)$ and $\mathcal{I}(c)$ in place of $\mathcal{O}_i(c)$ and $\mathcal{I}_i(c)$, respectively). Further,  $\lambda = \lambda(G)$ is the maximum acyclic subgraph number of $G$ (defined in Definition \ref{def:acyclic}).}\label{alg:exp3C1}
\end{policy}

\iffalse
\begin{policy}[htb]
\begin{center}
\fbox{ 
\begin{minipage}{0.9\textwidth}
{
\parbox{\columnwidth}{ \vspace*{3mm}
\setlength{\parskip}{0.5em}
\textbf{\expcrossu{} Algorithm}
\begin{itemize}[itemsep=0.25em]
    \item    Choose $\alpha = (\log K / K^2T)^{1/3}$, and $\beta = \sqrt{\frac{\alpha\log K}{T}}$.
    	\item Initialize $K\cdot C$ weights, one for each pair of action $i$ and context $c$, letting $w_{i, t}(c)$ be the value of the $i$th weight for context $c$ at round $t$. Initially, set all $w_{i,0} = 1$, $i\in [K]$.
    	
    	\item For {$t=1$ to $T$}, 
    	\begin{itemize}[itemsep=0.25em]
    	    \item  Observe context $c_t \sim \D$.
			\item  For all $i \in [K]$ and $c \in [C]$, let $p_{i,t}(c) = (1-K\alpha)\frac{w_{i,t-1}(c)}{ \sum_{j=1}^{K}w_{j,t-1}(c)} + \alpha$.
			\item Sample an arm $I_{t}$ from the distribution $p_{t}(c_t)$. 
			\item Pull arm $I_t$, receiving reward $r_{I_t, t}(c_t)$, and learning the value of $r_{I_t, t}(c)$ for all $c$. 
			\item For {each $c$ in $[C]$}, set $w_{I_t,t}(c) = w_{I_t,t-1}(c) \cdot \exp\left(\beta \cdot \frac{r_{I_t,t}(c)}{p_{I_t, t}(c_t)}\right)$.

    	\end{itemize}
        
\end{itemize}
}
}
\vspace{0.25em}
\end{minipage}
}
\end{center}  
\vspace{0.3cm}
\caption{$\tilde{O}(K^{1/3}T^{2/3})$ regret algorithm (\expcrossu{}) for the contextual bandits problem with cross-learning when the distribution $\D$ over contexts is unknown.}\label{alg:exp3C1}
\end{policy}
\fi

\expcrossu{} is similar to \expcrossk{}, in that all the algorithms maintain a weight for each action in each context, and update the weights via multiplicative updates by an exponential of an estimator of the reward. 
The main difference between \expcrossu{} and \expcrossk{} is their estimators. The estimator in \expcrossk{}, given in Equation (\ref{eq:estimator}), uses the knowledge of distribution  $\cal D$ while the estimator in  \expcrossk{}, which is defined below, cannot use such knowledge: 
\begin{align}\label{eq:estimator_2}%\hat{r}_{i,t}(c) = \frac{r_{i,t}(c)}{ p_{i,t}(c_t)} \cdot \mathbbm{I}({I_t = i}).
\hat{r}_{i, t}(c) = \frac{r_{i,t}(c)}{p_{i, t}(c_t)}\mathbbm{I}(I_t = i, c \in \mathcal{O}(c_t))\,.
\end{align}
We highlight that 
unlike the estimator in \expcrossk{}, $\hat{r}_{i, t}(c)$ is not an unbiased estimator of $r_{i, t}(c)$ for partial cross learning. (This estimator is only unbiased under complete cross-learning.)  Indeed, we have that: 
$$\E[\hat{r}_{i, t}(c)] = \sum_{c' \in \mathcal{I}(c)} \Pr[c'] p_{i,t}(c') \cdot \frac{r_{i, t}(c)}{p_{i, t}(c')} = \left(\sum_{c' \in \mathcal{I}(c)} \Pr[c']\right) r_{i, t}(c)\,.$$ 
 However, note that we can write $\E[\hat{r}_{i, t}(c)]$ in the form $f(c)r_{i, t}(c)$, where $f(c)$ is a function which only depends on a context (and in this case is given by $f(c) = \sum_{c' \in \mathcal{I}(c)} \Pr[c']$). That is, our estimator is consistently biased across all actions $i\in [K]$ for a fixed context $c$. It turns out this property is enough to adapt the previous analysis of Theorem \ref{thm:exppcross}.

Given this estimator, the question is why does \expcrossu{} have regret of order $T^{2/3}$ (see Theorem~\ref{thm:exp3cross}) when the dependence on $T$ in \expcrossk{}  is only of order $\sqrt{T}$? The answer lies in understanding how the variance of the unbiased estimator used affects the regret bound of the algorithm. In the standard analysis {of EXP3 (and \expcrossk{} and $S$-EXP3)}, one of the quantities in the regret bound is the \textit{total expected variance of the  estimator of rewards}. In $S$-EXP3,  
this quantity takes the form
$$\sum_{t=1}^{T}p_{i,t}(c_t)\E[\hat{r}_{i,t}(c)^2] = \sum_{t=1}^{T} \frac{p_{i,t}(c_t)}{p_{i,t}(c_t)}\hat{r}_{i,t}(c)^2 = \sum_{t=1}^{T}\hat{r}_{i,t}(c)^2 \leq T.$$
However, in \expcrossu{} (where the desired exploration distribution $p_{i,t}(c)$ can differ from the exploration distribution due to cross-learning), this quantity becomes 
$$\sum_{t=1}^{T} p_{i,t}(c)\E[\hat{r}_{i,t}(c)^2] = \sum_{t=1}^{T} \frac{p_{i,t}(c)}{p_{i,t}(c_t)}\hat{r}_{i,t}(c)^2 \leq \frac{T}{\min_{i\in [K],t\in [T]} p_{i,t}(c)}.$$ 
Optimizing $\min_{i\in [K], t\in [T]} p_{i,t}(c)$ (through selecting the parameter $\alpha$) leads to a $\tilde{O}(\lambda T^{2/3}K^{1/3})$ regret bound. In the case of complete cross-learning, $\lambda = 1$ and we have a $\tilde{O}(T^{2/3}K^{1/3})$ regret bound.

\begin{theorem}[Regret of \expcrossu{}]\label{thm:exp3cross} 
Consider the contextual bandits problem with  adversarial rewards, and stochastic contexts, where all the CL graphs $G_i$ are equal to $G$ and the context distribution $\cal D$ is unknown. Let $\lambda = \lambda(G)$ be the maximum acyclic subgraph number of $G$ (defined in Definition \ref{def:acyclic}). Then \expcrossu{} (Algorithm \ref{alg:exp3C1}) incurs at most $O(\lambda K^{1/3}T^{2/3}(\log K)^{1/3})$ regret in this setting.
\end{theorem}

{\color{black} Observe that the regret bound of both \expcrossu{} and \expcrossk{} algorithm   scales with the  maximum acyclic subgraph number of $G$ (i.e., $\lambda(G)$ defined in Definition \ref{def:acyclic}), where $\lambda(G) =1$ for complete cross-learning settings and $\lambda(G) =C$ when there is no cross-learning between contexts. However, as stated earlier, while the regret of \expcrossk{} scales with $O(\sqrt T)$, the regret of  \expcrossu{} scales with $O(T^{2/3})$.} 
An interesting open problem is determining whether it is possible to achieve $\tilde O(\sqrt{T})$ regret in the adversarial reward regime without knowing the distribution over contexts. 

{\color{black}We have explored this open problem by considering an extension of \expcrossk{} algorithm. For this discussion, we focus on the complete cross-learning setting.  The main feature of \expcrossk{} is  its low-variance unbiased estimator $\hat{r}_{i, t}(c) = (r_{i,t}(c) / D_{i,t} )\mathbbm{I}(I_t = i)$, where the denominator is $D_{i,t} = \sum_{c'=1}^C \Pr[c'] \cdot p_{i, t}(c')$. However, this unbiased estimator and, in particular, its denominator $D_{i, t}$ require knowledge of the distribution over contexts. Our idea is  to replace this unbiased estimator  with a similar estimator $\tilde{r}_{i, t}(c) = (r_{i,t}(c) / \tilde{D}_{i, t})\mathbbm{I}(I_t = i)$, where $\tilde{D}_{i, t}$ is some sufficiently close approximation to $D_{i, t}$ that does not require knowledge of the distribution over contexts. One natural choice is to replace the true probability $\Pr[c']$ of context $c'$ with the current empirical probability $\widehat{\Pr}[c']$ to get 
$
\tilde{D}_{i, t} = \sum_{c'=1}^C \widehat{\Pr}[c']p_{i, t}(c')\,.
$
 While our empirical evaluation in Section \ref{sect:experiments} provides some convincing evidence that the proposed approach works well, we were not able to theoretically prove that the described algorithm obtains regret on the order of $O(\sqrt T)$. In Appendix \ref{sec:discuss}, we discuss the main technical challenges we encountered for analyzing  this algorithm. Nonetheless, we conjecture that this proposed algorithm incurs at most $O(\sqrt T)$ regret. 
}

\subsection{Adversarial Rewards and Adversarial Contexts}\label{sect:advlb}

A natural question is whether we can design an algorithm whose regret is lower than that of $S$-EXP3 when both the rewards and contexts are chosen adversarially (but where we still can cross-learn between different contexts). A positive answer to this question would subsume the results of the previous sections. Unfortunately, we next show  that even under complete cross-learning, any learning algorithm for the contextual bandits problem with cross-learning must necessarily incur $\Omega(\sqrt{CKT})$ regret (which is achieved by $S$-EXP3).

We will need the following regret lower-bound for the (non-contextual) multi-armed bandits problem from \cite{AuerCNS03}.

\begin{lemma}[\citealt{AuerCNS03}]\label{lem:mab-lb}
There exists a distribution over instances of the multi-armed bandit problem where any algorithm must incur an expected regret of at least $\Omega(\sqrt{KT})$.
\end{lemma}

For concreteness, we describe one instance of a reward distribution  over instances (parametrized by $K$ and $T$), denoted by $\mathcal{D}_{K, T}$, that satisfies Lemma \ref{lem:mab-lb}. This distribution is the uniform distribution over $K$ instances under which  rewards of each arm are drawn from independent Bernoulli distributions. In the $i$th instance, the rewards of arm $i$ are drawn from $\mathrm{Bern}(1/2 + \sqrt{K/T})$ (i.e., a Bernoulli distribution with mean $1/2 + \sqrt{K/T}$), and the rewards of all other arms $j\ne i$ are drawn from $\mathrm{Bern}(1/2)$ (i.e., in $i$th instance, arm $i$ is the optimal arm and this arm outperforms all other arms by at least $\sqrt{K/T}$).

With this lemma, we can construct the following lower-bound for the contextual bandits problem with (complete) cross-learning by connecting $C$ independent copies of these hard instances in sequence with one another so that cross-learning between instances is not possible.

\begin{theorem}[Regret Lower Bound for Adversarial Rewards and Adversarial Contexts]\label{thm:advlb}
 There exists a distribution over instances of the contextual bandit problem with complete cross-learning where any algorithm must incur a regret of at least $\Omega(\sqrt{CKT})$.
\end{theorem}
\begin{proof}{Proof of Theorem \ref{thm:advlb}}

Divide the $T$ rounds into $C$ epochs of $T/C$ rounds each. Label the $C$ contexts $c_1, c_2, \dots, c_C$, and adversarially assign contexts so that the context during the $j$th epoch is always $c_j$.

Next, assign rewards so that $r_{i,t}(c) = 0$ if $t$ is in the $j$th epoch and $c \neq c_j$. In other words, during the $j$th epoch in which the context is $c_j$, the rewards of all other contexts are zero. On the other hand, for $t$ in the $j$th epoch, set rewards $r_{i,t}(c_j)$ according to a hard instance for the multi-armed bandit problem sampled from the distribution $\mathcal{D}_{K, T/C}$ satisfying Lemma \ref{lem:mab-lb}. Call this random instance $P_j$, and let $i_j$ be the (random) optimal action to play in $P_j$, i.e., the action with rewards drawn from $\mathrm{Bern}(1/2 + \sqrt{K/(T/C)})$. Figure~\ref{fig:lb_instance} provides a pictorial representation of the worst-case instance.

\begin{figure}[t]
    \centering
    \begin{tabular}{cccccc}
         \multirow{4}{*}{\rotatebox[origin=c]{90}{contexts}} & \multicolumn{1}{c|}{$c_1$} & $P_1 \sim \mathcal D_{K,T/C}$ & 0 & $\cdots$ & 0 \\
         & \multicolumn{1}{c|}{$c_2$} & 0 & $P_2 \sim \mathcal D_{K,T/C}$ & $\cdots$ & 0 \\
         & \multicolumn{1}{c|}{$\vdots$} & $\vdots$ & $\vdots$ & $\ddots$ & $\vdots$ \\ 
         & \multicolumn{1}{c|}{$c_C$} & 0 & 0 & $\cdots$ & $P_C \sim \mathcal D_{K,T/C}$ \\ \cline{3-6}
         & &  $\underbrace{1,\ldots,\tfrac T C,}_{\substack{\text{epoch 1}\\c_t = c_1}}$ 
         &  $\underbrace{\tfrac T C + 1,\ldots,2 \tfrac T C,}_{\substack{\text{epoch 2}\\c_t = c_2}}$
         & $\cdots$
         & $\underbrace{(C-1)\tfrac T C + 1,\ldots,T}_{\substack{\text{epoch }T\\c_t = c_C}}$\\
         && \multicolumn{4}{c}{time}
     \end{tabular}
    \caption{Structure of the worst-case instance for the lower bound for adversarial rewards and adversarial contexts in Theorem~\ref{thm:advlb}. The entries of the table correspond to the distribution of rewards for each context/epoch pair.}
    \label{fig:lb_instance}
\end{figure}

Note that in this construction, cross-learning between contexts offers zero additional information, since all cross-learned rewards will be deterministically $0$ (and thus can be simulated by a learner without access to cross-learning). It suffices to show that any contextual bandits algorithm in the classic setting (i.e., without cross-learning) must incur regret at least $\Omega(\sqrt{CKT})$ on this distribution over instances.

Consider the (optimal) stationary strategy $\pi$ that plays $i_j$ (i.e., the optimal action under instance $P_j$) whenever the context is $c_j$. Fix an arbitrary contextual bandit algorithm $\mathcal{A}$ and consider its performance on the $j$th epoch. We argue that $\pi$ receives at least $\Omega(\sqrt{KT/C})$ more reward in expectation than algorithm $\mathcal{A}$ on this epoch because the length of the epochs is $T/C$. To see this, note that if this were not the case, by examining the restriction of $\mathcal{A}$ to the rounds in this epoch, we can construct an algorithm $\mathcal{A'}$ for the regular multi-armed bandit problem that would violate Lemma \ref{lem:mab-lb}.

It follows that over all $C$ epochs our strategy $\pi$ receives at least $\Omega(C \sqrt{KT/C}) = \Omega(\sqrt{CKT})$ more reward in expectation than algorithm $\mathcal{A}$. Thus, any algorithm $\mathcal{A}$ must have regret at least $\Omega(\sqrt{CKT})$ when compared to the optimal stationary strategy.\Halmos
\end{proof}

\section{Regret Lower Bounds under  Cross-Learning}\label{sec:lb_partial}

In this section, we prove some lower bounds on regret for contextual bandits with  cross-learning that complement the results of the previous two sections. In our lower bounds, we will consider a restricted set of instances where the graph $G_i$ of each arm $i$ is equal to the same graph $G$. We present two lower bounds. The first lower bound, presented in Theorem \ref{thm:partial-stoch-lb},  is for a setting with  stochastic rewards and stochastic contexts. We show that in this setting, any algorithm must incur expected regret of $\Omega(\sqrt{\nu_2(G)KT})$, where $\nu_2(G)$ is defined in Definition \ref{def:nu_2}.  The second lower bound, presented in Theorem \ref{thm:partial-adv-lb}, is for a setting with   stochastic  rewards and  adversarial contexts. 
We show that in this setting,  any algorithm must incur expected regret of $\Omega(\sqrt{\lambda(G)KT})$, where $\lambda(G)$  is the maximum acyclic subgraph number of graph $G$  and is defined in Definition \ref{def:acyclic}. 
Recall that by Lemma \ref{lem:orderlemma}, for all graphs $G$,
$ \nu_2(G) \leq \lambda(G) \leq \kappa(G)$, and by Theorem \ref{thm:ucbcross}, when $G_i =G$, $i\in [K]$, the expected regret of \ucbcross{} under stochastic rewards and stochastic/adversarial contexts is $O(\sqrt{{\kappa}(G)KT\log K})$. 

{\color{black}We note that per   Lemma \ref{lem:orderlemma},  if graph $G$ is the union of $r$ disjoint cliques, then  $\iota(G) = \nu_2(G) = \lambda(G) = \kappa(G)= r$. Note that for complete cross-learning setting  (i.e., fully connected CL graph), $r= 1$ and for no cross-learning setting (i.e., CL graphs with only self loops), $r$ is equal to the number of contexts $C$. This shows that our regret bounds are tight for certain graphs that are the union of $r$ disjoint cliques. In fact, our regret bounds are tight for an even larger family of graphs: all graphs where $\iota(G) = \kappa(G)$. This is true for unions of disjoint cliques, for the line graph in Figure~\ref{fig:partial}, and more generally for any undirected graph (i.e., directed graphs with symmetric edges) which is \textit{perfect}. In graph theory, a perfect (undirected) graph $G$ is a graph where for any subgraph $G'$ of $G$, the size of the largest clique of $G'$ is equal to the chromatic number of $G'$ (the minimum number of colors needed to color the vertices of $G'$ so that no two adjacent vertices share the same color). It follows as a simple corollary of the perfect graph theorem \citep{lovasz1972normal, lovasz1972characterization} that any perfect graph $G$ satisfies $\iota(G) = \kappa(G)$. Perfect graphs include all bipartite graphs, forests, interval graphs, and comparability graphs of posets (see, e.g., \citealt{west1996introduction}).}

\subsection{Regret Lower Bound with Stochastic Rewards and  Contexts}
The following theorem is the main result of this section.

\begin{theorem}[Lower Bound with Stochastic Rewards and  Contexts]\label{thm:partial-stoch-lb}
Any learning algorithm solving the contextual bandits problem with  cross-learning (for a fixed CL graph $G$) and stochastic rewards and stochastic contexts must incur expected regret of $\Omega(\sqrt{\nu_2(G)KT})$, where $\nu_2(G)$ is defined in Definition \ref{def:nu_2}.
\end{theorem}

Note also Theorem \ref{thm:partial-stoch-lb} holds as a lower bound in the adversarial rewards and stochastic contexts setting. Although the regret benchmarks differ slightly between the adversarial rewards and stochastic rewards setting, they differ by at most $\tilde O(\sqrt{T})$, which is subsumed by the  regret bound in Theorem \ref{thm:partial-stoch-lb}. See Appendix \ref{sec:benchmark_app} for more details.

{Roughly, the proof of Theorem \ref{thm:partial-stoch-lb} proceeds as follows. For each context, we choose a ``hard'' distribution of rewards such that any algorithm that has only seen $t$ samples of rewards must incur $\Omega(\sqrt{K}/t)$ expected regret in their next round. Now, if we fix a distribution $f(c)$ over contexts, then after $t$ rounds, we expect to have seen approximately $g(c)t$ total samples of rewards from context $c$, where $g(c) = \sum_{c' \in \cal{I}(c)}f(c')$. Over all $T$ rounds, our regret is therefore at least
$\sum_{c=1}^Cf(c)T\sqrt{\frac{K}{g(c)T}}\,.$ 
Taking the supremum over $f$, we find that the total expected regret is at least $\Omega(\sqrt{\nu_2(G)KT})$; see the definition of $\nu_2(G)$ in Definition \ref{def:nu_2}.}

\subsection{Regret Lower Bound with Stochastic Rewards and  Adversarial  Contexts}
We now present our second lower bound for the setting with stochastic rewards and adversarial contexts. When we allow the contexts to be adversarially chosen, we can improve the lower bound to $\Omega(\sqrt{\lambda(G)KT})$. 

\begin{theorem}[Lower Bound with Stochastic Rewards and  Adversarial  Contexts]\label{thm:partial-adv-lb}
Any learning algorithm solving the contextual bandits problem with  cross-learning (for a fixed CL graph $G$) with stochastic rewards and adversarial contexts must incur regret $\Omega(\sqrt{\lambda(G)KT})$, where $\lambda(G)$  is the maximum acyclic subgraph number of graph $G$  and is defined in Definition \ref{def:acyclic}.
\end{theorem}

Note that when the graphs are undirected, $\lambda(G) = \iota(G)$ (since in that case, the definition of acyclic subgraph and independent set coincide), and therefore $\lambda(G) = \nu_2(G) = \iota(G)$ (by Lemma \ref{lem:orderlemma}). It follows that when all $G_i$ are undirected and equal, the lower bound of Theorem \ref{thm:partial-adv-lb} matches the upper bound of Theorem \ref{thm:exppcross} in the setting where contexts are stochastic. Likewise, as stated earlier, when $G$ is the disjoint union of $r$ cliques, all of our graph invariants coincide, and our lower bounds are tight. In other settings and for other feedback structures an instance-dependent gap between the best upper bound and best lower bound persists; reducing this gap is an interesting open problem. %{\color{red} Please compare this discussion with the earlier discussion on our bounds being tight. let's decide about where to put this discussion.}

\section{Discussion on  Applications of Cross-Learning and Implementation of Proposed Algorithms  }\label{sect:applications}

In this section, we discuss how to apply our results on cross-learning to the problem of how to bid in a first-price auction. We show that our algorithms yield a non-trivial improvement over naively applying $S$-EXP3 or $S$-UCB. We further discuss how to efficiently implement our algorithms when the number of contexts is infinite. Before presenting our results for first-price auctions,  we provide other applications that enjoy cross-learning between contexts and thus fit our framework.

\subsection{Applications with Cross-learning: Beyond Bidding in First-price Auctions }
The following applications can be modeled as contextual bandit problems with cross-learning between contexts. Here, we present an overview of these applications  and we defer the details of how to apply our algorithms to these applications to Appendix \ref{sec:extra_app}.

\begin{enumerate} [leftmargin=*]
\item \textbf{Multi-armed bandits with exogenous costs: } Consider a multi-armed bandit problem where at the beginning of each round $t$, a cost $s_{i,t}$ for playing arm $i$ at this round is publicly announced. That is, choosing arm $i$ this round results in a net reward of $r_{i,t} - s_{i,t}$. This captures settings where, for example, a buyer must choose every round to buy one of $K$ substitutable goods -- they are aware of the price of each good (which might change from round to round) but must learn over time the utility each type of good brings them. 

This is a contextual bandits problem where the context in round $t$ is the $K$ costs $s_{i, t}$ ($i\in [K]$) at this time. Cross-learning between contexts is present in this setting: given the net utility of playing action $i$ with a given up-front cost $s_{i}$, one can infer the net utility of playing $i$ with any other up-front cost $s_{i}'$. For the problem of multi-armed bandits with exogenous costs, standard contextual bandit algorithms get regret $O(T^{(K+1)/(K+2)}K^{1/(K+2)})$. Our algorithms get regret $\tilde{O}(\sqrt{KT})$, which is tight. See Appendix \ref{sec:extra_app} for more details. 

\item
\textbf{Dynamic pricing with variable cost:} Consider a dynamic pricing problem where a firm offers a service (or sells a product) to a stream of customers who arrive sequentially over time. Consumers have private and independent willingness-to-pay and the cost of serving  a customer is exogenously given and customer dependent. After observing the cost, the firm decides on what price to offer to the consumer who decides whether to accept the service at the offered price. The optimal price for each consumer is contingent in the cost; for example, when demand is relatively inelastic consumers that are more costly to serve should be quoted higher prices. This extends dynamic pricing problems to cases where the firm has exogenous costs (see, e.g., \citealt{denBoer2016pricing} for an overview of dynamic pricing problems).

This is a special case of the multi-armed bandits with exogenous costs problem defined earlier, and hence an instance of contextual bandits with cross-learning.

\item
\textbf{Sleeping bandits: } Consider the following variant of ``sleeping bandits,'' where there are $K$ arms and in each round some subset $S_t$ of these arms are awake. The learner can play any arm and observe its reward, but only receives this reward if they play an awake arm. This problem was originally proposed in \cite{kleinberg2010regret}, where one of the motivating applications is ecommerce settings where not all sellers or items (and hence ``arms'') might be available every round.

This is a contextual bandits problem where the context is the set $S_t$ of awake arms. Again, cross-learning between contexts is present in this setting: given the observation of the reward of arm $i$, one can infer the received reward for any context $S_t'$ by just checking whether $i \in S_t'$. For our variant of sleeping bandits, standard contextual bandit algorithms get regret $\tilde{O}(\sqrt{2^K KT})$. Our algorithms get regret $\tilde{O}(\sqrt{KT})$, which is tight. By applying our algorithms, we can achieve regret $\tilde{O}(\sqrt{KT})$ in the original sleeping bandits setting studied in \cite{kleinberg2010regret}, which recovers their results and is similarly tight. 

\item
\textbf{Repeated Bayesian games with private types: } Consider a player participating in a repeated Bayesian game. Each round the player learns their (private and independent) type for the current game, performs some action, and receives some utility (which depends on their type, their action, and the other players' actions). Again, this can be viewed as a contextual bandit problem where a player's types are contexts, actions are actions, and utilities are rewards, and once again this problem allows for cross-learning between contexts (as long as the player can compute their utility based on their type and all players' actions). 
\end{enumerate}

\subsection{Bidding in first-price auctions}

In the problem of learning to bid in a first-price auction, every round $t$ (for a total of $T$ rounds) an item is put up for auction. This item has value $v_t \in [0, 1]$ to our bidder. Based on $v_t$, our bidder submits a bid $b_t \in [0, 1]$. Simultaneously, other bidders submit bids for this item; we let $h_t$ be the highest bid of the other bidders in the auction. If $b_t \geq h_t$, the buyer receives the item and pays $b_t$, obtaining an utility of $v_t - b_t$; otherwise, the buyer does not receive the item and pays nothing, obtaining a utility of zero. More formally, the net utility is $r_{b,t}(v) = (v-b) \cdot \mathbbm{I}(b \geq h_t)$. The buyer only learns whether or not they receive the item and how much they pay -- notably, they do not learn $h_t$ (i.e., this is a non-transparent first price auction). The bidder's goal is to maximize their total utility (total value of items received minus total payment) over the course of $T$ rounds.

{As stated in the introduction, this problem can be seen as a contextual bandits problem for the bidder where the context $c$ is the bidder's value for the item, the action is their bid, and their reward is their net utility from the auction: 0 if they do not win, and their value for the item minus their payment $p$ if they do win. This problem, indeed, 
 allows for cross-learning between contexts: The net utility $r_{i,t}(c')$ that would have been received if they had value $c'$ instead of value $c$ is just $(c'-p)\cdot \mathbbm{I}(\mbox{win item})$, which can be  computed from the outcome of the auction assuming the value and highest competing bid (that influences $\mathbbm{I}(\mbox{win item})$) are independent of each other.}

Here, we assume that the value $v_t$ and the highest competing bid $h_t$ are independently drawn each round from distributions $\D_v$ and $\D_h$, respectively, where both distributions are unknown to the bidder. The independence assumption is motivated by the fact that, in online advertising markets, most advertisers base bids on cookies, which are bits of information stored on users' browsers. Because cookies are private, cookie-based bids are typically weakly correlated. {\color{black}In Section~\ref{sect:experiments}, we conduct some experiments using data from a major advertising platform and observe that our cross-learning algorithms that assume complete cross-learning between values perform well even  when values are not perfectly independent of the other bids. Nonetheless, conservative learners might be still concerned about such correlation. Under such correlation, with the same action/bid, the chance of winning (i.e., $\mathbbm{I}(\mbox{win item})$) under two values that are far from each other may not be the same. For instance, when there is a positive correlation between values and highest competing bids, as the value increases, the highest competing bid may increase as well, reducing the chance of winning. 
To handle this, such learners can only allow for cross-learning between close values. We remark that, from the theoretical perspective, when the correlation between values and bids is arbitrary, cross-learning is impossible and the decision maker cannot do better than running a different learning algorithm for each context. A promising research direction is to incorporate correlation by introducing a statistical or behavioral model to capture the dependency between bids and values.}

Naively applying $S$-UCB to our problem by discretizing the value space and bid space into $C$ and $K$ pieces respectively results in a regret bound of $\tilde{O}(\sqrt{CKT} + T/C + T/K)$ (here the last two terms come from discretization error). Optimizing $C$ and $K$, we find that when $C = K = T^{1/4}$, we can achieve $\tilde{O}(T^{3/4})$ regret in this way. 
On the other hand, by taking advantage of (complete) cross-learning between contexts and  applying \ucbcross{}, after discretizing the bid space into $K$ pieces, results in a regret bound of $\tilde{O}(\sqrt{KT} + T/K)$. By optimizing this,  we get an algorithm which achieves $\tilde{O}(T^{2/3})$ regret. It follows from a reduction to known results about dynamic pricing that any algorithm must incur $\Omega(T^{2/3})$ regret when learning to bid (even when the value $v$ is fixed) -- see Appendix \ref{sect:fpa_lb} for details.\footnote{The regret lower bound of $\Omega(T^{2/3})$ is driven by our binary feedback structure. Under the binary feedback structure, which is commonly assumed in the literature, a bidder can only learn whether they win or lose in an auction. The follow-up paper \cite{han2020optimal} shows that a regret bound of $\tilde O(\sqrt T)$ is attainable when the highest bid is revealed at the end of each auction to all the bidders who lost the auction.}

Interestingly, in the case of bidding in first-price auctions, the decision maker can also potentially cross-learn across different actions/bids. For example, if the decision maker wins when submitting a bid $b_t$, then they simultaneously learn that any higher bid $b'$ would also win the auction. Conversely, if the decision maker does not win, then they learn that lower bids also would necessarily lose in the auction. While our algorithm does not explicitly take into account cross-learning across actions, the previous lower bound shows that, in the worst case, cross-learning across actions does not lead to any additional benefit (in terms of lower regret) if we are already cross-learning across contexts. We emphasize that our algorithms apply when the auctioneer runs other non-truthful auctions.

\subsection{Implementation of Proposed Algorithms}

We conclude with a brief note on implementation efficiency of our algorithms. Even though, under complete cross-learning, the regret bounds we prove in Section \ref{sect:crosslearning} do not scale with $C$, note that the computational complexity of all three of our algorithms from Section \ref{sect:crosslearning} (\ucbcross{}, \expcrossk{}, and \expcrossu{}) scales with the number of contexts $C$: both algorithms have time complexity $O(C + K)$ per round and space complexity $O(CK)$.  In many of the above settings, the number of contexts can be very large. For example, when the space of contexts is the interval $[0, 1]$, the number of contexts is infinite. However, these settings often also have additional structure which let us run these same algorithms with improved complexity. 

Most generally, for all the settings we consider, the observed reward is always affine with respect to a function $\rho(c)$ mapping a context into $\R^{d}$ for some small dimension $d$. The function $\rho$ is computable by the learner. That is, for each $i$ and $t$, it is possible to write $r_{i,t}(c) = a_{i,t}^{\top}\rho(c) + b_{i,t}$, where $a_{i,t} \in \R^d$ and $b_{i,t} \in \R$; moreover, the coefficients $a_{I_t, t}$ and $b_{I_t, t}$ are directly revealed to the learner each round. It in turn follows that the averages $\overline{r}_{i, t}(c)$ stored by \ucbcross{} are simply linear functions of $\rho(c)$. Since there is one such function for each arm $i$, this requires a total of $O(Kd)$ space (i.e., we simply store the running averages $\overline a_{i,t}$ and $\overline b_{i,t}$ and then determine the average reward using the formula $\overline{b}_{i, t} = \overline{a}_{i, t}^{\top}\rho(c) + \overline{b}_{i, t}$). Similarly, the coefficients can be updated each round in $O(d)$ time simply by updating the average for $I_t$. For example, for $b_{i,t}$ the update is given by
$\overline{b}_{I_t, t} = \frac{\tau_{I_t, t-1}\overline{b}_{I_t, t-1}+ b_{I_t, t}}{\tau_{I_t, t-1} + 1}.$

Likewise, the weights $w_{i,t}(c)$ stored by \expcrossk{}, for example, are always of the form $\exp(x_{i,t}\rho(c) + y_{i,t})$, and again it suffices to just maintain a linear function of $\rho(c)$ (with the caveat that to compute the estimators, we must be able to efficiently take expectations over our known distribution on contexts).

\section{Empirical Evaluation}\label{sect:experiments}

In this section, we empirically evaluate the performance of our contextual bandit algorithms on the problem of learning how to bid in a first-price auction. 

Recall that our cross-learning algorithms rely on cross-learning between contexts being possible: if the outcome of the auction remains the same, the bidder can compute their net utility they would receive given any value they could have for the item. This is true if the bidder's value for the item is independent of the other bidders' values for the item. Of course, this assumption (while common in much research in auction theory) does not necessarily hold in practice. We can nonetheless run our contextual bandit algorithms (with complete cross-learning) as if this were the case, and compare them to existing contextual bandit algorithms which do not make this assumption.

Our basic experimental setup is as follows. We take existing first-price auction data from a large ad exchange that runs first-price auctions on a significant fraction of traffic, remove one participant (whose true values we have access to), substitute in one of our bandit algorithms for this participant, and replay the auction. This experiment answers the question ``how well would this (now removed) participant do if they instead ran this bandit algorithm?'' 

We collected anonymized data from 10 million consecutive auctions from this ad exchange, which were then divided into 100 groups of $10^5$ auctions. To remove outliers, bids and values above the 90\% quantile were removed, and remaining bids/values were normalized to fit in the $[0, 1]$ interval.\footnote{Our numerical results in Appendix \ref{sec:numercis} show that the performance of our algorithms is robust to outliers and, thus, not sensitive to how outliers are handled.} We then replayed each group of $10^5$ auctions, comparing the performance of our  algorithms with cross-learning (\ucbcross{}, \expcrossk{}, and \expcrossu{}) and the performance of classic contextual bandits algorithms that take no advantage of cross-learning ($S$-EXP3, and $S$-UCB1). {\color{black} To run \expcrossk{}, we do not assume that we know the distribution over contexts. Instead, we replace the estimator in  \expcrossk{} with its empirical version, presented  at the end of Section \ref{sec:unknown_dist}. More specifically, we consider the following estimator
$
\tilde{r}_{i, t}(c) = (r_{i,t}(c)/ \tilde{D}_{i, t})\mathbbm{I}(I_t = i),
$
where $\tilde{D}_{i, t} = \sum_{c'=1}^C \widehat{\Pr}[c']\cdot p_{i, t}(c')
$  does not require knowledge of the distribution over contexts.  Here,  $\widehat{\Pr}[c']$ is the empirical estimate (sample mean) of  the true probability $\Pr[c']$ of context $c'$. }

All algorithms considered here  require a discretized set of actions. Thus,  allowable bids are discretized to multiples of $0.01$. Parameters for each of these algorithms (including level of discretization of contexts for $S$-EXP3 and $S$-UCB1) were optimized via cross-validation on a separate data set of $10^5$ auctions from the same ad exchange.

\begin{figure}[hbt]
\begin{center}
\includegraphics[scale=0.2]{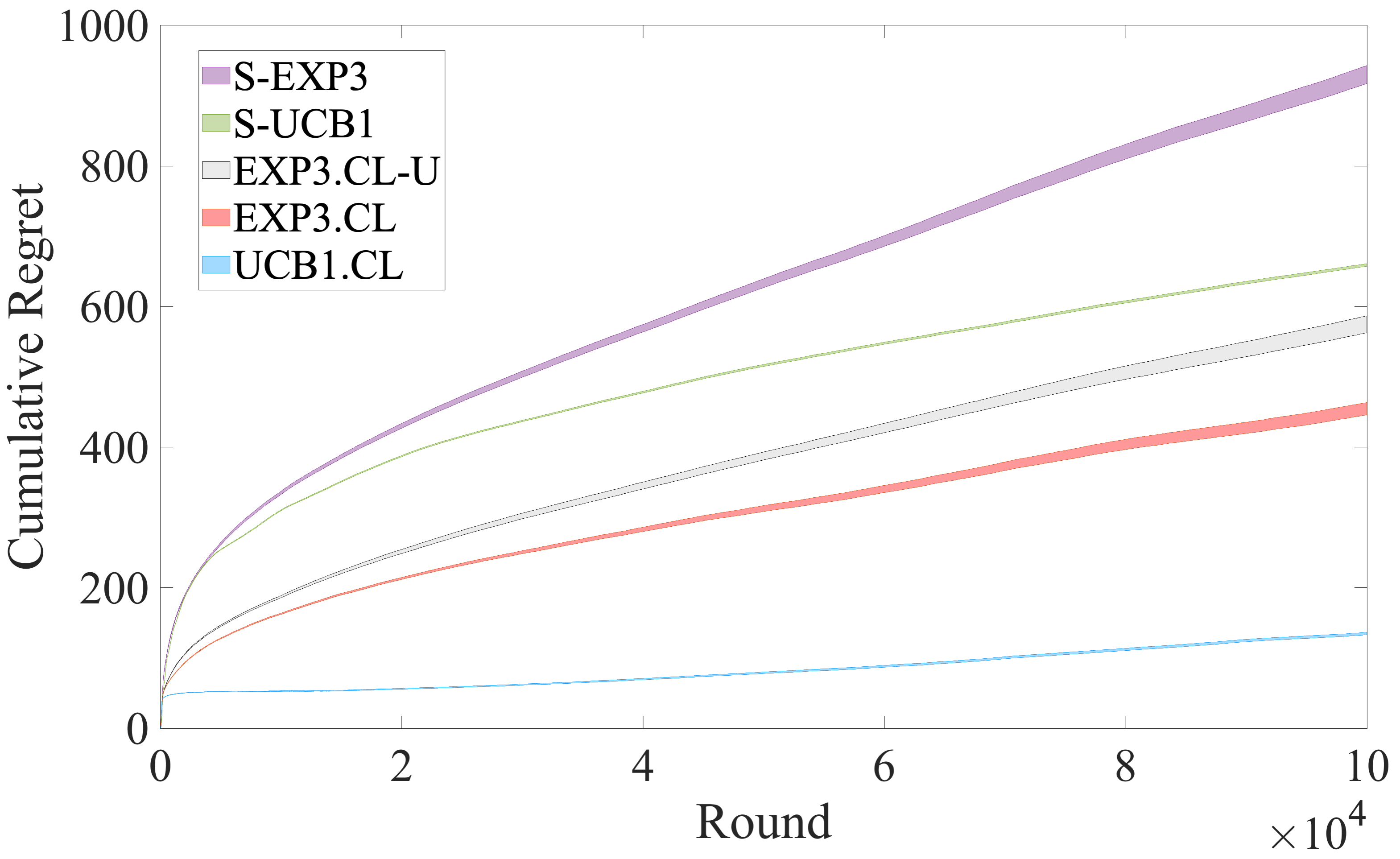}
\end{center}
\caption{Graph of average cumulative regrets of various learning algorithms (y-axis) versus time (x-axis). Shaded regions indicate 95\% confidence intervals. }
\label{fig:regrets}
\end{figure}

The results of this evaluation are summarized in Figure \ref{fig:regrets}, which plots the average cumulative regret of these algorithms over the $10^5$ rounds. To compute the regret of   all algorithms, we consider the best in hindsight benchmark (i.e., for each context/value, we choose the bid that performs the best for that value in hindsight after observing the bidder's values and highest competing bids in   all $T$ rounds). This regret benchmark, which does not require any assumption on the stochasticity of values or the highest competing bids, matches our stationary policy  under adversarial rewards and  adversarial contexts.   
The three algorithms that take advantage of cross-learning (\ucbcross{}, \expcrossk{}, and \expcrossu{}) significantly outperform the two algorithms that do not ($S$-EXP3 and $S$-UCB1). Note that  \ucbcross{} outperforms   \expcrossk{} and \expcrossu{}. {\color{black}Furthermore, it is worth highlighting that  \expcrossk{} performs better than \expcrossu{}. Recall that in our empirical studies, both \expcrossk{} and  \expcrossu{} do not know context/value distribution.}    

What is surprising about these results is that cross-learning works at all, let alone gives an advantage, given that the basic assumption necessary for cross-learning -- that values are independent from other players' bids, so that the learner can predict what would have happened if the value was different -- does not hold. Indeed, for this data, the Pearson correlation coefficient between the values $v$ and the maximum bids $r$ of the other bidders is approximately $0.4$. This suggests that these algorithms are somewhat robust to errors in the cross-learning hypothesis. It is an interesting open question to understand this phenomenon theoretically.

\section{Conclusion}

In this paper, we studied the contextual multi-armed bandit problem with cross-learning between contexts. Our model can be applied to various applications including bidding in non-truthful auctions,  dynamic pricing with variable  costs, and sleeping bandits.  We show that in all of these applications cross-learning between contexts can lead to a non-trivial improvement over the existing algorithms that do not consider cross-learning. 

We designed different learning algorithms that are tailored to distinct environments. These environments vary in several fronts: (i) how rewards and contexts are generated (stochastically versus adversarially), and (ii) to what extend cross-learning is possible (complete versus partial cross-learning). For all of these environments, we proposed effective learning algorithms that  harness the benefit of cross-learning.

An interesting future direction is to design algorithms that take advantage of cross learning between both actions and contexts. Learning algorithms might exploit the particular structure of the problem at hand to obtain regret bounds with better dependence on the number of actions. For example, in some settings the reward function could be Lipschitz continuous or concave in the actions. When rewards are adversarial and contexts are stochastic, we designed an algorithm that under complete cross-learning obtains regret $\tilde{O}(\sqrt{KT})$ when the distribution of contexts is known. It is an interesting research direction to obtain similar regret bounds when the distribution of contexts is unknown to the learner. Finally, an application of cross-learning that has received a considerable amount of attention lately is the problem of bidding in non-truthful auctions. Our algorithms for this problem assume that values are independent of the highest competing bid. Although these algorithms perform well in practice, a promising future direction is to study whether more value can be captured by algorithms that explicitly account for correlation between values and bids.

\bibliographystyle{ACM-Reference-Format}
\bibliography{ref}

%\ECSwitch

%\ECDisclaimer
%%%%%%%%%%%%%%%%%%%%%%%%%%%%%%%%%%%%%%%%%%%%%%%%%%%%%%%%%%

%%% Main head for the e-companion
\newpage
\begin{APPENDICES}

%\DoubleSpacedXI
%\SingleSpacedXI
\OneAndAHalfSpacedXI

\section{Regret in contextual bandits}\label{sect:regret}

{In this section, we elaborate on our discussion of regret in Section \ref{sect:contextual_prelims}, contrasting the difference between the different notions of regret for different settings. }

Recall that we define the regret of an algorithm $\cal A$ in the contextual setting as the difference between the performance of our algorithm and the performance of the best stationary strategy $\pi$. In other words,
$$\Reg(\cal A) = \sum_{t=1}^{T}r_{\pi(c_t),t}(c_t) - \sum_{t=1}^{T}r_{I_{t}, t}(c_t).$$
However, when rewards are adversarial and  contexts are stochastic, there are two different natural ways to define ``the best stationary policy'' $\pi$. The first maximizes the reward of this strategy for the specific contexts $c_t$ we observed in our run of algorithm $\cal A$:
$$\pi'(c) = \arg\max_{i\in [K]} \sum_{t=1}^{T}r_{i,t}(c)\mathbbm{I}({c_{t} = c})$$
The second way simply maximizes the reward of this strategy in expectation over all time 
$$\pi(c) = \arg\max_{i\in [K]} \sum_{t=1}^{T}r_{i,t}(c).$$ 
Note that the policy $\pi$  maximizes the expectation of performance over contexts $\sum_{t=1}^T \E_{c_t \sim \D}[ r_{\pi(c_t), t}(c_t)] = \sum_{c \in [C]} \Pr[c] \sum_{t=1}^T r_{\pi(c), t}(c)$.

These two stationary strategies give rise to two different definitions of regret. We call the regret against strategy $\pi'$ the \textit{ex post regret} $\Reg_{\post}(\cal A)$ (and denote the associated strategy by $\pi_{\post}$), and we call the regret against strategy $\pi$ the \textit{ex ante regret}, $\Reg_{\ante}(\cal A)$ (and denote the associated strategy by $\pi_{\ante}$). This captures the idea that to the adversary at the beginning of the game (who knows all the rewards, but not when each context will occur), the best stationary strategy in expectation is $\pi_{\ante}$. On the other hand, after the game has finished, the best stationary strategy in hindsight is $\pi_{\post}$. 

In this paper, our bounds for the adversarial rewards and stochastic contexts are for {ex ante regret}.  
One reason for this is that, while it is possible to eliminate the dependence on $C$ in the ex ante regret, it is impossible to do so for the ex post regret. In particular, for a large enough number of different contexts $C$, it is impossible to get ex post regret that is sublinear in $T$.

{
\begin{theorem}\label{theorem:expost-lower}
Under adversarial rewards and stochastic contexts,  for any algorithm $\cal A$, there is an instance of the contextual bandits problem with cross-learning where $\liminf_{C\rightarrow \infty} \E[\Reg_{\post}(\cal A)] \geq T/2$. 
\end{theorem}
\begin{proof}{Proof of Theorem \ref{theorem:expost-lower}}
We will consider an instance of the problem where there are $K=2$ actions and $C$ contexts, where the distribution $\D$ is uniform over all $C$ contexts. We will choose $C$ to be large enough so that with high probability all the observed contexts $c_t$ are distinct. The adversary will assign rewards as follows. For each round $t$ and context $c$, with probability $1/2$, the adversary sets $r_{1,t}(c) = 1$ and $r_{2,t}(c) = 0$, and with probability $1/2$ sets $r_{1,t}(c) = 0$ and $r_{2,t}(c) = 1$. 

Now consider the best strategy $\pi_{\post}$ in hindsight. Since each context only appears once with high probability, and since there is always an arm with reward $1$, for any context and any time, $\pi_{\post}$ will receive total reward $T$ as $C$ goes to $\infty$. To see this, let $B$ the event that all contexts are distinct, i.e., $\left|\{t \in [T] : c_t = c\} \right | \le 1 $ for all $c \in C$. The expected performance of best stationary policy $\pi_{\post}$ can be lower bounded as follows
\begin{align*}
    \mathbb E\left[ \sum_{t=1}^T r_{\pi_{\post}(c_t),t}(c_t) \right]
    \ge \mathbb E\left[ \sum_{t=1}^T r_{\pi_{\post}(c_t),t}(c_t) \mathbbm{I}(B)\right]
    = T \mathbb{P} (B)\,,
\end{align*}
where the first inequality follows because rewards are non-negative and the equality holds because the best strategy with the benefit of hindsight is to pick the arm with reward one because all contexts are distinct. Additionally, the probability that the contexts are distinct can be lower bounded by
\begin{align*}
    \mathbb{P} (B) = \frac{C!}{(C-T)! C^T} \ge \left( \frac C{C-T}\right)^{C-T} \exp(-T) \ge \exp(-T^2/C) \ge 1 - \frac{T^2}{C}\,,
\end{align*}
where the first equation follows from a combinatorial argument (out of all $C^T$ possible sample paths of contexts we have $C!/(C-T)!$ distinct ones). The first inequality follows from writing $C!/(C-T)! = \exp\left( \sum_{x=C-T+1}^C \log(x) \right)$ and then using the integral bound for summations. The second inequality holds because $\left( C/(C-T)\right)^{C-T} = \exp((C-T)\log(C/(C-T))$ and using that $\log(x) \ge 1 - 1/x$, and the last because $\exp(x) \ge 1 + x$. Combining both inequalities, we obtain that
\begin{align}\label{eq:worst-case-lb}
    \mathbb E\left[ \sum_{t=1}^T r_{\pi_{\post}(c_t),t}(c_t) \right]\ge T \left (1 - \frac{T^2}{C}\right)\,.
\end{align}

On the other hand, since each $r_{i,t}$ is completely independent of the rewards from previous rounds, the maximum expected reward any learning algorithm can guarantee is at most $T/2$ as $C$ goes to infinity. To see this, denote by $N_c = \left|\{t \in [T] : c_t = c\} \right |$ the number of times that context $c \in [C]$ appears and by $t_j(c)$ the time period corresponding to the $j$th occurrence of context $c$ for $j \in [N_c]$. The expected reward of a learning algorithm $\cal A$ can be written as follows:
\begin{align*}
    \mathbb E\left[ \sum_{t=1}^T r_{I_t,t}(c_t) \right]
    = \mathbb E\left[ \sum_{c=1}^C \sum_{t=1}^T r_{I_t,t}(c) \mathbbm{I}(c_t=c) \right]
    = \sum_{c=1}^C \mathbb E\left[ \sum_{j=1}^{N_c} r_{I_{t_j(c)},t_j(c)}(c) \right]\,.
\end{align*}
For every context $c \in [C]$, the first time the context is seen the best possible expected reward is $1/2$, which is obtained by playing any arm at random, because rewards are independent and, thus, the two arms are ex-ante identical to the learner. Therefore,
\[
 \mathbb E\left[ \sum_{j=1}^{N_c} r_{I_{t_j(c)},t_j(c)}(c) \right] \le \mathbb E\left[ \frac 1 2 \mathbbm{I}(N_c \ge 1) + (N_c - 1) \mathbbm{I}(N_c \ge 2) \right] = \mathbb E\left[ N_c \right] - \frac 1 2 \mathbb P(N_c \ge 1)\,,
\]
where the first inequality follows because the highest possible reward is one and the last equality because $(N_c - 1) \mathbbm{I}(N_c \ge 2) = (N_c - 1) \mathbbm{I}(N_c \ge 1) = N_c - \mathbbm{I}(N_c \ge 1)$. Now, using that $N_c$ is binomially distributed with $T$ trials as success probability $1/C$, we conclude that
\begin{align}\label{eq:worst-case-ub}
    \mathbb E\left[ \sum_{t=1}^T r_{I_t,t}(c_t) \right]
    &\le C \left( \mathbb E\left[ N_c \right] - \frac 1 2 \mathbb P(N_c \ge 1)\right)
    = T - \frac C 2 \left(1 - \left(1 - \frac 1 C\right)^T\right) \nonumber\\
    &= \frac T 2 + \frac 1 2 \left(T+ C \left(\left(1 - \frac 1 C\right)^T-1\right)\right) \nonumber\\
    &\le \frac T 2 + \frac 1 4 \frac {T^2}{C}\,,
\end{align}
where we used that $\mathbb E\left[ N_c \right] = T/C$ and $\mathbb P(N_c \ge 1) = 1-\mathbb P(N_c = 0) = 1 - (1-1/C)^T$ for the first equality, and the last inequality follows because $\left(1 - 1/C\right)^T = \exp(T \log(1-1/C)) \le \exp(-T/C) \le 1 - T/C + 1/2 (T/C)^2$ because $\log(1 - x)\le -x$ and $\exp(-x) \le 1- x + x^2/2$ for $x>0$. 

Putting everything together, from combining \eqref{eq:worst-case-lb} and \eqref{eq:worst-case-ub} it follows that any algorithm $\cal A$ must have 
\[
\liminf_{C\rightarrow \infty} \Reg_{\post}(\cal A) =
    \liminf_{C\rightarrow \infty} \mathbb E\left[ \sum_{t=1}^T r_{\pi_{\post}(c_t),t}(c_t) - \sum_{t=1}^T r_{I_t,t}(c_t) \right] 
    \ge \frac T 2\,. \Halmos
\]
\end{proof}
}

\subsection{Translating regret lower bounds from stochastic rewards to adversarial rewards}\label{sec:benchmark_app}

The previous section focuses on the stochasticity of the contexts. In this section, we focus on the stochasticity of the \textit{rewards} and how it affects our regret bounds. Specifically, we aim to show that any regret lower bound we prove in a model with stochastic rewards and contexts (see Theorem \ref{thm:partial-adv-lb}) continues to hold in a model with adversarial rewards and stochastic contexts (as long as this bound is at least $\tilde{\Omega}(\sqrt{T})$).

The reason why this is not entirely obvious is that the regret benchmarks we give for stochastic rewards and adversarial rewards are slightly different. When the rewards are stochastic, we compete against the ``pseudo-regret'' benchmark that plays the arm with the highest mean for each context, i.e.,
$$\pi_{\stoch}(c) = \arg\max_{i\in [K]} \mu_{i}(c).$$

When the rewards are adversarial, there are no reward distributions so the above strategy is undefined; instead, we play the best action in hindsight:
$$\pi_{\adv}(c) = \arg\max_{i \in [K]} \sum_{t=1}^{T} r_{i, t}(c).$$

In some sense, $\pi_{\stoch}$ and $\pi_{\adv}$ can be thought of as analogues to $\pi_{\post}$ and $\pi_{\ante}$ in thee previous section. 

\begin{theorem}
Assume that, for any algorithm $\mathcal{A}$, there is an instance of multi-armed bandits with learning with stochastic rewards and contexts where
$$\E\left[\sum_{t=1}^{T}(r_{\pi_{\stoch}(c_t), t}(c_t) - r_{I_t, t}(c_t))\right] \geq R\,,$$
where the expectation is with respect to rewards and contexts.  
Then, for any algorithm $\mathcal{A}$, there exists an instance of multi-armed bandits with learning with \textit{adversarial} rewards and stochastic contexts where
$$\E\left[\sum_{t=1}^{T}(r_{\pi_{\adv}(c_t), t}(c_t) - r_{I_t, t}(c_t))\right] \geq R - \sqrt{T\log(KT)}\,,$$
where the expectation is with respect to contexts and  the potential randomness in the adversarially  chosen rewards.   
\end{theorem}

\begin{proof}{Proof}

Fix an algorithm $\mathcal{A}$. Consider a stochastic bandits instance which achieves regret $R$. This bandits instance is parameterized by a distribution $\mathcal{D}$ over contexts and a distribution $\mathcal{F}_{i}(c)$ of rewards for each arm $i\in [K]$ and context $c\in [C]$.

We will now randomly sample a bandits instance with adversarial rewards and stochastic contexts as follows. The distribution over contexts will remain as $\mathcal{D}$. For each time step $t$, arm $i$, and context $c$, sample the reward $r_{i, t}(c)$ independently from $\mathcal{F}_i(c)$. We will prove that, in expectation, algorithm $\mathcal{A}$ achieves regret at least $R - \sqrt{T\log(KT)}$ on such a sampled instance. In particular, this implies that there exists a specific instance with adversarial rewards and stochastic contexts  where $\mathcal{A}$ incurs at least $R - \sqrt{T\log(KT)}$ regret, as desired.

Note that in expectation, $\E[\sum_{t=1}^T r_{I_t, t}(c_t)]$ is the same for both the stochastic rewards instance and the distribution over adversarial rewards instances (since $\mathcal{A}$ sees the same distribution over instances in both cases). It suffices to show that the two benchmarks perform similarly. 

First, note that both 

$$\E\left[\sum_{t=1}^T r_{\pi_{\stoch}(c_t), t}(c_t)\right] = \sum_{c=1}^C \Pr_{\mathcal{D}}[c] \E\left[\sum_{t=1}^T r_{\pi_{\stoch}(c), t}(c)\right],$$

\noindent
and

$$\E\left[\sum_{t=1}^T r_{\pi_{\adv}(c_t), t}(c_t)\right] = \sum_{c=1}^C \Pr_{\mathcal{D}}[c] \E\left[\sum_{t=1}^T r_{\pi_{\adv}(c), t}(c)\right],$$

\noindent
where in both cases we have applied the law of total expectation to the randomness in the contexts. 

Now, fix a context $c$. Let

$$R_{\adv}(c) = \E\left[\sum_{t=1}^T r_{\pi_{\adv}(c), t}(c)\right] = \E_{r_{i, t}(c) \sim \mathcal{F}_{i}(c)}\left[\max_{i \in [K]} \sum_{t=1}^{T} r_{i, t}(c)\right],$$

\noindent
and 

$$R_{\stoch}(c) = \E\left[\sum_{t=1}^T r_{\pi_{\stoch}(c), t}(c)\right] = T\cdot \max_{i\in [K]} \mu_{i}(c).$$

We will show that $R_{\adv}(c) - R_{\stoch}(c) \leq \sqrt{T\log(KT)}$, and this completes the proof.  Let $\Delta = \sqrt{T\log(KT)}/2$. Note that 
\begin{eqnarray*}
R_{\adv}(c) &=& \E\left[\max_{i \in [K]} \sum_{t=1}^{T} r_{i, t}(c)\right] \\
&= & \E\left[\max_{i \in [K]} \sum_{t=1}^{T} r_{i, t}(c) \mathbbm{I} \left( \max_{i \in [K]} \sum_{t=1}^{T} r_{i, t}(c) \le  R_{\stoch}(c) + \Delta \right) \right]
\\
&+& \E\left[\max_{i \in [K]} \sum_{t=1}^{T} r_{i, t}(c)  \mathbbm{I} \left( \max_{i \in [K]} \sum_{t=1}^{T} r_{i, t}(c) >  R_{\stoch}(c)  +\Delta \right) \right]\\
&\leq & R_{\stoch}(c) + \Delta + T \cdot \Pr\left[\max_{i \in [K]} \sum_{t=1}^{T} r_{i, t}(c) > R_{\stoch}(c) + \Delta\right]\\
&\leq & R_{\stoch}(c) + \Delta + T \sum_{i \in [K]} \Pr\left[\sum_{t=1}^{T} r_{i, t}(c) > T \mu_i(c) + \Delta\right] \\
&\leq & R_{\stoch}(c) + \Delta + T \sum_{i \in [K]} \exp(-2T(\Delta/T)^2) \\
&\leq & R_{\stoch}(c) + \Delta + T \sum_{i \in [K]} 1/(TK) \\
&\leq & R_{\stoch}(c) + 2\Delta\,,
\end{eqnarray*}
where the second equation follows from conditioning, the second inequality follows from an union bound together with $R_{\stoch}(c) \ge T \mu_i(c)$, the third by Hoeffding's inequality, and the fourth inequality from our definition of $\Delta$. It thus follows that $\E_{c}[R_{\adv}(c)] - \E_{c}[R_{\stoch}(c)] \leq \sqrt{T\log(KT)}$, as desired.
\end{proof}

\section{Lower bound for learning to bid}\label{sect:fpa_lb}

In this section, will show that any algorithm for learning to bid in a first-price auction must incur at least $\Omega(T^{2/3})$ regret even if there is only one value (so no potential for cross-learning between contexts). To show this, we will use a reduction to the problem of dynamic pricing. 

The problem of dynamic pricing is as follows. The learner must repeatedly (for $T$ rounds) sell an item to a buyer with value $x_t$ drawn i.i.d. from some unknown distribution $\D$. At each point in time, a price $p_t$ is proposed. If $x_t \geq p_t$, the buyer purchases the item and the learner receives payment $p_t$ (alternatively, regret $(x_t - p_t)$); otherwise if $x_t < p_t$ the buyer does not purchase the item and regret is $x_t$. The goal of this game is to maximize total revenue, or equivalently, minimize the total regret (with respect to the optimal fixed price $p^{\star}$).

\cite{kleinberg2003value} prove the following bounds on this problem.

\begin{theorem}[Theorem 4.3 in \cite{kleinberg2003value}]\label{thm:lowerbound23}
For any $T$, there exists a family of distributions $\mathcal{P} = \{\D_i\}$ on $[0,1]$ such that if $\D$ is sampled uniformly from $\mathcal{P}$ and the buyer's valuations are sampled iid according to $\D$, any pricing strategy must incur expected regret $\Omega(T^{2/3})$.
\end{theorem}

This lower bound can be matched (up to log factors) by discretizing (to $K = O(T^{1/3})$ intervals) and running EXP3. 

We now show this lower bound immediately implies a lower bound on the learning to bid problem, even when there is only one context/value.

\begin{theorem}[{Lower Bound for Learning to Bid}]\label{thm:reduction}
Any algorithm must incur $\Omega(T^{2/3})$ regret for the learning to bid in first price auctions problem, even if the value of the bidder is fixed (i.e., there is only one context).
\end{theorem}
\begin{proof}{Proof of Theorem \ref{thm:reduction}}
We will show how to use a learning algorithm for the learning to bid problem to solve the dynamic pricing problem. 

Consider an instance of the learning to bid problem where $v_t = 1$ always (i.e., $\D_v$ is the singleton distribution supported on $1$). If the bidder bids $b_t$ in this auction, then with probability $\Pr_{h \sim \D_h}[b_t \geq h]$,  the bidder wins the auction and receives reward $(1-b_t)$, and with probability $1 - \Pr_{h \sim \D_h}[b_t \geq h]$ the bidder loses the auction and receives reward $0$. Here, $\D_h$ is the distribution of the highest competing bid in the auction. 

Now consider pricing when the value of the buyer is drawn from $\D = 1 - \D_h$ (that is, one can sample from $\D$ by sampling $x$ from $\D_h$ and returning $1-x$). If we set a price $p_t$ in this auction, then with probability $\Pr_{x \sim \D}[x \geq p_t]$, the item is sold and the seller receives reward $p_t$, and with probability $1 - \Pr_{x \sim \D}[x \geq p_t]$, the item is not sold and the seller receives reward $0$.

But note that $\Pr_{x \sim \D}[x \geq p_t] = \Pr_{h \sim \D_h}[1-h \geq p_t] = \Pr_{h \sim \D_h}[1-p_t \geq h]$. In particular, setting a price of $p_t$ in the pricing problem with distribution $1 - \D_h$ results in the exact same feedback and rewards as bidding $1-p_t$ in the learning to bid problem with distribution $\D_h$. One can therefore use any algorithm for the learning to bid problem to solve the dynamic pricing problem with the same regret guarantee; since Theorem \ref{thm:lowerbound23} implies any learning algorithm must incur $\Omega(T^{2/3})$ regret on the dynamic pricing problem, it follows that any learning algorithm must incur $\Omega(T^{2/3})$ regret for the learning to bid problem.
\Halmos\end{proof}

\section{Omitted proofs}

\subsection{Proof of Lemma \ref{lem:ucbpart2}} 
Essentially, we must show that after observing arm $i$ $m_{i}(c)$ times, we no longer lose substantial regret from that arm in context $c$. 
Begin by noting that
\begin{eqnarray*}
\sum_{i=1}^{K}\sum_{c=1}^{C}\sum_{t=1}^{T}\Delta_{i}(c)\mathbbm{I}(I_{t}=i, c_{t}=c, \tau_{i,t}(c) > m_{i}(c)) & \leq & \sum_{i=1}^{K}\sum_{c=1}^{C}\sum_{t=1}^{T}\mathbbm{I}(I_{t}=i, c_{t}=c, \tau_{i,t}(c) > m_{i}(c)) \\
&=& \sum_{i=1}^{K}\sum_{t=1}^{T}\mathbbm{I}(I_{t}=i, \tau_{i,t}(c_t) > m_{i}(c_t))\,,
\end{eqnarray*}
where the inequality holds since the reward of each arm $i$ (and consequently the gap $\Delta_{i}(c)$) is bounded in $[0,1]$.  %{\color{red} $\tau_{i,t}(c)$ or $\tau_{i,t}(c_t)$?}

In expectation, this is equal to
$$\sum_{i=1}^{K}\sum_{t=1}^{T}\Pr[I_{t}=i, \tau_{i,t}(c_t) > m_{i}(c_t)].$$

Now, define $U_{i,t}(c) = \overline{r}_{i,t}(c) + \omega(\tau_{i, t}(c))$ to be the upper confidence bound for arm $i$ under context $c$ in round $t$. Note that if $I_{t}=i$, then $U_{i,t-1}(c_t) \geq U_{j,t-1}(c_t)$ for any other arm $j$. This holds because the algorithm chooses the arm with the highest upper confidence bound. It follows that (fixing $i$ and $t$) 
$$\Pr[I_{t}=i, \tau_{i,t}(c_t) > m_{i}(c_t)] \leq \Pr\left[U_{i,t-1}(c_t)\geq U_{i^{\star}(c_t), t-1}(c_t), \tau_{i,t}(c_t) > m_i(c_t)\right].$$
(Here $i^{\star}(c_t)$ is the optimal arm under context $c_t$).

Define $t_{i}(n,c)$ to be the minimum round $t$ such that $\tau_{i,t}(c) = n$, and define $\overline{x}_{i,n}(c) = \overline{r}_{i,t_{i}(n, c)}(c)$ (in other words, $\overline{x}_{i,n}(c)$ is the average value of the first $n$ rewards from arm $i$, in context $c$). Note that if $\tau_{i,t}(c) \geq m_i(c)$, then $U_{i, t-1}(c) \geq U_{i^{\star}(c), t-1}(c)$ implies that %{\color{red} should not we have $n< t$ on both sides the following inequality?}
$$\max_{m_{i}(c_t) \leq n < t} \overline{x}_{i,n}(c) + \omega(n) \geq \min_{0<n'<t} \overline{x}_{i^{\star}(c), n'}(c) + \omega(n').$$

We can therefore write 
\begin{eqnarray*}
&\;& \Pr\left[U_{i,t-1}(c_t)\geq U_{i^{\star}(c_t), t-1}(c_t), \tau_{i,t}(c_t) > m_i(c_t)\right] \\
&\leq & \Pr\left[\max_{m_{i}(c_t) \leq n < t} \overline{x}_{i,n}(c_t) + \omega(n) \geq \min_{0<n'<t} \overline{x}_{i^{\star}(c_t), n'}(c_t) + \omega(n')\right] \\
&\leq & \sum_{n=m_{i}(c_t)}^{t}\sum_{n'=1}^{t}\Pr\left[\overline{x}_{i,n}(c_t) + \omega(n) \geq  \overline{x}_{i^{\star}(c_t), n'}(c_t) + \omega(n')\right].
\end{eqnarray*}
Here the last inequality follows from applying the union bound over all choices of $n$ and $n'$. 

Finally, observe that if $\overline{x}_{i,n}(c_t) + \omega(n) \geq \overline{x}_{i^{\star}(c_t), n'}(c_t) + \omega(n')$, then one of the following events must occur:

\begin{enumerate}
\item $\overline{x}_{i^{\star}(c_t),  n'}(c_t) \leq \mu^{\star}(c_t) - \omega(n')$.
\item $\overline{x}_{i, n}(c_t) \geq \mu_i(c_t) + \omega(n)$.
\item $\mu^{\star}(c_t) < \mu_{i}(c_t) + 2\omega(n)$.
\end{enumerate}

Now, recall that $m_{i}(c) = \frac{8\log T}{\Delta_{i}(c)^2}$. Note that since $n \geq m_i(c_t)$, we have that $\omega(n) \leq \omega(m_{i}(c_t)) \leq \Delta_{i}(c_t)/2$, so $\mu_{i}(c_t) + 2\omega(n) \leq \mu_{i}(c_t) + \Delta_{i}(c_t) \leq \mu^{\star}(c_t)$, and therefore the third event can never occur. 

The first two events both occur with probability at most $t^{-4}$ by Hoeffding's inequality. For example, for the first event, Hoeffding's inequality implies that
$$\Pr\left[\overline{x}_{i^{\star}(c_t),  n'}(c_t) - \mu^{\star}(c_t) \leq - \omega(n') \right] \leq \exp(-2n'\omega(n')^2) = \exp(-4\log T) \leq t^{-4}.$$

It is similarly true that the probability of the second event is at most $t^{-4}$. We thus have that 
\begin{eqnarray*}
\Pr[I_{t}=i, \tau_{i,t}(c_t) > m_{i}(c_t)] &\leq & \sum_{n=m_{i}(c_t)}^{t}\sum_{n'=1}^{t}\Pr\left[\overline{x}_{i,n}(c_t) + \omega(n) \geq  \overline{x}_{i^{\star}(c_t), n'}(c_t) + \omega(n')\right] \\
&\leq & \sum_{n=m_{i}(c_t)}^{t}\sum_{n'=1}^{t} 2t^{-4} \leq  2t^{-2}.
\end{eqnarray*}
Further summing this over all $i \in [K]$ and $t \in [T]$, we have that 
$$\sum_{i=1}^{K}\sum_{t=1}^{T}\Pr[I_{t}=i, \tau_{i,t}(c_t) > m_{i}(c_t)] \leq \frac{K\pi^2}{3}.$$

\subsection{Proof of Theorem \ref{thm:exppcross}}
We proceed similarly to the analysis of EXP3. Begin
by defining the estimator
$$\hat{r}_{i,t}(c) = \frac{r_{i,t}(c)}{\sum_{c' \in \cal I_i(c)} \Pr[c'] \cdot p_{i, t}(c')}\mathbbm{I}(I_t = i, c_t \in \cal I_i(c)).$$
Note that
$$\Pr[I_{t} = i, c_t \in \cal I_i(c)] = \sum_{c' \in \cal I_i(c)}\Pr[c']\cdot p_{i,t}(c'),$$
\noindent
so taking expectations over the algorithm's choice of $I_{t}$, we have that 
$$\E[\hat{r}_{i,t}(c)] = r_{i,t}(c),$$ 
\noindent
and 
$$\E[\hat{r}_{i,t}(c)^2] = \frac{r_{i,t}(c)^2}{\sum_{c' \in \cal I_i(c)}\Pr[c']\cdot p_{i,t}(c')}.$$
Now, let $W_{t}(c) = \sum_{i=1}^{K}w_{i,t}(c)$. Note that
\begin{eqnarray*}
\frac{W_{t+1}(c)}{W_{t}(c)} &=& \sum_{i=1}^{K} \frac{w_{i,t}(c)}{W_{t}(c)} \cdot e^{\beta \hat{r}_{i,t}(c)} \\
&=& \sum_{i=1}^{K} \frac{p_{i,t}(c) - \alpha}{1-K\alpha} e^{\beta \hat{r}_{i,t}(c)} \\
&\leq &  \frac{1}{1-K\alpha}\sum_{i=1}^{K} (p_{i,t}(c)-\alpha)\left(1 + \beta\hat{r}_{i,t}(c) + (e-2)\beta^2\hat{r}_{i,t}(c)^2\right) \\
&\leq & 1 + \frac{\beta}{1-K\alpha}\sum_{i=1}^{K}p_{i,t}(c)\hat{r}_{i,t}(c) + \frac{(e-2)\beta^2}{1-K\alpha}\sum_{i=1}^{K}p_{i,t}(c)\hat{r}_{i,t}(c)^2\,,
\end{eqnarray*}
where the first equation holds because for any $c\in [C]$, $w_{i, t+1} (c)= w_{i,t}(c) \cdot e^{\beta \hat{r}_{i,t}(c)}$, and the second equation follows because $p_{i,t}(c) = (1-K\alpha)\frac{w_{i,t}(c)}{W_{t}(c)} + \alpha$.
In the first inequality,  we have used the fact that $\beta\hat{r}_{i,t}(c) \leq \beta r_{i,t}(c)/\alpha \leq 1$ (since $\beta/\alpha \leq 1$ for any choice of $T$ and $K$), that $e^{x} \leq 1 + x + (e-2)x^2$ for $x \in [0, 1]$, and that all rewards $r_{i,t}(c)$ are bounded in $[0,1]$. Now, using the fact that $\log (1+x) \leq x$, we have that:
$$\log\left( \frac{W_{t+1}(c)}{W_{t}(c)}\right) \leq \frac{\beta}{1-K\alpha}\sum_{i=1}^{K}p_{i,t}(c)\hat{r}_{i,t}(c) + \frac{(e-2)\beta^2}{1-K\alpha}\sum_{i=1}^{K}p_{i,t}(c)\hat{r}_{i,t}(c)^2\,, $$

\noindent
and therefore (summing over all $t$)
\begin{equation}\label{eq:W_T}
\log \left( \frac{W_{T}(c)}{W_{0}(c)}\right) \leq \frac{\beta}{1-K\alpha}\sum_{t=1}^{T}\sum_{i=1}^{K}p_{i,t}(c)\hat{r}_{i,t}(c) + \frac{(e-2)\beta^2}{1-K\alpha}\sum_{t=1}^{T}\sum_{i=1}^{K}p_{i,t}(c)\hat{r}_{i,t}(c)^2 .
\end{equation}

Recall that we compute regret against the optimal stationary policy $\pi(c) = \arg\max_{i}\sum_{t=1}^{T}r_{i,t}(c)$. Then,
\begin{eqnarray}
\log \left(  \frac{W_{T}(c)}{W_{0}(c)}\right) & \geq & \log \frac{w_{\pi(c), T}(c)}{K} \nonumber \\
&=& \beta \sum_{t=1}^{T} \hat{r}_{\pi(c),t}(c) - \log K\,,  \label{eq:W_T_pi}
\end{eqnarray}
where the first inequality holds because (i) $w_{i,0} (c) =1$ for any $i\in [K]$ and as a result,  $W_{0}(c) = K$, and (ii) $W_{T}(c) =  \sum_{i=1}^{K}w_{i,T}(c) \ge w_{\pi(c), T}(c)$. From (\ref{eq:W_T}) and (\ref{eq:W_T_pi}), we get
\begin{equation}\label{eq:4}
\frac{\beta}{1-K\alpha}\sum_{t=1}^{T}\sum_{i=1}^{K}p_{i,t}(c)\hat{r}_{i,t}(c) + \frac{(e-2)\beta^2}{1-K\alpha}\sum_{t=1}^{T}\sum_{i=1}^{K}p_{i,t}(c)\hat{r}_{i,t}(c)^2 \geq \beta\sum_{t=1}^{T} \hat{r}_{\pi(c),t}(c) - \log K.
\end{equation}
Simplifying (\ref{eq:4}) (multiplying through by $(1-K\alpha)/\beta$ and applying the fact that $r_{i,t}(c)$ is bounded), this becomes\footnote{Note that for $T \geq K\log K$, $\alpha \leq 1/K$, so $1-K\alpha$ is always positive.}
\begin{equation}\label{eq:5}
\sum_{t=1}^{T} \hat{r}_{\pi(c),t}(c) - \sum_{t=1}^{T}\sum_{i=1}^{K}p_{i,t}(c)\hat{r}_{i,t}(c) \leq \frac{\log K}{\beta} + (e-2)\beta\sum_{t=1}^{T}\sum_{i=1}^{K}p_{i,t}(c)\hat{r}_{i,t}(c)^2 + KT\alpha .
\end{equation}

We now take expectations (with respect to all randomness, both of the algorithm and of the contexts) of both sides of (\ref{eq:5}). 
\begin{align} \nonumber
&\sum_{t=1}^{T} r_{\pi(c),t}(c) - \sum_{t=1}^{T}\sum_{i=1}^{K}\E[p_{i,t}(c)]r_{i,t}(c) \\
&\leq \frac{\log K}{\beta} + (e-2)\beta\sum_{t=1}^{T}\sum_{i=1}^{K}\E\left[\frac{p_{i,t}(c)}{\sum_{c'\in \cal I_i(c)}\Pr[c']\cdot p_{i,t}(c')}\right]r_{i,t}(c)^2 + KT\alpha\,. \label{eq:6c}
\end{align}
Note that the expected regret $\E[\Reg({\cal A})]$ of our algorithm is equal to
\begin{eqnarray*}
\E[\Reg({\cal A})] &=&  \E\left[\sum_{t=1}^{T} r_{\pi(c_t), t}(c_t) - \sum_{t=1}^{T} r_{I_t(c_t), t}(c_t)\right] \\
&=& \sum_{t=1}^{T} \E\left[r_{\pi(c_t), t}(c_t) - r_{I_t(c_t), t}(c_t)\right] \\
&=& \sum_{t=1}^{T} \sum_{c=1}^{C}\Pr[c]\E\left[r_{\pi(c), t}(c) - r_{I_t(c), t}(c)\right] \\
&=& \sum_{t=1}^{T} \sum_{c=1}^{C}\Pr[c]\left(r_{\pi(c), t}(c) - \E\left[r_{I_t(c), t}(c)\right]\right)\,.
\end{eqnarray*}
Since arm $I_t$ is drawn from distribution $p_t(c)$, we have 
\begin{eqnarray*}
\E[\Reg({\cal A})] &=& \sum_{t=1}^{T} \sum_{c=1}^{C}\Pr[c]\left(r_{\pi(c), t}(c) - \sum_{i=1}^{K}\E[p_{i,t}(c)]r_{i,t}(c)\right) \\
&=& \sum_{c=1}^C \Pr[c]\left(\sum_{t=1}^{T} r_{\pi(c),t}(c) - \sum_{t=1}^{T}\sum_{i=1}^{K}\E[p_{i,t}(c)]r_{i,t}(c)\right)\,.
\end{eqnarray*}
From Equation  (\ref{eq:6c}), we get that
\begin{eqnarray*}
\E[\Reg({\cal A})] &\leq & \sum_{c=1}^{C} \Pr[c]\left(\frac{\log K}{\beta} + (e-2)\beta\sum_{t=1}^{T}\sum_{i=1}^{K}\E\left[\frac{p_{i,t}(c)}{\sum_{c' \in \cal I_i(c)}\Pr[c']\cdot p_{i,t}(c')}\right]r_{i,t}(c)^2 + KT\alpha\right) \\
&=& \frac{\log K}{\beta} + (e-2)\beta \sum_{t=1}^{T}\sum_{i=1}^{K}\sum_{c=1}^{C}\Pr[c]\cdot \E\left[\frac{p_{i,t}(c)}{\sum_{c' \in \cal I_i(c)}\Pr[c']\cdot p_{i,t}(c')}\right] r_{i,t}(c)^2 + KT\alpha  \\
&\leq& \frac{\log K}{\beta} + (e-2)\beta \sum_{t=1}^{T}\sum_{i=1}^{K}\E\left[\sum_{c=1}^{C}\frac{\Pr[c]p_{i,t}(c)}{\sum_{c' \in \cal I_i(c)}\Pr[c']\cdot p_{i,t}(c')}\right] + KT\alpha  \\
&\leq& \frac{\log K}{\beta} + (e-2)\beta \sum_{t=1}^{T}\sum_{i=1}^{K}\nu(G_i) + KT\alpha\\
&= & \frac{\log K}{\beta} + (e-2)\beta \overline{\nu}KT + KT\alpha \\
&=& O(\sqrt{\overline{\nu}KT\log K}).
\end{eqnarray*}
where $\bar \nu = \frac{1}{K} \sum_{i =1} \nu(G_i)$ and { $\nu(G_i)=\sup_{\substack{f : [C] \rightarrow \R^{+} \\ \sum_{c=1}^C f(c) = 1}} \sum_{c=1}^{C} \frac{f(c)}{\sum_{c' \in \cal I_i(c)} f(c')}$.  
The last inequality follows because  
\[
\sum_{c=1}^{C}\frac{\Pr[c]p_{i,t}(c)}{\sum_{c' \in \cal I_i(c)}\Pr[c']\cdot p_{i,t}(c')} 
\leq 
\sup_{\substack{f : [C] \rightarrow \R^{+} \\ \sum_{c=1}^C f(c) = 1}} \sum_{c=1}^{C} \frac{f(c)}{\sum_{c' \in \cal I_i(c)} f(c')} = \nu(G_i) = \lambda(G_i)\,.
\]
Here, the inequality holds by setting $f(c) = \Pr[c]\cdot p_{i,t}(c)$ and renormalizing such that $\sum_{c=1}^C f(c) = 1$}, and the last equality follows from  Lemma~\ref{lem:mas_equiv}.

\subsection{Proof of Lemma \ref{lem:mas_equiv}}
{\color{black}Let
\[\nu(G) =\sup_{\substack{f : [C] \rightarrow \R^{+} \\ \sum_{c=1}^C f(c) = 1}} \sum_{c=1}^{C} \frac{f(c)}{\sum_{c' \in \cal I(c)} f(c')}\,,\]
where $\cal I(c)$ is the set of in-neighbors of node/context $c$ in graph $G$.}
We begin by showing that $\nu(G) \geq \lambda(G)$.

Let $(v_1, v_2, \dots, v_{\lambda(G)})$ be an acyclic subgraph of $G$ of maximum size. Fix a large $M > 1$, and consider the following function $f:V \rightarrow \mathbbm{R^{+}}$: $f(v) = M^{i}$ if $v = v_i$, $i\in [\lambda(G)]$,  and $f(v) = 1$ otherwise. We claim that as $M \rightarrow \infty$, the quantity
\begin{equation}
\sum_{v\in V} \frac{f(v)}{\sum_{w \in \cal I(v)} f(w)}
\end{equation}

\noindent
approaches a value larger than {or equal to} $\lambda(G)$. To do this, we will simply show that for each $v_i$ in our acyclic subgraph, the quantity
$$\frac{f(v_i)}{\sum_{w \in \cal I(v_i)} f(w)}$$

\noindent
approaches a value larger than {or equal to} $1$. {Since there are at least $\lambda(G)$ such terms, and since all terms are always nonnegative, this implies the desired result.} 

To see this, note that by the definition of an acyclic subgraph, for all $j > i$, there is no edge $v_j \rightarrow v_i$. Therefore, for every $w \in \cal I(v_i)$ (with the exception of $v_i$  itself), $f(w) \leq M^{i-1}$ because every $w$ in $\cal I(v_i)$ is of the form $v_j$ for some $j < i$, and therefore $\sum_{w \in \cal I(v_i)} f(w) \leq |V|M^{i-1} + M^{i}$. It follows that
$$\frac{f(v_i)}{\sum_{w \in \cal I(v_i)} f(w)} \geq \frac{M^{i}}{|V|M^{i-1} + M^{i}}.$$

The right hand side of this expression converges to 1 as $M$ approaches infinity. 

The proof that $\nu(G) \leq \lambda(G)$ follows from Lemma 10 in \cite{alon2017nonstochastic}. 

\subsection{Proof of Lemma \ref{lem:orderlemma}}

We prove the inequalities in order.

\paragraph{\underline{$\iota(G) \leq \nu_2(G)$:}}

Let $S$ be an independent set in $G$ of size $\iota(G)$. Define the distribution $f$ via $f(v) = \frac{1 - \eps}{\iota(G)}$ (for some small $\eps$) for $v \in S$ and $f(v) = \frac{\eps}{|V| - \iota(G)}$ for $v \not\in S$. As $\eps \rightarrow 0$, we have that

$$\sum_{v \in V} \frac{f(v)}{\sqrt{\sum_{w \in \cal I(v)} f(w)}} \longrightarrow \sum_{v \in S} \frac{1/\iota(G)}{\sqrt{1/\iota(G)}} = \sqrt{\iota(G)}.$$

It follows that

$$\nu_2(G) = \sup_{\substack{f : V \rightarrow \R^{+} \\ \sum_{v\in V} f(v) = 1}}\left(\sum_{v \in V} \frac{f(v)}{\sqrt{\sum_{w \in \cal I(v)} f(w)}}\right)^2 \geq \iota(G).$$

\paragraph{\underline{$\nu_2(G) \leq \lambda(G)$:}}

By {Jensen's inequality}, for any distribution $f$ over $V$, we have that 
$$\left(\sum_{v \in V} \frac{f(v)}{\sqrt{\sum_{w \in \cal I(v)} f(w)}}\right)^2 \leq \sum_{v \in V} \frac{f(v)}{\sum_{w \in \cal I(v)} f(w)}.$$
To see why note that by Jensen's inequality, for any concave function $\phi$ and positive weights $a(v)$ and any numbers $x(v)$, we have $\phi(\frac{\sum_v a(v) x(v)}{\sum_v a(v)})\ge \frac{\sum_v a(v) \phi(x(v))}{\sum_v a(v)}$. Here, $\phi(y) = \sqrt{y}$, weights are $f(v)$'s and $x(v)$'s are $\frac{1}{\sum_{w\in \cal I(v)} f(w)}$. 
Taking suprema of both sides, it follows that
$$\nu_2(G) = \sup_{\substack{f : V \rightarrow \R^{+} \\ \sum_{v\in V} f(v) = 1}}\left(\sum_{v \in V} \frac{f(v)}{\sqrt{\sum_{w \in \cal I(v)} f(w)}}\right)^2 \leq \sup_{\substack{f : V \rightarrow \R^{+} \\ \sum_{v\in V} f(v) = 1}}\sum_{v \in V} \frac{f(v)}{\sum_{w \in \cal I(v)} f(w)} = \lambda(G),$$

where the last equality follows from Lemma \ref{lem:mas_equiv}.

\paragraph{\underline{$\lambda(G) \leq \kappa(G)$:}}

Let $(S_1, S_2, \dots, S_{\kappa(G)})$ be a minimum size clique covering of $G$. Note that no two elements $v, v'$ in the same $S_i$ can belong to the same acyclic subgraph (since by the definition of a clique, there exist edges $v \rightarrow v'$ and $v' \rightarrow v$). It follows that the size of the largest acyclic subgraph is at most $\kappa(G)$, and thus $\lambda(G) \leq \kappa(G)$.

\paragraph{Unions of cliques}

We now show that when $G$ is a disjoint union of $r$ cliques, $\iota(G) = \nu_2(G) = \lambda(G) = \kappa(G) = r$. To do so it suffices (from the above inequalities) to show that $\iota(G) = r$ and $\kappa(G) = r$. The independence number $\iota(G) = r$ since choosing one element from each clique creates an independent set, and any set of $r+1$ or more vertices must have two vertices from the same clique. The clique covering number $\kappa(G) = r$ since we can cover the graph with the $r$ given cliques, and any covering with fewer than $r$ sets must combine elements in disjoint cliques (thus violating the fact that each set is a clique).

\subsection{Proof of Theorem \ref{thm:exp3cross}}

As in our analysis of  \expcrossk{}, we define the estimator 
$$\hat{r}_{i, t}(c) = \frac{r_{i,t}(c)}{p_{i, t}(c')}\mathbbm{I}(I_t = i, c \in \mathcal{O}(c')).$$

Note that $\hat{r}_{i, t}(c)$ is not an unbiased estimator of $r_{i, t}(c)$. Indeed, we have that: 
$$\E[\hat{r}_{i, t}(c)] = \sum_{c' \in \mathcal{I}(c)} \Pr[c_t] p_{i,t}(c') \cdot \frac{r_{i, t}(c)}{p_{i, t}(c')} = \Big(\sum_{c' \in \mathcal{I}(c)} \Pr[c']\Big) r_{i, t}(c).$$
However, note that we can write $\E[\hat{r}_{i, t}(c)]$ in the form $f(c)r_{i, t}(c)$, where $f(c)$ is a function which only depends on a context (and in this case is given by $f(c) = \sum_{c' \in \mathcal{I}(c)} \Pr[c']$). It turns out this property is enough to adapt the previous analysis of Theorem \ref{thm:exppcross}.

Indeed, proceeding in the same way as the proof of Theorem \ref{thm:exppcross}, we arrive at the inequality
\begin{equation}\label{eq:5c_rep}
\sum_{t=1}^{T} \hat{r}_{\pi(c),t}(c) - \sum_{t=1}^{T}\sum_{i=1}^{K}p_{i,t}(c)\hat{r}_{i,t}(c) \leq \frac{\log K}{\beta} + (e-2)\beta\sum_{t=1}^{T}\sum_{i=1}^{K}p_{i,t}(c)\hat{r}_{i,t}(c)^2 + KT\alpha.
\end{equation}

In addition, note that it is still true that
$$\E[\hat{r}_{i,t}(c)^2] \leq \frac{r_{i, t}(c)^2}{\alpha}.$$

Taking expectations of \eqref{eq:5c_rep}, we therefore have that:
\begin{eqnarray*}
f(c) \left(\sum_{t=1}^{T} r_{\pi(c),t}(c) - \sum_{t=1}^{T}\sum_{i=1}^{K}\E[p_{i,t}(c)]r_{i,t}(c)\right) &\leq & \frac{\log K}{\beta} + (e-2)\beta\sum_{t=1}^{T}\sum_{i=1}^{K}\frac{\E[p_{i,t}(c)]}{\alpha}r_{i,t}(c)^2 + KT\alpha \\
&\leq & \frac{\log K}{\beta} + (e-2)KT\frac{\beta}{\alpha} + KT\alpha \\
&=& O(K^{1/3}T^{2/3}(\log K)^{1/3}).
\end{eqnarray*}

Now, multiply both sides of this inequality by $\Pr[c]/f(c)$, and sum over all $c$. On the left hand side, we have
$$\sum_{c=1}^{C}\sum_{t=1}^{T} \Pr[c](r_{\pi(c),t}(c) - \E[p_{i,t}(c)]r_{i,t}(c)) = \E[\Reg(\mathcal{A})]\,.$$

\noindent
(by the same logic as in  Theorem \ref{thm:exppcross}). On the other hand, the right hand side is now
$$\left(\sum_{c=1}^{C} \frac{\Pr[c]}{\sum_{c' \in \mathcal{I}(c)}\Pr[c']}\right) \cdot O(K^{1/3}T^{2/3}(\log K)^{1/3}) = O(\lambda K^{1/3}T^{2/3}(\log K)^{1/3}),$$
where this inequality follows due to Lemma \ref{lem:mas_equiv}. It follows that the expected regret is at most $O(\lambda K^{1/3}T^{2/3}(\log K)^{1/3})$.

\subsection{Proof of Theorem \ref{thm:partial-stoch-lb}
}

To prove this, we will need a slightly stronger variant of Lemma \ref{lem:mab-lb}. Recall that Lemma \ref{lem:mab-lb} states that in the non-contextual multi-armed bandit setting, any algorithm incurs an expected cumulative regret of $\Omega(\sqrt{KT})$. The following lemma states that in the same setting, any algorithm incurs an expected regret of at least $\Omega(\sqrt{K/T})$ per round. 

\begin{lemma}\label{lem:mab-lb-strong}
There exists a distribution over instances of the multi-armed bandit problem (with $K$ arms and $T$ rounds) where for any round $t \in [T]$, any algorithm must incur an expected regret of at least $\Omega(\sqrt{K/T})$ in round $t$.
\end{lemma}
\begin{proof}{Proof of Lemma \ref{lem:mab-lb-strong}}
See Section \ref{sec:proof:lem:mab-lb-strong}.
\Halmos\end{proof}

Now, let $f: [C] \rightarrow \R^{+}$ be any distribution on contexts (i.e., $\sum_{c=1}^C f(c) = 1$). Define $g(c) = \sum_{c'\in \cal I(c)}f(c')$. Consider the following distribution over instances of the contextual bandits problem with partial cross-learning:

\begin{itemize}
    \item Every round, the context $c_t$ is drawn independently from the distribution $f$.
    \item The distribution of rewards for a context $c$ is drawn from the distribution over hard instances in Lemma \ref{lem:mab-lb-strong} for a multi-armed bandit problem with $K$ arms and $g(c)T/2$ rounds. 
\end{itemize}

Note that in the second point, the distribution over rewards changes per context depending on $g(c)$. Intuitively, this is because we expect to observe (through learning) the performance of some action in context $c$ in approximately $g(c)T$ rounds.

For each context $c$ and round $t$, let $\tau_c(t) = \sum_{s=1}^{t}\mathbbm{I}(c_s \in \cal I(c))$ be the number of rounds up to round $t$ where we observe the performance of some action in context $c$. Let $T_c$ be the total number of rounds $t$ where $c_t = c$ and $\tau_c(t) \leq g(c)T$. We claim that any algorithm  must incur regret at least 
\begin{equation}\label{eq:context_reg_lb}
\Omega\left(\E[T_c]\sqrt{\frac{K}{g(c)T}}\right)
\end{equation}
\noindent
from the rounds where $c_t = c$. To see this, let $\{t_1, t_2, \dots, t_{\min(\tau_c(T), g(c)T)}\}$ be the set of (the first $g(c)T$) rounds where {$c_t \in \cal{I}(c)$}   and let $S(c) = \{j | c_{t_j} = c\}$ be the subset of indices where $c_{t_j}$ equals $c$. We claim that, conditioned on $S(c)$, any algorithm  must incur expected regret at least
$$\Omega\left(|S(c)|\sqrt{\frac{K}{g(c)T}}\right)$$
\noindent
from the rounds $t_j$ for $j \in S(c)$. If not, this means that there is one $j \in S(c)$ where the expected regret from this round is $o(\sqrt{K/(g(c)T)})$; but this would violate Lemma \ref{lem:mab-lb-strong} (in particular, this gives a regular multi-armed bandits algorithm which incurs expected regret $o(\sqrt{K/(g(c)T)})$ in round $j$). Since $|S_c| = T_c$, taking expectations over $T_c$, equation (\ref{eq:context_reg_lb}) follows.

Now, we claim that $\E[T_c] = \Omega(f(c)T)$. This follows since 
\begin{eqnarray*}
\E[T_c] &=& \sum_{j=1}^{g(c)T} \Pr[c_{t_j} = c]\cdot \Pr[\tau_c(T) \geq j] \\
&=& \frac{f(c)}{g(c)}\sum_{j=1}^{g(c)T} \Pr[\tau_c(T) \geq j] \\
&\geq & \frac{f(c)}{g(c)}(g(c)T/2)\cdot \Pr[\tau_c(T) \geq g(c)T/2] \\
&\geq & \frac{f(c)T}{2}\cdot (1 - \exp(-g(c)^2T/2)) \\
&\geq & \Omega(f(c)T)\,,
\end{eqnarray*}
\noindent
where the second inequality holds because conditioned on cross-learning  the reward under context $c$, the probability that context is $c$ is $f(c)/g(c)$. Recall that we learn  the reward under context $c$ when the realized  context  belongs to set $ \cal I(c)$ and by our definition, $g(c) =\sum_{c'\in \cal I(c)} f(c')$. 
Furthermore, in the last step, we use that $(1 - \exp(-g(c)^2T/2)) \geq \Omega(1)$ for sufficiently large $T$.

This implies that the expected regret from rounds where $c_t = c$ is at least $\Omega(f(c)\sqrt{KT/g(c)})$. Summing over all contexts $c$, the total expected regret is at least

$$\Omega\left(\left(\sum_{c=1}^{C}\frac{f(c)}{\sqrt{g(c)}}\right)\sqrt{KT}\right).$$

Since $\nu_2(G) = \sup_{f}\left(\sum_{c=1}^{C}\frac{f(c)}{\sqrt{g(c)}}\right)^2$, taking the supremum over $f$ we find that any algorithm must incur expected regret at least $\Omega(\sqrt{\nu_2(G)KT})$, as desired.

\subsection{Proof of Lemma \ref{lem:mab-lb-strong}}\label{sec:proof:lem:mab-lb-strong}

Consider the following distribution over instances of the multi-armed bandit problem. Let $\eps = \Theta(\sqrt{K/T})$ (the precise value to be chosen later). An arm $i$ is drawn uniformly at random from $[K]$. The rewards from arm $i$ are distributed according to $B((1+\eps)/2)$, and the arms for all $j \neq i$ are distributed according to $B((1-\eps)/2)$ (where here $B(p)$ is the Bernoulli distribution with probability $p$). 

We wish to claim that at any round $t \leq T$, the probability any learner plays the optimal arm $i$ is less than $1/2$, and therefore the learner must incur $\Omega(\eps) = \Omega(\sqrt{K/T})$ regret this round. This is therefore a best-arm identification problem. Theorem 4 in \cite{audibert2010best} implies there exists some $\eps = \Theta(\sqrt{K/T})$ such that this result holds for our distribution of instances.
%\"}

\subsection{Proof of Theorem \ref{thm:partial-adv-lb}}
Let $\{v_1, v_2, \dots, v_{\lambda(G)}\}$ be a maximum acyclic subset of $G$ (with the property that if $i < j$, there is no edge $v_i \rightarrow v_j$). We now proceed as in the proof of Theorem \ref{thm:advlb}. Divide the $T$ rounds into $\lambda(G)$ epochs of $T/\lambda(G)$ rounds each. The adversary must decide both the contexts every round, and the reward distributions for each context. The adversary will do so as follows:

\begin{itemize}
    \item For each round $t$ in epoch $i$, the adversary will set the context $c_t = v_i$.
    \item For each context $c$, the adversary will set the reward distribution equal to a hard instance for the multi-armed bandit problem sampled from the distribution from Lemma \ref{lem:mab-lb}.
\end{itemize}

Note that since the contexts $v_i$ belong to an acyclic subset of $G$, any information cross-learned in epoch $i$ will reveal nothing about the reward distribution for any context $v_j$ with $j > i$ (and hence nothing about the reward distribution in any epoch $j > i$). Since the hard instances are all independent of each other, any algorithm for the contextual bandits problem with partial cross-learning which achieves $o(\sqrt{\lambda(G)KT})$ expected regret on this instance must achieve $o(\sqrt{KT/\lambda(G)})$ expected regret on one of the individual instances, which contradicts Lemma \ref{lem:mab-lb}. %{\color{red} todo: update this after the proof of Theorem 4 is updated.}

\section{Settings with Adversarial  Rewards and Stochastic Contexts, and Unknown Context Distribution}
\label{sec:discuss}
One of the biggest open questions in this work is whether there exists an algorithm in the adversarial rewards setting that can achieve regret that scales with $O(\sqrt{T})$ (as opposed to $O(T^{2/3})$) in the setting where the context distribution $\D$ is \textit{unknown}. Since there is an $\tilde{O}(\sqrt{KT})$ regret in the case where the context distribution is known (\expcrossk{}), it is natural to ask whether there is some way to generalize this algorithm to the setting where the context distribution is unknown (perhaps by using an empirical estimate of the context distribution in place of the known distribution). 

We conjecture that doing this should indeed work and result in a $\tilde{O}(\sqrt{KT})$ regret -- one piece of empirical evidence in favor of this is that the variant of \expcrossk{} tested in our empirical simulations in Section \ref{sect:experiments} does exactly this (uses the empirical distribution of contexts observed thus far to compute the value estimator in \expcrossk{}). Nonetheless, showing that this modified algorithm achieves regret that scales with $O(\sqrt{T})$ remains stubbornly out of reach of our current analytical tools. In this appendix, we provide a short discussion of some of the difficulties that come with rigorously analyzing this modified algorithm.

Before we proceed, we quickly present a brief reminder of the main properties of \expcrossk{}. For this discussion, we will work entirely in the complete cross-learning setting, and suppress dependence on $K$ (which we can take to be a constant) and logarithmic terms in $T$. The main feature of \expcrossk{} is that it updates the weights by multiplying them by the unbiased estimator
\begin{equation}\label{app:eq1}
\hat{r}_{i, t}(c) = \frac{r_{i,t}(c)}{\sum_{c'=1}^C \Pr[c'] \cdot p_{i, t}(c')}\mathbbm{I}(I_t = i).
\end{equation}
This unbiased estimator has low variance, but computing the denominator requires knowledge of the distribution over contexts. We would like to replace it with a similar estimator

\begin{equation}\label{app:eq2}
\tilde{r}_{i, t}(c) = \frac{r_{i,t}(c)}{\tilde{D}_{i, t}}\mathbbm{I}(I_t = i),
\end{equation}

\noindent
where $\tilde{D}_{i, t}$ is some sufficiently close approximation to the quantity $\sum_{c'=1}^C \Pr[c'] \cdot p_{i, t}(c')$ (which we will refer to from now on as $D_{i, t}$) that does not require full knowledge about $\D$ to compute. One natural choice for $\tilde{D}_{i, t}$ is
\begin{equation}\label{eq:app_approx}
\tilde{D}_{i, t} = \sum_{c'=1}^C \widehat{\Pr}_t[c']p_{i, t}(c'),
\end{equation}
where we have replaced the true probability $\Pr[c']$ of context $c'$ with the empirical probability $\widehat{\Pr}_t[c']$ from our observation of contexts thus far.

The first difficulty we run into is that even when $\tilde{D}_{i, t}$ is a good approximation to $D_{i, t}$, since these terms occur in the denominators of equations \eqref{app:eq1} and \eqref{app:eq2}, tiny additive errors can be amplified if both quantities are small. To be more concrete, assume that with high probability $|\tilde{D}_{i, t} - D_{i, t}|$ is at most $\delta$ (at best we should expect $\delta$ to be around $T^{-1/2}$, since we have at most $T$ samples of contexts). We can then relate the first and second moments of our new approximate estimator $\tilde{r}_{i, t}(c)$ to those of our original unbiased estimator:
$$\E[\tilde{r}_{i,t}(c)] = \E[\hat{r}_{i, t}(c)] + \frac{D_{i, t} - \tilde{D}_{i, t}}{\tilde{D}_{i, t}}r_{i,t}(c)$$
and
$$\E[\tilde{r}_{i,t}(c)^2] = \E[\hat{r}_{i, t}(c)^2] + \frac{\left(\sum_{c'=1}^C\Pr[c'] p_{i,t}(c')\right)^2 - \tilde{D}_{i, t}^2}{\tilde{D}_{i, t}^2D_{i, t}}r_{i,t}(c)^2.$$

If $D_{i, t}$ and $\tilde{D}_{i, t}$ are both reasonably large (e.g. bounded away from $0$ by a constant), these approximations are quite good: in both cases, we have an $O(\delta)$ additive approximation to the corresponding moment for the unbiased estimator $\hat{r}_{i, t}(c)$.  These $O(\delta)$ additive approximations end up contributing an extra $O(\delta T)$ regret, which is fine if $\delta = O(T^{-1/2})$. 

But when $D_{i,t}$ and $\tilde{D}_{i, t}$ are close to zero, these additive errors explode. And although it may be rare, it is possible for $D_{i, t}$ to be small (this just means that it is unlikely to pick action $i$ this round over all contexts). In this case, the best direct lower bound we can show for $D_{i,t}$ is the exploration constant $\alpha$ -- but setting $\alpha$ to a value smaller than $O(1/\sqrt{T})$ would result more than $\sqrt{T}$ regret, and lower bounding $D_{i, t}$ by $1/\sqrt{T}$ increases the additive error in these approximations to $O(\delta \sqrt{T})$. Since we sustain this additive error each round, for $\delta = O(T^{-1/2})$ this leads to $O(1)$ extra regret per round for a total of $O(T)$ regret, so we get no meaningful regret guarantee. (By balancing $\delta$ and $\alpha$ it is possible to recover an $\tilde{O}(T^{2/3})$ regret algorithm, albeit one more complicated than \expcrossu).

The second difficulty we face is more subtle, but arguably more pernicious. Recall that earlier we mentioned that with $O(T)$ samples from $\D$, we should be able to approximate $D_{i, t}$ to within $O(T^{-1/2})$ with high probability. Ordinarily, this type of bound would follow from Hoeffding's inequality: if we have $S$ samples from $\D$, we can rewrite the definition of $\tilde{D}_{i, t}$ in Equation \eqref{eq:app_approx} as $\tilde{D}_{i,t} = \sum_{s=1}^S p_{i,t}(c_s)$, which is a sum of $S$ bounded independent random variables (where the $s$th such variable takes on value $p_{i, t}(c_s)$ if $c_s$ is the $s$th context we observe), which should not differ from their mean by more than $O(1/\sqrt{S})$. But there is one big caveat here: in practice, we do not see a fresh {(independent)} set of samples from $\D$ each round, but rather we add a single new sample to our empirical distribution each round. 
Since our choice of function $p_{i, t}(c)$ in round $t$ depends on the contexts we have observed before round $t$, our sample of contexts at round $t$ \textit{is not statistically independent from} the values of $p_{i, t}(c)$ at round $t$. To wit, if we use the realized contexts observed so far $c_1,\ldots,c_t$ as samples, we could write $\tilde{D}_{i,t} = \sum_{s=1}^t p_{i,t}(c_s)$ and, now, the random variables $p_{i,t}(c_s)$ are no longer independent because the historical contexts $c_1,\ldots,c_{t-1}$ are used to compute $p_{i,t}$. This prevents the above random variables from being independent, and hence prevents us from applying Hoeffding's inequality as desired.

\section{Other Applications of Cross-Learning}\label{sec:extra_app}

\paragraph{Multi-armed bandits with exogenous costs:}
In this problem, as in the standard stochastic multi-armed bandit problem, a learner must repeatedly (for $T$ rounds) make a choice between $K$ options, where the reward $r_{i,t} \in [0, 1]$ from choosing option $i$ is drawn from some distribution $\D_{i}$ with mean $\mu_i$. However, in addition to this, at the beginning of each round $t$, a cost $s_{i,t} \in [0, 1]$ for playing arm $i$ is adversarially chosen and publicly announced (and choosing arm $i$ this round results in a net reward of $r_{i,t} - s_{i,t}$). The learner's goal is to get low regret compared to the optimal strategy, which always chooses the option which maximizes $\mu_{i} - s_{i,t}$. 

This can be thought of as a contextual bandits problem where the context $c_t$ is the cost vector $s_{t}$. Discretizing the context space $[0, 1]^K$ into $(1/\eps)^K$ pieces and running $S$-UCB results in an overall regret bound of $\tilde{O}(\sqrt{TK\eps^{-K}} + \eps T)$. Optimizing over $\eps$ yields $\eps = (K/T)^{1/(K+2)}$, which results in a regret of $\tilde{O}(T^{(K+1)/(K+2)}K^{1/(K+2)})$.

Again, cross-learning between contexts is possible. Applying \ucbcross{}, this immediately leads to an algorithm which achieves regret $\tilde{O}(\sqrt{KT})$ (which is optimal since the standard stochastic multi-armed bandit problem is a special case of this problem).

\paragraph{Sleeping bandits:}
In this variant of sleeping bandits, there are $K$ arms (with stochastically generated rewards in $[0, 1]$) and in each round some nonempty subset $S_t$ of these arms are awake. The learner can play any arm and observe its reward, but only receives this reward if they play an awake arm. The learner would like to get low regret compared to the best policy (which always plays the awake arm whose distribution has the highest mean). 

This is a contextual bandits problem where the context $c_t$ is the set $S_t$ of awake arms. Since there are $2^{K} - 1$ possible contexts, naively applying $S$-UCB results in an regret bound of $\tilde{O}(\sqrt{2^{K}KT})$. On the other hand, cross-learning between contexts is again present in this setting: given the observation of the reward of arm $i$, one can infer the received reward for any context $S_t'$ by just checking whether $i \in S_t'$. Applying \ucbcross{}, this leads to an optimal $\tilde{O}(\sqrt{KT})$ regret algorithm for this problem. 

In the setting of sleeping bandits originally studied by %Kleinberg, Niculescu-Mizi, and Sharma,
\cite{kleinberg2010regret}, the learner can neither play nor observe sleeping arms. We can capture this setting via contextual bandits with partial cross-learning. 
We adjust the previous setting so that if a learner chooses an asleep arm, they receive zero reward and observe nothing else. Note that in this case, we have the following partial learning structure between contexts:

\begin{itemize}
    \item If arm $I_t \in S_t$, rewards $r_{I_t, t}(S)$ are revealed for all other subsets $S$. These rewards are given by $r_{I_t, t}(S) = \mathbbm{I}(I_t \in S)r_{I_t, t}(S_t)$.
    \item If arm $I_t \not\in S_t$, rewards $r_{I_t, t}(S)$ are revealed only for subsets $S$ where $I_t \not\in S$. These rewards are given by $r_{I_t, t}(S) = 0$.
\end{itemize}

In other words, $G_i$ is the following graph: there is an edge from $S_1 \rightarrow S_2$ if either $i \in S_i$ or if $i \not\in S_1 \cup S_2$. Note that $G_i$ has clique cover number $\kappa(G_i) = 2$; the set of subsets containing $i$ and the set of subsets not containing $i$ both form subcliques of $G_i$. It follows from Theorem \ref{thm:ucbcross} that running Algorithm \ref{alg:UCB1PC} results in an optimal regret bound of $\tilde{O}(\sqrt{KT})$.

In \cite{KanadeMB09}, the authors study a variant of sleeping bandits where the subsets $S_t$ are generated stochastically, but the rewards are chosen adversarially. They demonstrate an algorithm for this setting with $\tilde{O}(K^{1/5}T^{4/5})$ regret, which was later improved to $\tilde{O}(K^{1/3}T^{2/3})$ regret by \cite{neu2014online}. In our language, this is simply the adversarial contexts/stochastic rewards variant of the above problem. Applying Theorem \ref{thm:exppcross}, it follows that if the distribution over subsets is known, Algorithm \ref{alg:exp3PC} incurs regret at most  $\tilde{O}(\sqrt{KT})$ for this problem.

\section{Further Numerical Studies} \label{sec:numercis}

In this section, we investigate if our algorithms -- when used in the problem of bidding in the first-price auctions -- are sensitive to how outliers are handled. To do so, we generate synthetic auction data in which the value of a bidder and their highest competing bid are generated from a  correlated log-normal distribution with a correlation coefficient of $0.4$. (Recall that in our empirical results presented in Section \ref{sect:experiments}, the correlation coefficient between the bidder's value and their highest competing bid is $0.4$.) In the generated auction data, the probability that the maximum of the bidder's value and the highest competing bid are larger than $1$ is $0.1$. (The bidder's bid and highest competing bid are independent across auctions.)

We consider two settings. In the first setting, we remove all the auctions in which the maximum of the bidder's value and the highest competing bid is greater than $1$, and in the second setting, we keep all the auctions. In both of these settings, in all the considered algorithms, the submitted bids by the bidders are restricted to the range of $[0,1]$, and allowable bids are discretized to multiples of $0.01$. 

Figures \ref{fig:regret_with_removing_outliers} and \ref{fig:regret_without_removing_outliers} show the cumulative  regret versus time $t$ of  our three algorithms (\ucbcross{}, \expcrossk{}, and \expcrossu{}) that use cross-learning between values/contexts and the two benchmark algorithms  $S$-EXP3, and $S$-UCB1  under the two aforementioned settings. We again observe that our three algorithms outperform  $S$-EXP3, and $S$-UCB1, demonstrating robustness of our algorithms to outliers. In addition, as in Figure \ref{fig:regrets}, \expcrossk{} surpasses \expcrossu{} despite the fact that both algorithms do not know the context distribution.  % consider complete cross-learning between contexts despite the correlation between the bidder's value and highest competing bid. 

\begin{figure}[t]
  \centering
  \begin{subfigure}[t]{0.8\textwidth}
    \centering    \includegraphics[height=7cm]{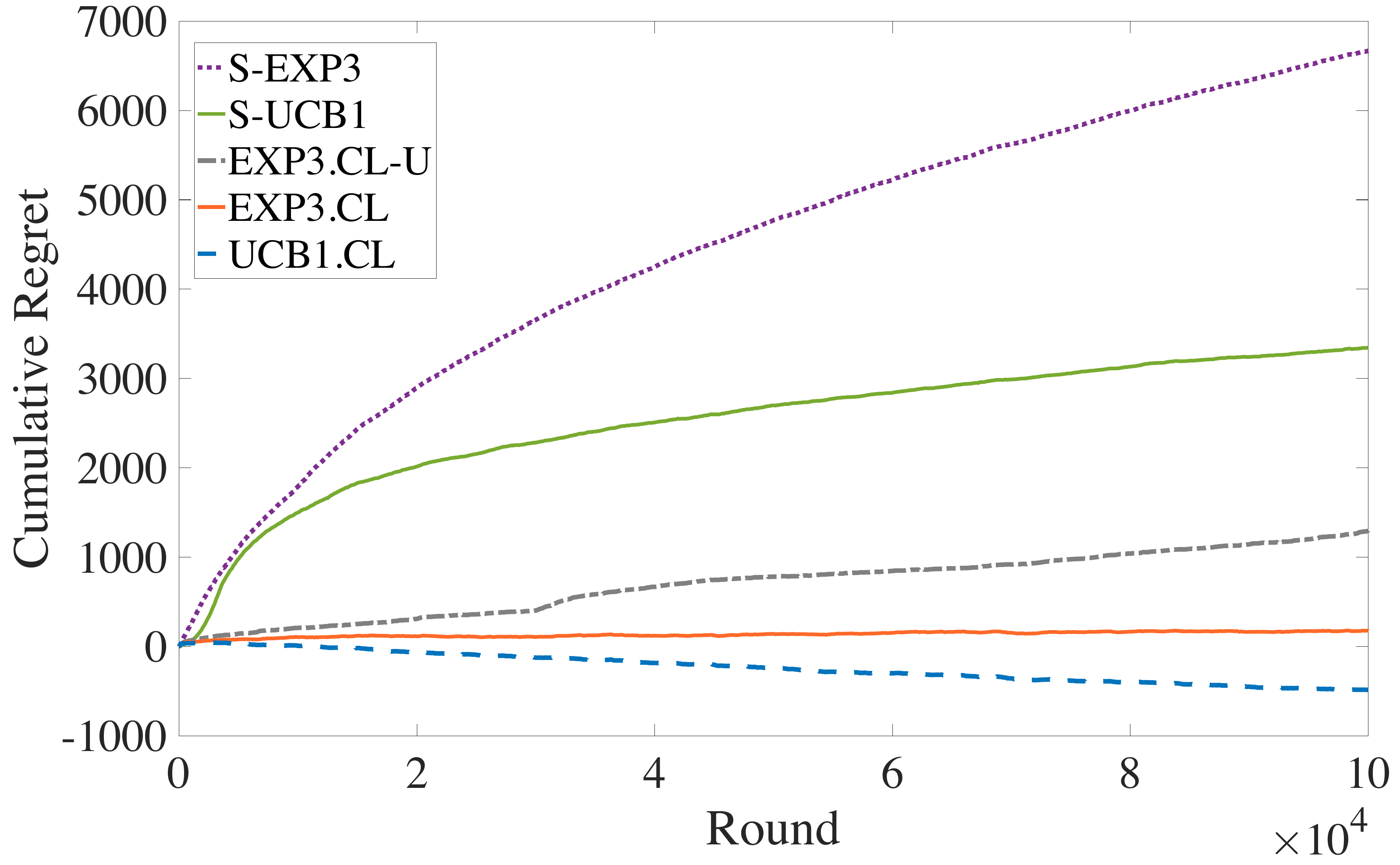}
    \caption{Outliers are removed}
    \label{fig:regret_with_removing_outliers}
  \end{subfigure}%
  \\
  \bigskip
  \begin{subfigure}[t]{0.8\textwidth}  
    \centering
    \includegraphics[height=7cm]{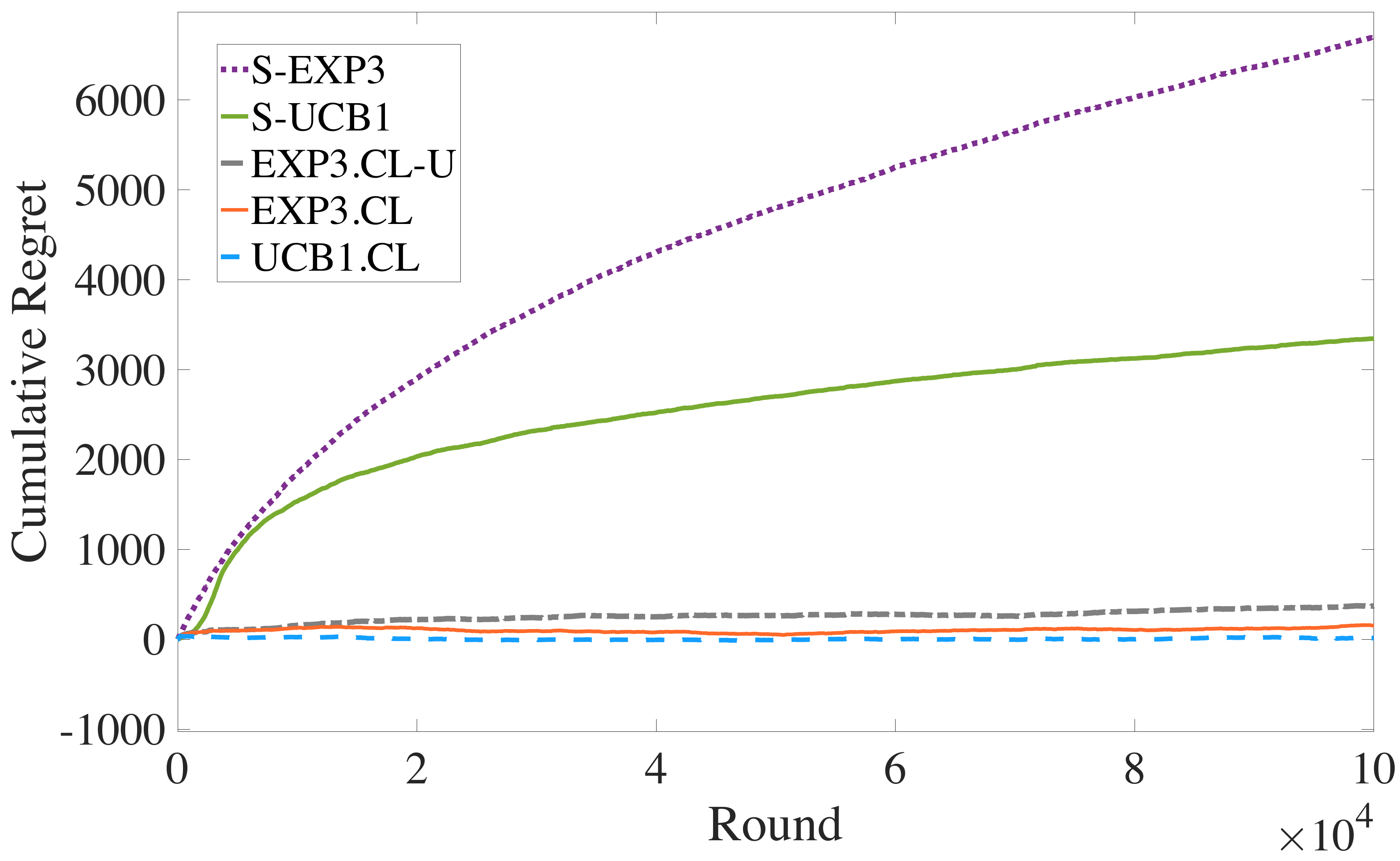}
    \caption{Outliers are not removed}
    \label{fig:regret_without_removing_outliers}
  \end{subfigure}  
  \caption{Graph of average cumulative regrets of various learning algorithms (y-axis) versus time (x-axis). Taking advantage of cross-learning via our algorithms (\ucbcross{}, \expcrossk{}, and \expcrossu{}) leads to a lower regret, regardless of how outliers are handled.  }
\end{figure}

\end{APPENDICES}

\end{document}